%% file: main.tex
\documentclass{article}
\usepackage{microtype}
\usepackage{graphicx}
\usepackage{subcaption}
\usepackage{stfloats}
\usepackage{booktabs} 
\usepackage{wrapfig}
\usepackage{xcolor}
\definecolor{blu}{RGB}{45,80,160}
\newcommand{\blu}[1]{{\color{blu}#1}}

\usepackage[most]{tcolorbox}
\usepackage{hyperref}

\usepackage[preprint]{icml2026}

\usepackage{amsmath}
\usepackage{amssymb}
\usepackage{mathtools}
\usepackage{amsthm}
\input{macros}
\newtcolorbox{findingbox}[1][]{%
  enhanced,
  colback=pink!15,
  colframe=black,
  width=0.99\linewidth,
  arc=5pt,
  boxrule=0.75pt,
  left=5pt,right=5pt,top=5pt,bottom=5pt,
  boxsep=2pt,
  before upper={\textbf{#1}:\ },
}

\usepackage[capitalize]{cleveref}
\usepackage[suppress]{color-edits}
\addauthor[B]{bv}{magenta}
\addauthor[P]{pd}{orange}
\addauthor[A]{ab}{red}
\addauthor[V]{vs}{pink}
\addauthor[C]{ct}{blue}
\usepackage{enumerate}

\theoremstyle{plain}

\theoremstyle{definition}

\theoremstyle{remark}

\usepackage[textsize=tiny]{todonotes}

\icmltitlerunning{Understanding Contextual Recall in Transformers}

\begin{document}

\twocolumn[
  \icmltitle{Understanding Contextual Recall in Transformers: \\How Finetuning Enables In-Context Reasoning over Pretraining Knowledge}



  \icmlsetsymbol{equal}{*}

  \begin{icmlauthorlist}
    \icmlauthor{Bhavya Vasudeva}{USC}
    \icmlauthor{Puneesh Deora}{UBC}
    \icmlauthor{Alberto Bietti}{Flatiron Institute}
    \icmlauthor{Vatsal Sharan}{USC}
    \icmlauthor{Christos Thrampoulidis}{UBC}
  \end{icmlauthorlist}

  \icmlaffiliation{USC}{University of Southern California}
  \icmlaffiliation{UBC}{University of British Columbia}
  \icmlaffiliation{Flatiron Institute}{Flatiron Institute}

  \icmlcorrespondingauthor{Bhavya Vasudeva}{bvasudev@usc.edu}

  \icmlkeywords{Machine Learning, ICML}

  \vskip 0.3in
]



\printAffiliationsAndNotice{}  

\begin{abstract} 
Transformer-based language models excel at in-context learning (ICL), where they can adapt to new tasks based on contextual examples, without parameter updates. In a specific form of ICL, which we refer to as \textit{contextual recall}, models pretrained on open-ended text leverage pairwise examples to recall specific facts in novel prompt formats.
We investigate whether contextual recall emerges from pretraining alone, what finetuning is required, and what mechanisms drive the necessary representations. For this, we introduce a controlled synthetic framework where pretraining sequences consist of subject-grammar-attribute tuples, with attribute types tied to grammar statistics. We demonstrate that while such pretraining successfully yields factual knowledge, it is insufficient for contextual recall: models fail to implicitly infer attribute types when the grammar statistics are removed in ICL prompts. However, we show that finetuning on tasks requiring implicit inference, distinct from the ICL evaluation, using a subset of subjects, triggers the emergence of contextual recall across all subjects. This transition is accompanied by the formation of low-dimensional latent encodings of the shared attribute type. For mechanistic insight, we derive a construction for an attention-only transformer that replicates the transition from factual to contextual recall, corroborated by empirical validation.
\end{abstract}
\input{sections/intro}

\input{sections/setup}

\input{sections/expts}
\input{sections/construction}

\input{sections/conclusion}

\vspace{-2.5mm}
\section*{Acknowledgments}
\vspace{-1mm}
The authors acknowledge use of the Discovery cluster by USC CARC, and thank Tina Behnia for helpful discussions in the early phases of the work. PD was supported by the UBC 4YF Doctoral Fellowship. VS was supported by NSF CAREER Award CCF-2239265, an Amazon Research Award, a Google Research Scholar Award and
an Okawa Foundation Research Grant. CT and PD were partially supported by the NSERC Discovery Grant No. 2021-03677 and the Alliance Grant ALLRP 581098-22.

\bibliography{refs}
\bibliographystyle{icml2026}

\newpage
\onecolumn
\appendix
\input{sections/app}


\end{document}

%% file: macros.tex
\usepackage{graphicx,amsmath,
amsfonts,amssymb,
bm,
url,breakurl,epsf,color,mathbbol,
fmtcount,semtrans,caption,multirow,comment,boldline,microtype}
\usepackage{cancel}
\usepackage{hyperref}
\usepackage{amssymb}
\usepackage{mathrsfs}
\usepackage{nicematrix}
\usepackage{mathtools}

\hypersetup{
    colorlinks=true,
    linkcolor=[rgb]{1,0.0,0.5},
    filecolor=magenta,      
    urlcolor=[rgb]{0.3,0.75,0.7},
    citecolor=[rgb]{0.1,0.3,0.88}
    }
\usepackage{titlesec}

\usepackage{tikz}
\usepackage{pgfplots}
\usetikzlibrary{pgfplots.groupplots}

\usepackage{subcaption}

\usepackage[utf8]{inputenc} 
\usepackage[T1]{fontenc}    
\usepackage{url}            
\usepackage{booktabs}       
\usepackage{nicefrac}       
\usepackage{microtype}      
\usepackage{dsfont}
\usepackage{xspace}

\usepackage[bottom,hang,flushmargin,multiple,symbol]{footmisc} 
\usepackage{fnpct}

\newcommand{\vb}{{\vct{v}}}

\newcommand{\Ib}{{\mtx{I}}}

\newcommand{\emb}[1]{\bm{h}(#1)}

\newcommand{\obj}{s}
\newcommand{\objb}{\mathcal{S}}
\newcommand{\atr}{a}
\newcommand{\atrb}{\mathcal{A}}
\newcommand{\uatr}{u}
\newcommand{\uatrb}{\mathcal{U}}
\newcommand{\rel}{r}
\newcommand{\relb}{\mathcal{R}}
\newcommand{\gr}{g}
\newcommand{\grm}{\tilde g}
\newcommand{\grmb}{\mathcal{G}}
\newcommand{\sep}{[\text{sep}]}
\newcommand{\pe}{\pb}
\newcommand{\seq}{X}

\newcommand{\gh}{g_h}

\newcommand{\Ntr}{N_{\text{tr}}}
\newcommand{\Nft}{N_{\text{ft}}}
\newcommand{\Sft}{S_{\text{ft}}}

\newcommand{\ellN}{\ell{\Ntr}}

\newcommand{\object}{\text{subject}}

\newcommand{\MC}{\text{TM}}
\newcommand{\MCD}{\texttt{Div}}

\newcommand{\Bio}{\texttt{PT}\,\,}

\newcommand{\ICL}{\texttt{ICL}\,\,}
\newcommand{\ICLG}{\texttt{ICL+Grm}\,\,}
\newcommand{\cc}{||}

\newcommand{\sft}[1]{\bm{\varphi}(#1)}

\newcommand{\X}{{\mtx{X}}}

\newcommand{\Pbb}{{\mtx{P}}}

\newcommand{\Tau}{\mathrm{T}}

\newcommand{\vct}[1]{\bm{#1}}
\newcommand{\mtx}[1]{\bm{#1}}

\newcommand{\cut}[1]{\textcolor{red}{}}
\newcommand{\W}{\vct{W}}

\newcommand{\pb}{\vct{p}}

\newcommand{\ub}{\vct{u}}

\newcommand{\beq}{\begin{equation}}
\newcommand{\eeq}{\end{equation}}
\newcommand{\bea}{\begin{align}}
\newcommand{\eea}{\end{align}}

\newcommand{\vsn}{\vspace{0mm}}
\newcommand{\R}{\mathbb{R}}
\newcommand{\E}{\mathbb{E}}

\newcommand{\nn}{\notag}

\newcommand{\ind}[1]{\mathds{1}\left[#1\right]}

\DeclarePairedDelimiterX{\inp}[2]{\langle}{\rangle}{#1, #2}
\newcommand{\inpb}[2]{\left\langle #1, #2 \right\rangle}

\providecommand{\norm}[1]{\lVert#1\rVert}

\providecommand{\Unif}[1]{\operatorname{Unif}(#1)}

\newcommand{\xb}{{\vct{x}}}

\newcommand{\calM}{\mathcal{M}}

\newtheorem{theorem}{Theorem}
\newtheorem{lemma}{Lemma}[section]
\newtheorem{proposition}{Proposition}



%% file: sections/intro.tex
\vspace{-3.5mm}
\section{Introduction}
\vspace{-1mm}
\label{sec:intro}
Transformer-based large language models (LLMs) exhibit remarkable abilities to extrapolate far beyond  tasks seen during training. A notable instance of this extrapolation is in-context learning (ICL) \citep{brown2020language}, where models can adapt to new tasks based on contextual examples, without parameter updates. 

In this paper, we investigate a specific form of ICL, which we refer to as \emph{contextual recall}. Here, a model trained on open-ended text acquires factual knowledge and is later able to recall specific facts when prompted with example pairs in an unseen format.  To illustrate, pretraining data might include descriptions of various landmarks, from which the model learns multiple attributes for each—such as the country where it is located, the year it was built, or its architectural style. At test time, the model receives a prompt of the form \texttt{[Niagara Falls, Canada. Colosseum, Italy. Parthenon, ]} and must generate \texttt{[Greece]}. The prompt contains no explicit indication that the relevant attribute is ``country'' rather than, say, ``year built''; the model must infer the attribute type from the in-context examples and recall the corresponding fact for the query subject.

The ability to succeed at contextual recall therefore requires both the acquisition of factual knowledge and adaptability to novel prompt formats, since the model may never have seen these facts presented as implicit subject-attribute pairs during pretraining. In this work, we investigate the origins of this capability: 
\vspace{-1.5mm}
\begin{center}
    \textit{Does the ability to do contextual recall emerge naturally from pretraining, or does it necessitate specific finetuning? Furthermore, what mechanisms within the Transformer architecture enable the emergence of this ability? 
    }
\end{center}
\vspace{-1.5mm}

\begin{figure*}[h!]
    \centering
    \includegraphics[width=0.98\textwidth]{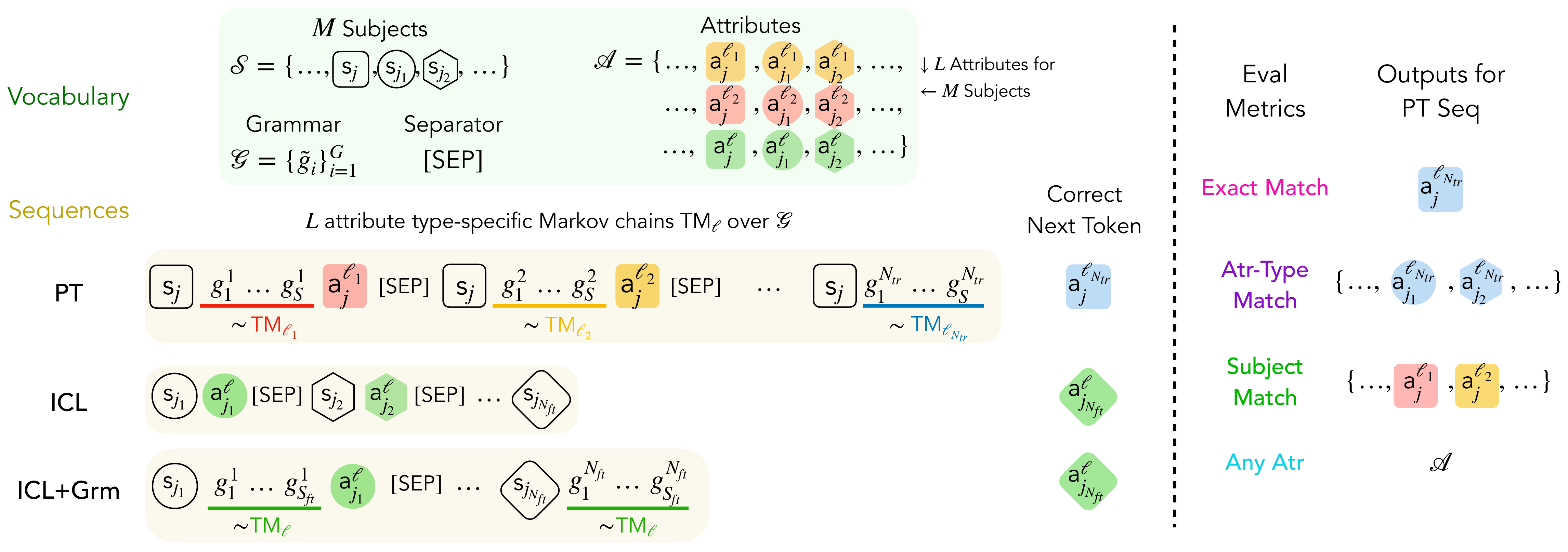}
    \caption{Illustration of the data generation process and the evaluation metrics. The vocabulary consists of $M$ subjects with $L$ attributes each (where the attribute tokens may be shared across the subjects), $G$ grammar tokens and a separator token.\abcomment{Are there $LM$ attributes in total or can they be shared across subjects/relations?} We sample $L$ Markov chains, one for each attribute type. We consider three types of sequences (see \cref{sec:mc_dgp} for details): \Bio sequences contain subject-grammar-attribute tuples, with attribute type information encoded in the grammar sequence statistics, for one subject, \ICL sequences contain pairwise subject-attribute examples for a shared attribute type, but no grammar, and \ICLG sequences are analogous to \ICL sequences, but contain subject-grammar-attribute tuples. We evaluate whether the model's prediction (on \Bio sequences) matches the ground truth attribute (exact match), any attribute of the same type as the last subsequence, any attribute of the same subject as the sequence, or any attribute token.}
    \label{fig:sequences}
\end{figure*}

Prior work has leveraged controlled, synthetic settings to understand both ICL and factual recall in transformers trained from scratch. For ICL, various studies have considered well-defined tasks, such as  linear regression \citep{garg_icl, akyurek_icl, vonOswald_icl, zhang2023trained, fu2024transformers}, discrete functions \citep{bhattamishra2024understanding}, and Markov processes \citep{bietti2023birth,statistical_induction_heads, rajaraman2024transformers, algorithmic_phases,deora2025incontext} that involve function induction, but do not require recalling learned factual information. Importantly, and in contrast to our contextual recall task, the train and test data share the sequence formats in these settings. 
On the other hand, in studies focused on question-answering and recall tasks \citep{allen2023physics,behnia2025factsstatsimpactspretraining,nichani2024understanding,zucchet2025language}, the context contains explicit information (e.g., specific words, or linguistic patterns) required to recall specific learned facts, and does not require the model to perform any implicit inference (e.g., what fact is relevant based on the in-context examples) that is characteristic of ICL. (See \cref{app:rel_work} for a detailed discussion on related work.)

\vspace{-3.5mm}
\paragraph{Contributions.} We summarize our contributions below. 

First, in \cref{sec:mc_dgp}, we introduce a controlled synthetic framework to study the emergence of contextual recall in transformers. The pretraining (\Bio\!\!) sequences, used to instill factual knowledge in the model, contain multiple attributes of a single subject, interspersed with grammar tokens that encode information about each attribute type. The \ICL sequences, used to evaluate contextual recall ability, consist of subject-attribute pairs sharing a common attribute type across different subjects, without any grammar (see \cref{fig:sequences} for an illustration of the data generation process).

In \cref{sec:pt_expts}, we show that transformers trained with \Bio sequences succeed on factual recall, but fail to generalize to \ICL sequences, which require the model to implicitly infer the attribute type from the in-context examples. However, we find that finetuning the model on sequences that are distinct from \ICL sequences, but require implicit inference, using a subset of subjects, enables out-of-distribution generalization on \ICL sequences on the held-out subjects (\cref{sec:ft_expts}). We also probe the effect of some key dataset parameters on the model's performance on \Bio and \ICL sequences in \cref{sec:ablations}. 

In \cref{sec:rep_analysis}, we analyze model representations and find that finetuning induces the formation of low-dimensional encodings of the shared attribute type (task vectors) based on the in-context examples. These representations become more disentangled as the number of in-context demonstrations increases. 

Finally, in \cref{sec:attn_only}, for mechanistic insight, we consider a simpler synthetic task and present constructions for an attention-only model that succeeds on factual recall after pretraining and contextual recall after finetuning, and corroborate them with an analysis of the training dynamics for finetuning as well as empirical validation.

%% file: sections/setup.tex
\vspace{-3mm}
\section{Data Generation Process}
\label{sec:mc_dgp}
\vspace{-1.5mm}

To study contextual recall in transformers, we define three different types of sequences for pretraining, finetuning, and/or evaluation. The first, \Bio\!\!, is designed to instill factual knowledge, while the other two, \ICL and \ICLG\!\!, are used to evaluate or facilitate how the model leverages such knowledge for in-context reasoning tasks.

We first introduce some useful notation. Let $\objb=\{\obj_j\}_{j=1}^M$ 
denote the set of $M$ subjects (e.g., landmarks such as \texttt{Parthenon} ($j\!=\!1$), \texttt{Colosseum} ($j\!=\!2$)). Let the set of unique attributes be denoted as $\uatrb\!=\!\cup_{\ell=1}^L\uatrb_\ell$, where $\uatrb_\ell\!=\!\{\uatr_i^\ell\}_{i=1}^{M_\ell}$, and 
$\ell$ indexes the attribute type (e.g., $\ell\!=\!1$ for ``country'', $\ell\!=\!2$ for ``year built''), and $M_{\ell}$ denotes the number of unique values for index $\ell$ (e.g., \{\texttt{Greece}, \texttt{Italy}, \texttt{Canada}, \dots\} for ``country''). Next, let 
$\atrb=\cup_{j=1}^M\{\atr_j^1,\dots,\atr_j^L\}$ denote the set of attributes assigned to the subjects, where  $\atr_j^\ell$ is the type-$\ell$ attribute of subject $j$ (e.g., $\atr_{2}^1 \!=\! \texttt{Italy}$). Each subject $s_j$ has one attribute per type, giving $L$ attributes per subject. Let $\grmb\!=\!\{\grm_1,\dots,\grm_G\}$ denote a set of $G$ grammar tokens, and $\sep$ denote the separator token. The full vocabulary is $\mathcal{V}\!=\!\objb\cup\uatrb\cup\grmb\cup\{\sep\}$, with size $V\!=\!M\!+\!\sum_{\ell=1}^L M_\ell\!+\!G\!+\!1$. Let $\Unif{\cdot}$ denote the uniform distribution, and $\cc$ denote concatenation.

With this notation established, we now describe the data generation process. The key building block is a \emph{\object-grammar-attribute subsequence}. For subject $\obj_j$, attribute type $\ell$, and grammar sequence length $S$, this subsequence is denoted as $\seq_j^\ell = \bigl[\obj_j , \gr_1,\dots,\gr_S , \atr^{\ell}_j\bigr]$ (e.g., $[\texttt{Parthenon}, \textit{one}, \textit{of}, \textit{the}, \textit{wonders}, \dots, \textit{is in}, \texttt{Greece}]$). 
The grammar sequence $\gr_{1:S}$ is generated using a first-order Markov chain specific to attribute type $\ell$. That is, $p(\gr_t = \grm |\gr_{t-1},\dots, \gr_1) = p(\gr_t = \grm|\gr_{t-1})$ for all $\grm\in\grmb$, with the first token drawn uniformly at random. For each attribute type $\ell\in[L]$, we sample a row-stochastic transition matrix $\MC_\ell \in \R^{G \times G}$, where each row is drawn independently from a Dirichlet prior with parameter $\boldsymbol{\alpha}$. Note that since each attribute type has its own Markov chain, the bigram statistics of the grammar sequence $\gr_{1:S}$ implicitly encode information about the attribute type $\ell$.

We now define the three types of sequences used for pretraining, finetuning, and/or evaluation; see \cref{fig:sequences} for an illustration. 

\vspace{-3mm}
\paragraph{\Bio Sequences (Pretraining).} To instill factual knowledge in the model, we use \Bio sequences which contain information about a specific \object~and its associated attributes (analogous to a short encyclopedia entry for a given landmark). Let $\Ntr$ denote the number of subsequences (subject-grammar-attribute tuples) in each sequence. To generate a \Bio sequence, we first sample a subject $j\!\sim\! \Unif{[M]}$, then draw $\Ntr$ attribute types $\ell_{1:\Ntr}\!\sim\!\Unif{[L]}^{\Ntr}$, and sample each grammar subsequence $\gr^i_{1:S}$ from the corresponding Markov chain $\MC_{\ell_i}$. The resulting sequence is $\tilde\seq\!=\!\seq_j^{\ell_1}\cc \sep\cc \seq_j^{\ell_2}\dots \seq_j^{\ell_{\Ntr}}$, with sequence length $T\!=\!(S+3)\Ntr\!-\!1$. 
To model generic text that does not convey factual information, we also include grammar-only subsequences: with probability $p_G$, each subject-grammar-attribute subsequence is replaced by a grammar-only subsequence $\gr_{1:S+2}$, generated from a separate fixed Markov chain (with transition matrix sampled from a Dirichlet prior). This gives the final sequence $\seq$.
 
\vspace{-3mm}
\paragraph{\ICL Sequences (Finetuning and Evaluation).} To test for contextual recall, we use \ICL sequences that contain \object-attribute pairs with a shared attribute type across different subjects—mirroring the evaluation format from the introduction (e.g., [\texttt{Niagara Falls}, \texttt{Canada}, \sep, \texttt{Colosseum}, \texttt{Italy}, \sep, \texttt{Parthenon}, ]). Let $N$ denote the number of in-context demonstrations. First, we sample an attribute type $\ell\sim\Unif{[L]}$, and $N+1$ subjects $j_{1:N+1}\sim \Unif{[M]}^{N+1}$. The resulting \ICL sequence is $\seq=[\obj_{j_1},\atr_{j_1}^\ell,\sep,\dots,\obj_{j_{N+1}},\atr_{j_{N+1}}^\ell]$. At evaluation, the model observes $N$ subject-attribute pairs and must predict the attribute $\atr_{j_{N+1}}^\ell$ for the final subject $\obj_{j_{N+1}}$, without any grammar cues indicating the attribute type. 
\vspace{-3mm}
\paragraph{\ICLG Sequences (Finetuning).} We also use \ICLG sequences for finetuning.  As the name suggests, these sequences are similar to \ICL sequences (subject-attribute pairs sharing a common attribute type across different subjects), but they also contain grammar tokens between the subject and attribute. (As we will see in \cref{sec:ft_expts}, these sequences help instill the implicit inference ability required to succeed at contextual recall on \ICL sequences.) Let $\Sft$ denote the grammar sequence length used in each subsequence, and $\Nft$ denote the number of subsequences. We sample attribute type $\ell\sim\Unif{[L]}$ and \object s $j_{1:\Nft+1}\sim \Unif{[M]}^{\Nft+1}$, then generate grammar subsequences $\gr^i_{1:\Sft}$ from the corresponding Markov chain $\MC_{\ell}$. The resulting \ICLG sequence is $\seq=\seq_{j_1}^\ell\cc\sep\cc\seq_{j_2}^\ell\dots\seq_{j_{\Nft+1}}^\ell$. Note that \ICL sequences are a special case of \ICLG with $\Sft=0$.

Here, we note that several prior works leverage controlled/synthetic settings to study factual recall or ICL in transformers. In particular, closest to our work are \citet{nichani2024understanding,behnia2025factsstatsimpactspretraining}, which adopt abstracted setups with subject-grammar-attribute triplets to study factual recall. \citet{nichani2024understanding} uses a setup where the relation (or attribute type) is a single, dedicated token, while other tokens in the grammar sequence are sampled randomly (in \cref{sec:attn_only}, we use a similar setup for mechanistic analysis of contextual recall for an attention-only model). In contrast, \citet{behnia2025factsstatsimpactspretraining} propose a setup to more faithfully capture the statistical structure in the grammar sequences, by modeling each template as a randomly sampled Markov chain, but omit the relation token. Our synthetic framework bridges these approaches: inspired by \citet{behnia2025factsstatsimpactspretraining}, we model templates using Markov chains, but we retain the relation or attribute type information by associating the Markov chains with specific attribute types. Crucially, as discussed in \cref{sec:intro}, both \citet{nichani2024understanding} and 
\citet{behnia2025factsstatsimpactspretraining} focus on \emph{factual recall}, 
where the context contains explicit cues about the relevant 
attribute type, whereas our focus is on \emph{contextual recall}, which requires the model to infer the attribute type implicitly from in-context examples, that is not captured in prior frameworks. (See App. \ref{app:rel_work} for a detailed discussion on related work.)

%% file: sections/expts.tex
\vspace{-3mm}
\section{Experimental Results}
In this section, we present the experimental results for pretraining and finetuning.
\vspace{-2.5mm}
\subsection{Pretraining on \Bio Sequences} \label{sec:pt_expts}
\vspace{-1mm}
\textbf{Setup.} We train a two-layer, single-head, decoder-only transformer on \Bio sequences to minimize the standard next-token prediction objective. We use an online training setup: at each iteration, we generate a fresh batch of \Bio sequences using the process described in Sec.~\ref{sec:mc_dgp}. Unless stated otherwise, we fix $M\!=\!256$ subjects, $L\!=\!8$ attribute types, and $S\!=\!80$ grammar length (see App. ~\ref{app:expts} for further details). 

\begin{figure}[h!]
    \centering  \includegraphics[width=0.5\linewidth]{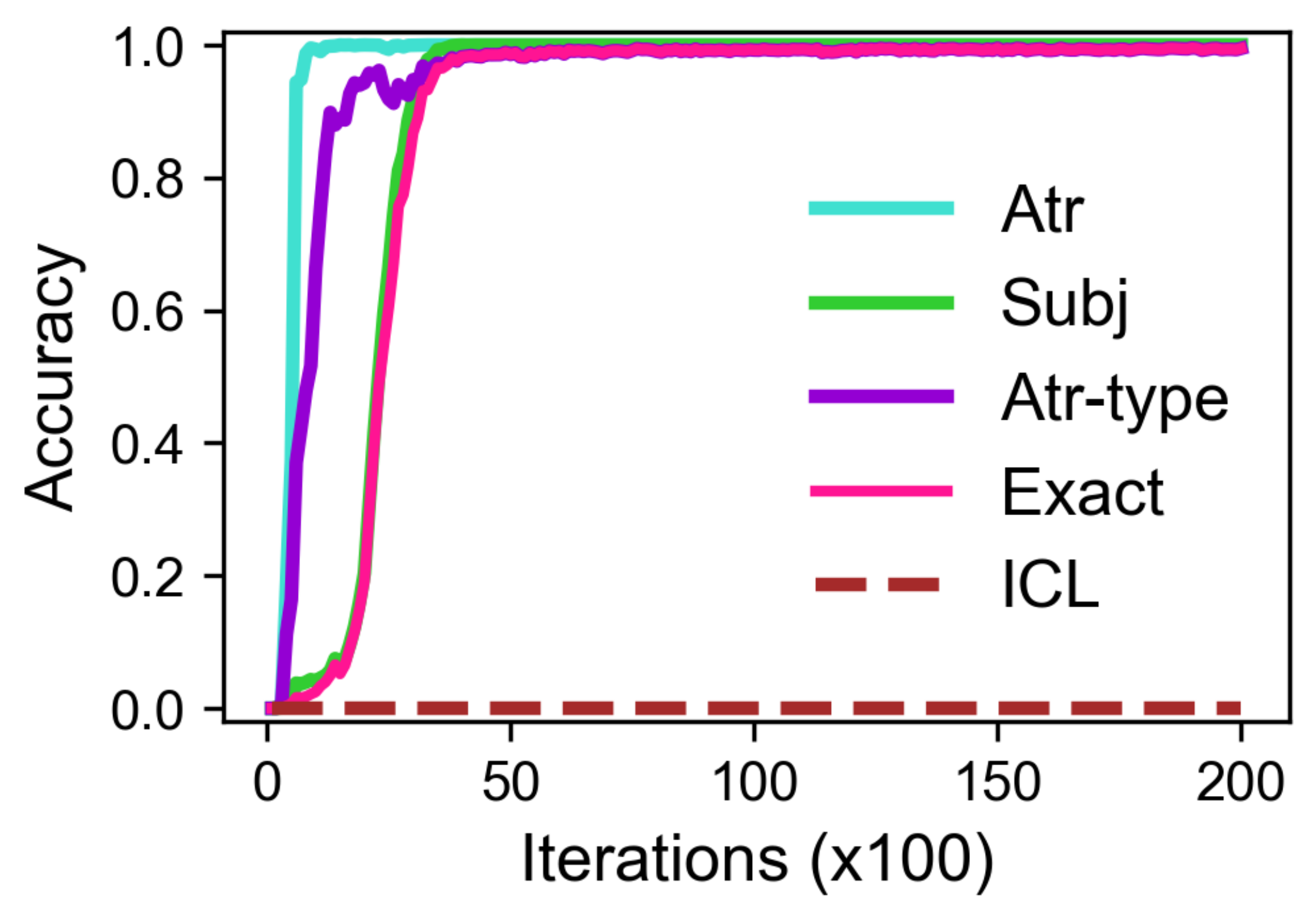}
    \vsn
    \caption{Transformer trained on \Bio sequences performs well on \Bio sequences, but does not generalize to \ICL sequences.}
    \label{fig:pt}
    \vspace{-0.2in}
\end{figure}
\paragraph{Evaluation.} We evaluate the model on two distinct held-out sets: i) \Bio sequences, to evaluate factual recall, and ii) \ICL sequences, to evaluate contextual recall capabilities. For both, we measure the accuracy of predicting the final token given the  preceding context. We report four metrics: 
\emph{Exact Match} (predicted attribute matches the ground truth), \emph{Attribute-Type Match} (predicted attribute has the correct type), \emph{Subject Match} (predicted attribute belongs to the correct subject), and \emph{Attribute Match} (any attribute token is predicted). See \cref{fig:sequences} for an illustration.

\cref{fig:pt} illustrates the performance as pretraining progresses. On the \Bio sequences (solid lines), we observe a stage-wise learning process: the model first learns to predict the attribute type, and then the exact attribute token. Crucially, however, we find that high performance on \Bio sequences does not transfer to \ICL sequences, with the model achieving near-zero accuracy on \ICL sequences. This shows that despite learning the factual associations, the model relies on \emph{explicit} grammar statistics for retrieval and cannot \emph{implicitly} infer the attribute type from in-context examples alone. We summarize this in our first main finding.

\vspace{-2mm}
\begin{center}
\begin{findingbox}[Finding 1]
Pretraining on \Bio sequences does \emph{not} suffice for good \ICL performance.
\end{findingbox}
\end{center}
\vspace{-2mm} 
\subsection{Finetuning Experiments}
\label{sec:ft_expts}
\vspace{-2mm}
Since pretraining on \Bio sequences alone does not suffice for good \ICL performance, we investigate whether finetuning can bridge this gap. Specifically, we ask: can a model originally trained to rely on explicit grammar-based cues be adapted to perform implicit inference from in-context examples? To answer this, we finetune the pretrained model using the standard next-token prediction objective on a new data distribution (detailed below). Crucially, to test for generalization, we finetune on only a subset of subjects, reserving the remaining subjects as a held-out set. We then evaluate the model on \ICL sequences (as defined in \cref{sec:mc_dgp}) where the query subject belongs to this held-out set. Unless otherwise stated, we set the number of demonstrations $\Nft=16$, and use $50\%$ of the total \object s for finetuning.

\begin{figure}[h!]
    \centering  \includegraphics[width=0.9\linewidth]{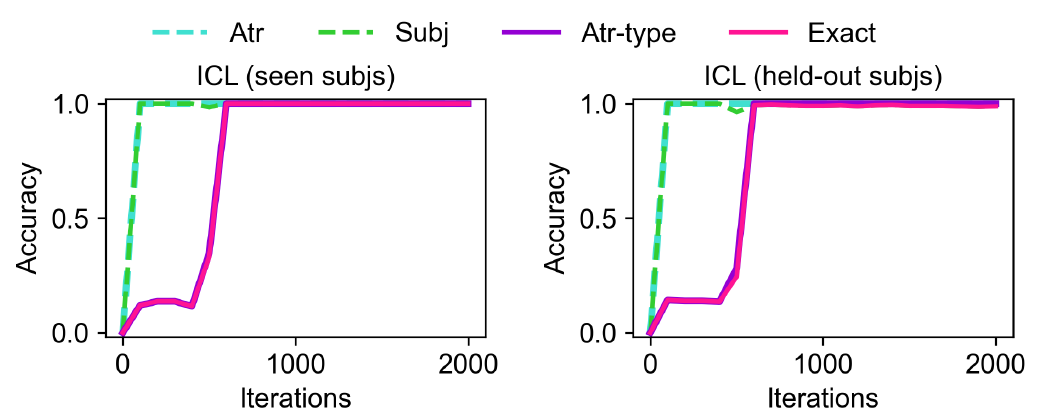}
    \vspace{-0.1in}
    \caption{Finetuning the model pretrained on \Bio sequences, using \ICL sequences with a subset of \object s, leads to good performance on \ICL sequences with held-out query \object s. \abcomment{Hard to see the curves other than green and purple}}
    \label{fig:ft_icl}
    \vspace{-0.15in}
\end{figure}
\vspace{-1mm}
\paragraph{FT on \ICL using a subset of \object s.} 
We first consider the most direct approach: finetuning the pretrained model directly on \ICL sequences. This serves as an upper bound, since the finetuning and evaluation formats are identical, \textit{i.e.}, there is no distribution shift in prompt format. \cref{fig:ft_icl} shows  performance on \ICL sequences with seen and held-out query subjects. As finetuning progresses, the model successfully learns to perform contextual recall, generalizing to held-out subjects. This is our second main finding.

\vspace{-1mm}
\begin{center}
\begin{findingbox}[Finding 2]
Finetuning on \ICL sequences with a subset of \object s leads to generalization on \ICL sequences with held-out \object s.
\end{findingbox}
\end{center}

\vspace{-2mm}
\paragraph{Mechanism of Transfer.} 
Why does finetuning on a subset of subjects enable generalization to the rest? We posit that pretraining and finetuning play distinct but complementary roles. During pretraining, the model \emph{acquires factual knowledge}: given a \Bio subsequence containing $\obj_i$, the model learns to decode the attribute type $\ell$ from the grammar statistics and predict the corresponding attribute $\atr_i^\ell$. During finetuning, the model does not learn new subject-attribute associations; rather, it learns a new \emph{access mechanism:} inferring the attribute type implicitly from the attribute tokens in the context, rather than from explicit grammar cues. Because there is a shared structure between every $\obj_i$ and its type-$\ell$ attribute $\atr_i^\ell$ (across $i\in[M]$), the model can be finetuned on a subset of subjects to learn this implicit inference mechanism and generalize to held-out subjects. We present additional evidence for this mechanism in \cref{sec:attn_only} using a simpler setting.

\vspace{-2mm}

\begin{figure}[h!]
    \centering
\includegraphics[width=0.85\linewidth]{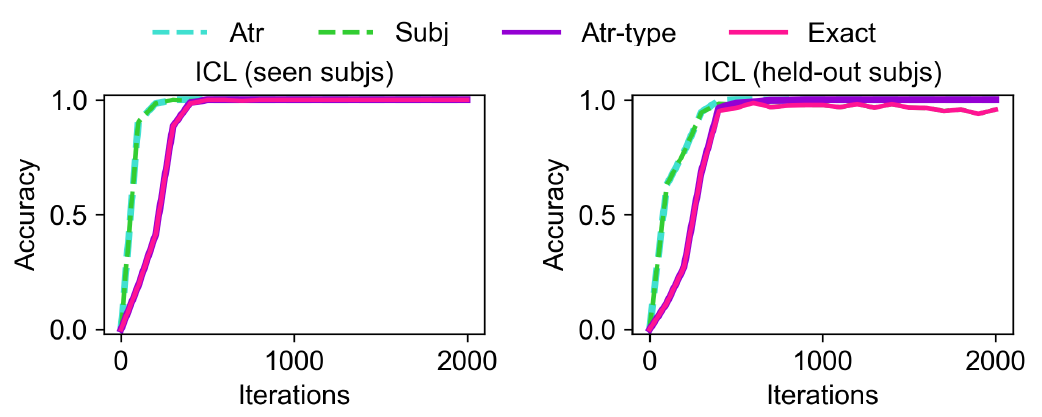}
    \caption{Finetuning the model pretrained on \Bio sequences, using \ICLG sequences with short, variable grammar length with a subset of \object s, leads to good performance on \ICL sequences with held-out query \object s.}\vspace{-0.2in}
    \label{fig:ft_iclwgrm}
\end{figure}

\vspace{-2mm}
\paragraph{FT on \ICLG using a subset of \object s.} We next ask whether finetuning on \ICL sequences is necessary to induce implicit attribute-type inference, or whether other distributions can achieve the same effect. To investigate this, we finetune the pretrained model on \ICLG sequences with short, variable grammar length. Specifically, we use a randomly sampled $\Sft\!\in\!\{1,\dots,5\}$ to generate each \ICLG sequence. Due to the short grammar length, the sequence statistics are insufficient to reliably encode the attribute type, encouraging the model to instead infer it implicitly from the attribute tokens in the context. In this sense, these sequences serve the same purpose as \ICL sequences. 

However, there remains a format distribution shift: during finetuning, attributes always follow a grammar token, whereas in evaluation \ICL sequences, attributes directly follow subjects. The model must therefore learn to bridge this format gap at test time.Using variable grammar lengths prevents the model from overfitting to a fixed positional offset, encouraging it to rely on the context to infer the position of the next attribute token. \cref{fig:ft_iclwgrm} shows that the model successfully generalizes to \ICL sequences for both seen and held-out subjects, as it is finetuned on \ICLG sequences.

\vspace{-2mm}
\begin{center}
\begin{findingbox}[Finding 3]
Finetuning on \ICLG sequences with short, variable grammar length using a subset of \object s leads to out-of-distribution generalization on \ICL sequences with held-out \object s.
\end{findingbox}
\end{center}
\vspace{-2mm}
\begin{figure}[h!]
    \centering  
    \includegraphics[width=0.425\linewidth]{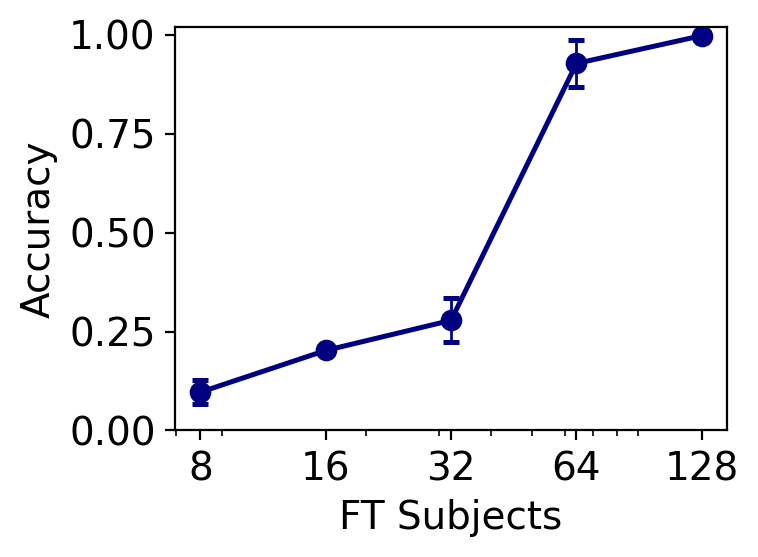}\hspace{1mm}\includegraphics[width=0.425\linewidth]{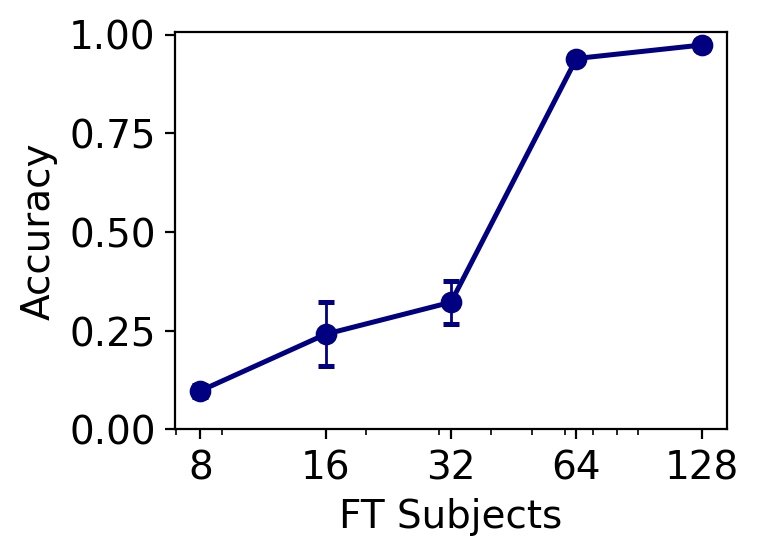}
    \vsn
    \caption{Increasing the number of finetuning \object s with \ICL (left) or \ICLG (right) sequences improves the model's performance on \ICL sequences with held-out query \object s.}
    \label{fig:subj_ft_plot}
    \vspace{-0.125in}
\end{figure}
Next, we investigate whether the format distribution shift inherent in finetuning with \ICLG sequences incurs a cost in terms of sample efficiency compared to direct finetuning with \ICL sequences. In \cref{fig:subj_ft_plot}, we compare the \ICL performance on held-out subjects as we vary the number of finetuning subjects. Interestingly, we find that the sample complexity is comparable across both settings, with performance improving monotonically as the number of finetuning subjects increases.

\vspace{-2mm}
\subsection{Effect of data distribution}\label{sec:ablations}Having established that finetuning on \ICLG enables contextual recall, we now investigate how the properties of the pretraining data influence this capability. Unless stated otherwise, we use the same parameters as in \cref{sec:pt_expts,sec:ft_expts} and use \ICLG sequences with short, variable grammar length for finetuning. 
\vspace{-2mm}

\begin{figure}[h!]
    \centering  \includegraphics[width=0.95\linewidth]{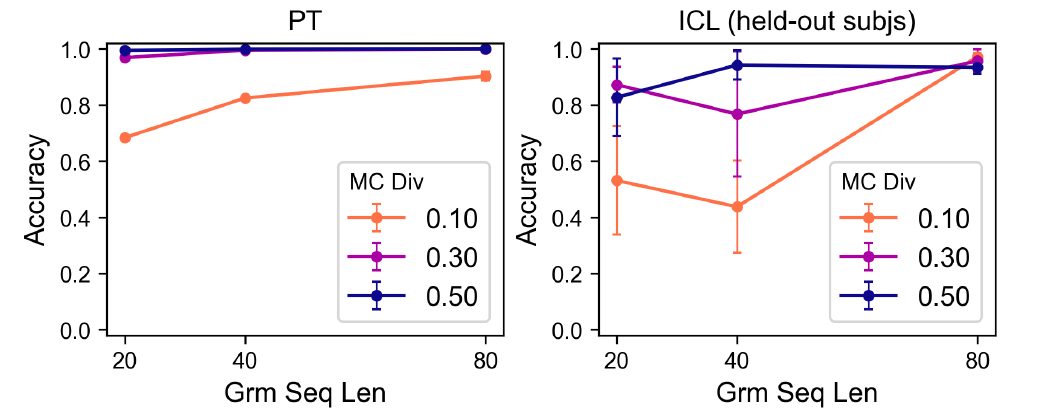}
    \vspace{-0.1in}
    \caption{Increasing the sequence length $S$ and/or Markov chain diversity $\MCD$ used while pretraining improves the model's performance on \Bio sequences as well as on \ICL sequences with held-out \object s after finetuning with \ICLG sequences.}
    \label{fig:div_plot}
    \vspace{-0.15in}
\end{figure}

Keeping other parameters fixed, we first consider the effect of varying the sequence length $S$ and the diversity between the Markov chains for different attribute types, quantified as $\MCD:=\min_{\ell,\ell'}\norm{\MC_\ell-\MC_{\ell'}}_1$, where $\norm{\cdot}_1$ denotes the $\ell_1$-norm. Intuitively, a higher $\MCD$ for a fixed sequence length implies that the grammar statistics for different attribute types are more distinct, making the attribute type easier to distinguish during pretraining. As shown in \cref{fig:div_plot}, increasing the sequence length $S$ and/or the Markov chain diversity $\MCD$ used while pretraining, improves the model's performance both on the \Bio sequences, as well as on \ICL sequences with held-out \object s after finetuning on \ICLG sequences with short, variable grammar length. 
\abcomment{Are grammar blocks always of fixed length~$S$ or of random lengths up to~$S$? Not super clear to me if larger~$S$ helps identify the relation, or helps avoid relying on positions too much. One way to test could be to replace grammars by one token identifying the relation, and repeat it~$S$ times.}\bvcomment{For pretraining, grammar blocks are fixed length $S$, would be interesting to probe effect of positional diversity.}

\vspace{-2mm}
\begin{figure}[h!]
    \centering  \includegraphics[width=0.5\linewidth]{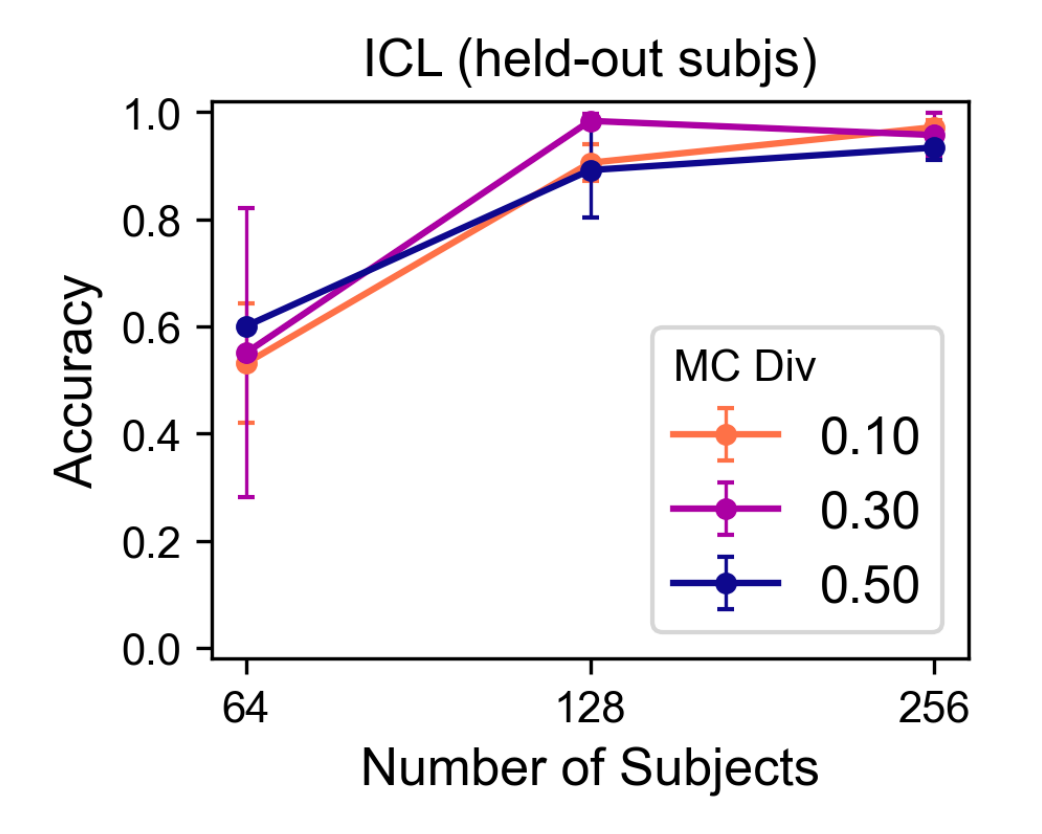}
    \vspace{-0.1in}
    \caption{Increasing the number of \object s used while pretraining improves the model's performance on \ICL sequences with held-out \object s after finetuning with \ICLG sequences.}
    \label{fig:subj_plot}
    \vspace{-0.125in}
\end{figure}

Next, in \cref{fig:subj_plot}, we probe the effect of increasing the number of \object s $M$ used for pretraining. Increasing $M$ improves the performance on \ICL sequences with held-out \object s after finetuning. Here, we use grammar sequence length $S\!=\!80$. Consistent with the observations in \cref{fig:div_plot}, we observe that using a sufficiently large $S$ for pretraining leads to comparable \ICL performance, across different levels of Markov chain diversity $\MCD$. We summarize these results in our next finding. 
\vspace{-1mm}
\begin{center}
\begin{findingbox}[Finding 4]
Increasing the number of \object s, the grammar sequence length, or the separation between the Markov chains for different attribute types used while pretraining improves the final performance on \ICL sequences. 
\end{findingbox}
\end{center}
\vspace{-2mm}
\subsection{Representational Analysis} \label{sec:rep_analysis}
We now analyze the model's internal representations to gain insight into the mechanism underlying contextual recall. Specifically, we investigate whether finetuning on \ICLG sequences induces the formation of a low-dimensional representation that encodes the shared attribute type $\ell$ from the in-context examples.

Let $X_{\ell,t}:= [\obj_{j_0}, \atr^\ell_{j_0}, \sep, \obj_{j_1}, \atr^\ell_{j_1}, \sep,\dots, \obj_{j_{t+1}}]$ denote an \ICL sequence for an attribute type $\ell$, with $t\in[N]$ demonstrations. For each $\ell\in[L]$, we sample $K$ such sequences, denoted $X^k_\ell$ for $k\in[K]$. Let $f_i(\cdot)$ denote the model's representation at layer $i$ at the last token position $\obj_{j_{t+1}}$. For fixed $i$ and $t$,
we measure the cosine similarity for inter- and intra-task representations, averaged across the $K$ contexts for each pair $\ell,\ell'\in[L]$: 
\vspace{-1mm}
\begin{align*}
\bar{C}^t_i(\ell,\ell')
= \tfrac{1}{K^2} \Sigma_{k,k'} 
\cos\!\big( f_i(X^{k}_{\ell,t}),\, f_i(X^{k'}_{\ell',t}) \big).
\vspace{-1mm}
\end{align*}
If the model perfectly disentangles attribute types in its representations, then, for some layer $j$ and number of demonstrations $t$, we would have  $\bar{C}^t_i(\ell,\ell')\!=\!\ind{\ell\!=\!\ell'}$. Additionally, we quantify the representation clustering strength, $\bar{S}^t_i\!\in\![-1,1]$, in terms of a clustering metric using $1-\cos(\cdot,\cdot)$ as the distance and  attribute-type $\ell$ as the cluster label (see App. \ref{app:expts} for details). A high $\bar{S}^t$ indicates that representations of sequences with shared attribute type are tightly clustered and well-separated from those of different attribute types.

\begin{figure}[h!]
    \centering
    \text{$\MCD\approx0.2$}\\
    \includegraphics[width=0.65\linewidth]{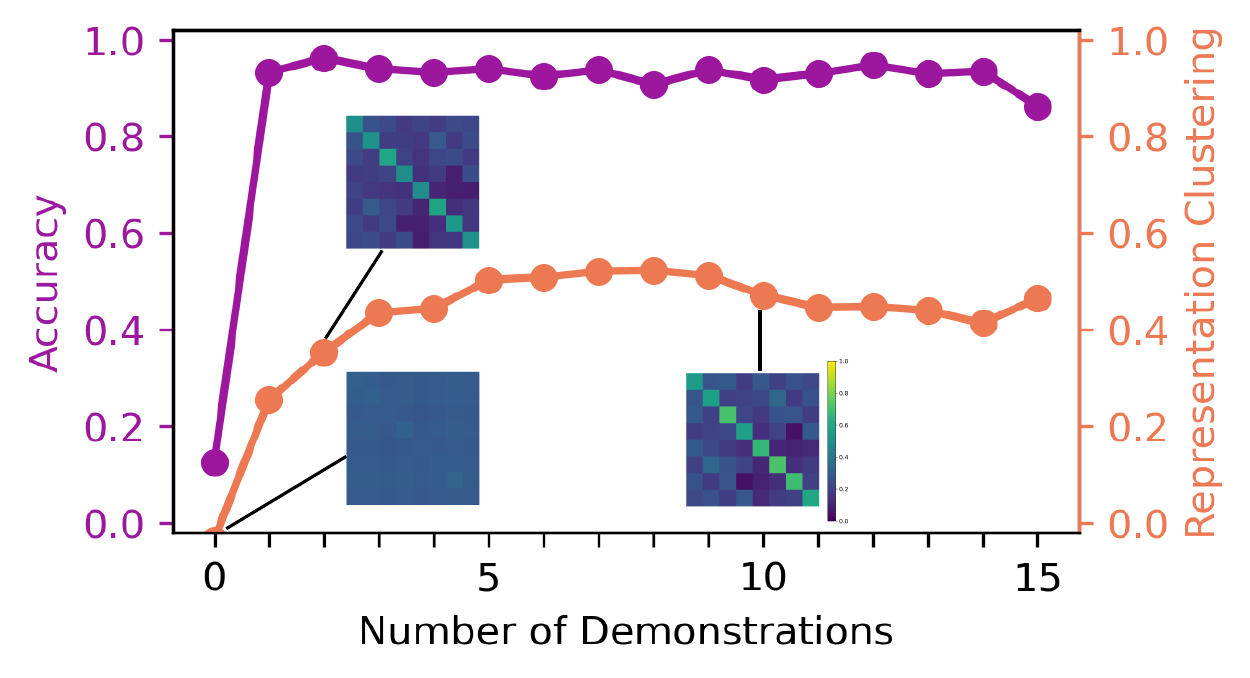}\raisebox{0.3\height}{%
  \includegraphics[width=0.33\linewidth]{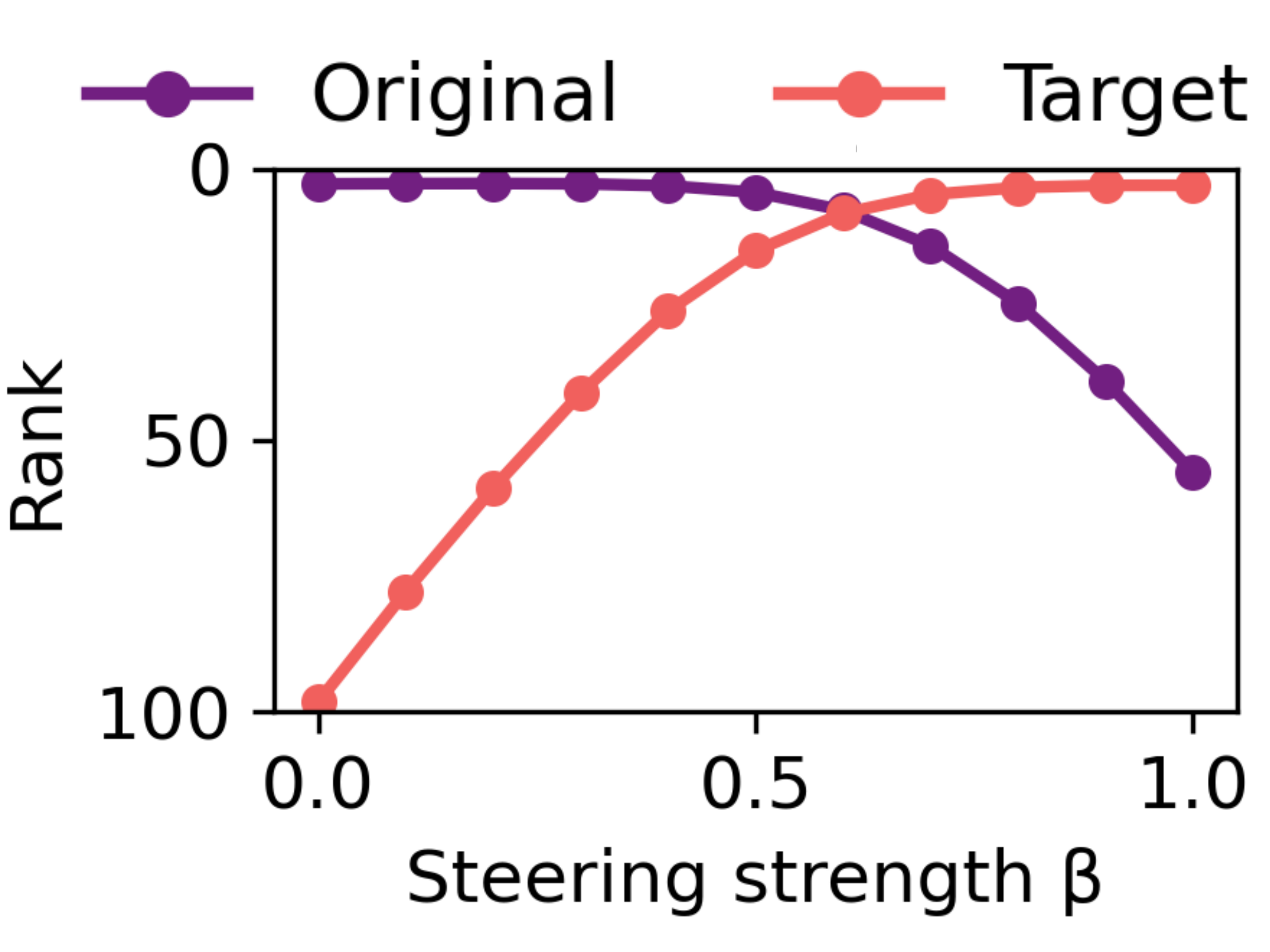}
}
\text{$\MCD\approx0.5$}\\
    \includegraphics[width=0.65\linewidth]{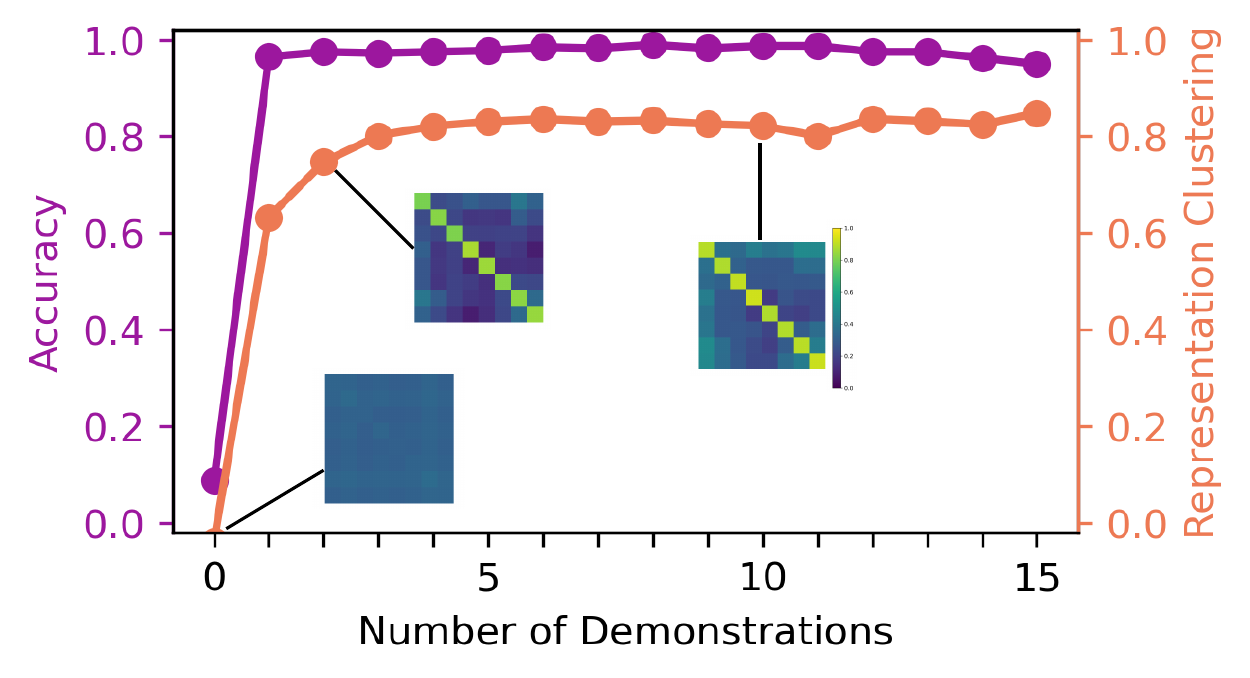}\raisebox{0.3\height}{%
  \includegraphics[width=0.33\linewidth]{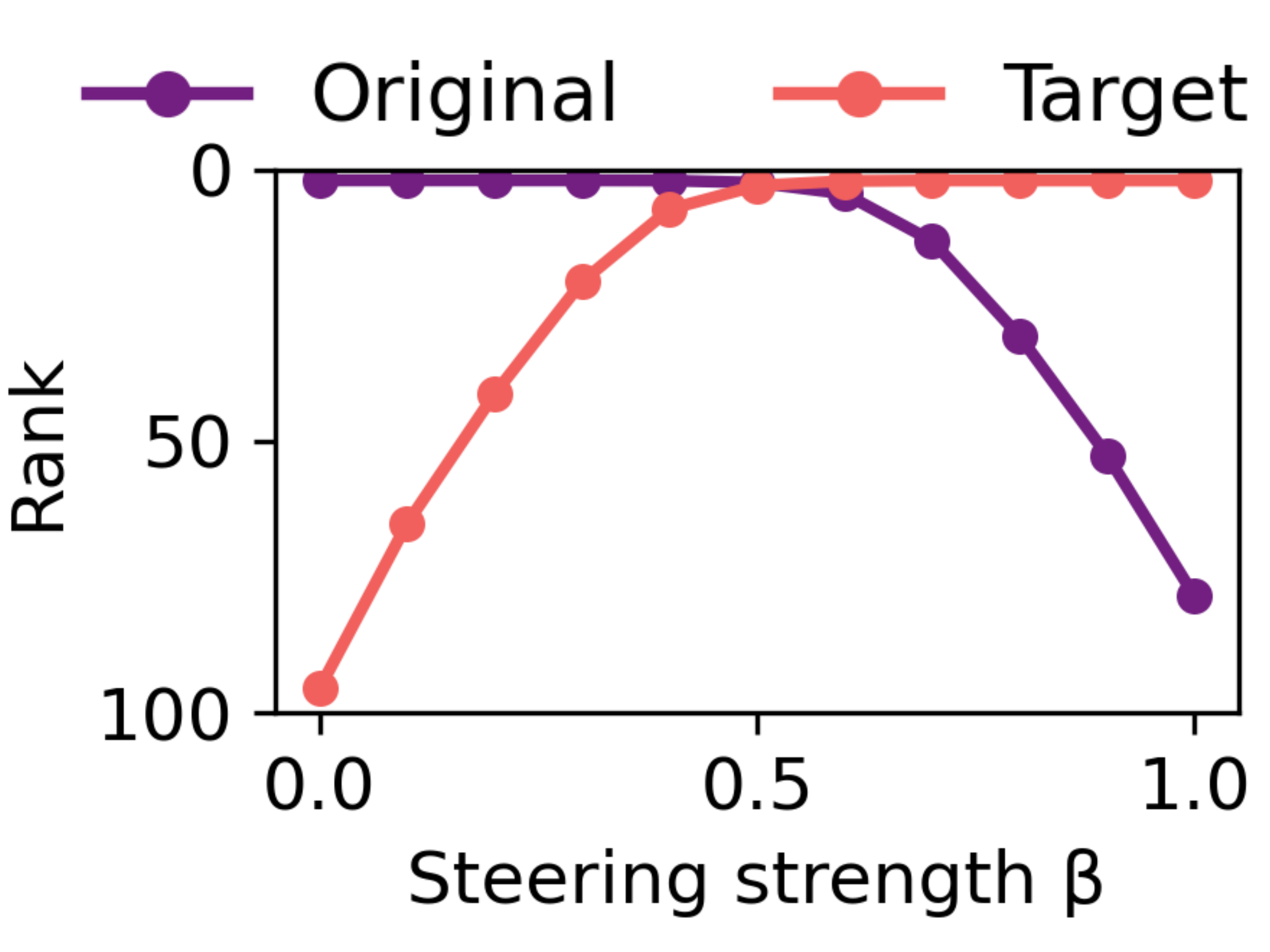}
}
    \caption{Comparison of accuracy on \ICL sequences with held-out subjects, and layer-2 attention representation clustering strength for a finetuned model, as the number of demonstrations is increased  (left), and ranks of the original and target attribute type when using those task representations for steering model output (right; see text for details), using $\MCD\approx0.2$ (top) and $\MCD\approx0.5$ (bottom). Inset figures visualize the averaged cosine similarity for inter- and intra-task representations. In both cases, both performance and clustering strength improve with number of demonstrations. }\vspace{-0.1in}
    \label{fig:clustering}
\end{figure}

\cref{fig:clustering} (left) shows the model's performance on \ICL sequences with held-out subjects, alongside the representation clustering strength $\bar{S}^t$ computed from layer-2 attention representations, as the number of demonstrations $t$ increases. (Results for other layers are included in App. \ref{app:expts}.) We consider two settings: $\MCD\!\approx\!0.2$ (top) and $\MCD\!\approx\!0.5$ (bottom). In both cases,  accuracy and  clustering strength initially improve as $t$ is increased, and eventually saturate. This is also corroborated by the inset figures, which visualize the averaged cosine similarity for inter- and intra-task representations $\bar{C}^t(\ell,\ell')$ across attribute types $\ell,\ell'\!\in\![L]$, after $t\!\in\!\{0,2,10\}$ demonstrations. This confirms that the finetuned model aggregates information from multiple examples to form a stable representation of the attribute type. 

Here, we note that prior work has proposed two main mechanisms for ICL, namely induction heads and task or function vectors, which are low-dimensional representations that encode information about the input-output relations in a given context (see App.~\ref{app:rel_work} for a discussion on related work on task and function vectors for ICL). These results suggest that the formation of these task vectors is the mechanism by which the model succeeds on the contextual recall task in this setting. 

Interestingly, while accuracy is comparable in the two settings, higher $\MCD$ during pretraining leads to stronger representation clustering. This suggests that the separation between the Markov chains (determined by $\MCD$) that encode attribute type information during pretraining, is reflected in the task vector separability and strength after finetuning. 

To corroborate this, we use these representations to steer model outputs. Specifically, let $\vb_\ell:=\tfrac{1}{N}\sum_kf(X^k_{\ell,\Tau})$ denote a task vector for attribute type $\ell$ (where we use layer-2 attention representations at the last token position for sequences of length $\Tau$). For some $\beta\in[0,1]$, we replace $f(X^{k'}_{\ell,\Tau})$ with $(1-\beta) \vb_\ell+\beta \vb_{\ell'}$, and then compare the rank of the original vs the target attribute based on the model's predicted probabilities. \cref{fig:clustering} (right) shows the comparison after averaging across multiple contexts with different source and target attribute types. In both cases, as $\beta$ increases, the model output is steered towards the target. However, for lower $\MCD$ (top), the required $\beta$ is much larger compared to the case with higher $\MCD$ (bottom).

\vspace{-2mm}
\begin{center}
\begin{findingbox}[Finding 5]
The finetuned model encodes attribute type information from in-context examples in layer-2 attention representations. Clustering strength increases with both the number of in-context examples and the separation between the attribute-specific Markov chains during pretraining.
\end{findingbox}
\end{center}

%% file: sections/construction.tex
\vspace{-3mm}
\section{Mechanistic Analysis in a Simpler Setting}
\label{sec:attn_only}
\vspace{-1.5mm}
To gain mechanistic insight into how pretraining and finetuning contribute to contextual recall, we study an analytically tractable setting that 
preserves the qualitative behavior observed in Findings 1-2 of \cref{sec:pt_expts}. Specifically, we simplify in two ways:
i) we encode attribute type information in explicit ``relation'' tokens rather than grammar sequence statistics, and ii) we use a one-layer attention-only model. 
Below, we first describe our setup, then present constructions for accurately predicting the final attribute token on both the \Bio and \ICL sequences, then analyze the training dynamics for FT, followed by experimental validation for the constructions.

\vspace{-2.5mm}
\paragraph{Data Setting.} 
Let $\relb:=\{\rel_1,\dots,\rel_L\}$ denote a set of relation tokens, one per attribute type. The \Bio sequences are generated in a similar manner to \cref{sec:mc_dgp}, with a modification to the grammar subsequences. For each subsequence $\gr_{1:S}$, we first sample a position $i\sim\Unif{[S]}$ and a relation token $\rel\sim\Unif{\relb}$, and set $\gr_i=\rel$. Then, we sample the remaining entries independently as $\gr_{i'}\sim\Unif{\grmb}$ for all $i'\neq i$. The \ICL sequences are generated  as in \cref{sec:mc_dgp}.
\vspace{-2.5mm}
\paragraph{Model.} We consider a single-layer attention-only model with fixed relative positional encoding added to the key inputs. The output at position $t$ for head $h\in[H]$ and the final model output are defined as follows: 
\begin{align}\label{eq:attnonly}
\gh^t(\X)&\!=\!\X_{1:t}^\top\sft{(\X_{1:t}\!+\!\Pbb_{\Tau-t+1:\Tau})\W_{KQ}^h\xb_t}\nonumber\\
f^t(\X)&\!=\!\sum_{h=1}^H\underbrace{\W_{OV}^h\gh^t(\X)}_{f_h^t(\X)}.
\end{align}
where $\X\!\in\! \R^{\Tau\times d}$ denotes the input sequence, $\xb_t\!\in\!\R^d$ denotes its $t^{\text{th}}$ row, $\Pbb\!=\![\emb{p_{-\Tau+1}},\dots,\emb{p_{0}}]$ denotes the positional encodings, where $p_{-\Tau+1}, \dots, p_0$ denote the position indices, and $\emb{\cdot}\!\in\! \R^{d}$ denotes the embedding for a vocabulary token or position index. We consider one-hot embeddings and positional encodings and set $d\!=\!V\!+\!\Tau$, so that these subspaces are orthogonal.\vscomment{i may have missed, did we define what the fixed positional encoding was? (maybe refer to it then?)} $\sft{\cdot}$ denotes the softmax. For convenience, we define $\W_{OV}^h\!=\!(\W_O^h)^\top\W_V^h$, $\W_{KQ}^h\!=\!(\W_K^h)^\top\W_Q^h$, for head $h$ using output,value, key, query weight matrices $\W_O^h,\W_V^h,\W_K^h,\W_Q^h\!\in\!\R^{d_h\times d}$. We use the shorthand $f(\cdot)\!=\!f^\Tau(\cdot)$ and $\gh(\cdot)\!=\!\gh^\Tau(\cdot)$ to denote the outputs at the last token position. 
The model prediction~is
\begin{align*}
v^*=\arg\max_{v\in\mathcal{V}}\emb{v}^\top f(\X),
\end{align*}
where we use fixed embeddings and tied unembeddings. 

\vspace{-2.5mm}
\paragraph{Construction for \Bio Sequences.} 
We first investigate whether an attention-only model is expressive enough to perform  factual recall on \Bio sequences of the form 
\begin{align}
\X&=\X_{j_\star}^{\ell_1}\cc [\emb{\atr_{j_\star}^{\ell_1}\!}]\cc \X_{j_\star}^{\ell_2}\dots \X_{j_\star}^{\ell_{\Ntr}}, \text{where}\nn\\
\X_{j_\star}^{\ell_1} \!\!&=\! [\emb{\obj_{j_\star}},\emb{\gr_{1}\!},..., \emb{\rel_{\ell_1}\!},...,
\emb{\gr_{S}\!},
\emb{\sep}].\label{eq:const_pt_x}
\end{align} Here, $\xb_\Tau=\emb{\sep}$, $j_\star \!\sim \!\Unif{[M]}$ denotes the subject index, and the correct last-token prediction is $\emb{\atr_{j_\star}^{\ell_{\Ntr}}}$. We first show that there exists an attention-only model capable of perfectly predicting the next token for such sequences. For simplicity, we focus on last-token prediction here; the construction extends to predictions at any position (see App.~\ref{app:const}).
\begin{proposition}[Informal]\label{prop:pt}
    Consider the input $\X$ in \cref{eq:const_pt_x}. There exists a single-layer attention-only model such that when weights $\W_{KQ}^h$ are scaled by a sufficiently large $\beta$, it correctly predicts the last token, returning the attribute $\atr_{j_\star}^{\ell_{\Ntr}}$ corresponding to the sequence's subject $\obj_{j_\star}$ and attribute type $\ell_{\Ntr}$.
\end{proposition}    
\vspace{-1mm}
\textit{Proof Sketch.} We present a construction with a $3$-head model. At a high level, two heads, which we call the \textbf{relation} and \textbf{subject} heads, are responsible for the prediction at the target position (following $\emb{\sep}$), and the third \textbf{grammar} head for other cases. For simplicity, we present the construction here assuming that the output of the grammar head is zero ($f_\text{grm}(\X)\!=\!0$), since the outputs from this head do not affect the conclusions in this case, as shown in the full proof in App.~\ref{app:const}.

First, the \textbf{relation head}  attends to the most recent relation token $r_{\ell_{\Ntr}}$ and maps it to the sum of \emph{all} attributes of type $\ell_{\Ntr}$. Specifically, 
\begin{align*}
\W_{KQ}^{\text{rel}}&=\beta\Big(\sum_{\ell}\emb{\rel_\ell}+\pe\Big)\emb{\sep}^\top,\\ \W_{OV}^{\text{rel}}&=\sum_{\ell}\Big(\sum_{j}\emb{\uatr_{j}^\ell}\Big)\emb{\rel_{\ell}}^\top,
\end{align*}
where $\pe:=\sum_{i=1}^{S+2}\emb{p_{-i+1}}$ 
With $\beta\to\infty$, the head's outputs become $g_{\text{rel}}(\X)=\emb{\rel_{\ell_{\Ntr}}}$, i.e., the most recent relation token in the sequence, 
and $f_\text{rel}(\X)=\sum_{j}\emb{\uatr_j^{\ell_{\Ntr}}}$, \textit{i.e.}, all attributes of type $\ell_{\Ntr}$.

On the other hand, the \textbf{subject head}  attends to the subject token $s_{j_\star}$ and filters out irrelevant attributes (those not associated with $s_{j_\star}$). Specifically, \begin{align*}
\W_{KQ}^{\text{subj}}&=\beta\Big(\sum_j\emb{\obj_j}\Big)\emb{\sep}^\top,\\  \W_{OV}^{\text{subj}}&=-\sum_j\Big(\sum_{j'\neq j,\ell}\emb{\atr_{j'}^\ell}\Big)\emb{\obj_j}^\top.
\end{align*}
Then, with $\beta\to\infty$, this head outputs $g_\text{subj}(\X)=\emb{\obj_{j_\star}}$, the subject token in the sequence, and $f_\text{subj}(\X)=-\Big(\sum_{\uatr}\emb{\uatr}-\sum_{\ell}\emb{\atr_{j_\star}^\ell}\Big)$, \textit{i.e.}, the negative of all attributes that are not associated with subject $\obj_{j_\star}$.

Combining these outputs, the model isolates the specific attribute $\atr_{j_\star}^{\ell_{\Ntr}}$, i.e., the attribute reinforced by both heads:  $f(\X)=\sum_{j}\emb{\uatr_j^{\ell_{\Ntr}}}+\sum_{\ell}\emb{\atr_{j_\star}^\ell}-\sum_{\uatr}\emb{\uatr}$, yielding the correct prediction $v^*=\atr_{j_\star}^{\ell_{\Ntr}}$.

\vspace{-2mm}
\paragraph{Construction for \ICL Sequences.} Next, consider \ICL sequences of the form
\begin{align}
\X\!\!=\![\emb{\obj_{j_1}\!},\emb{\sep},\emb{\atr^{\ell_\star}_{j_1}\!},\!...,\emb{\obj_{j_{\Nft+1}}\!},\emb{\sep}],\label{eq:const_ft_x}
\end{align}
with correct last token prediction $\emb{\atr^{\ell_\star}_{j_{\Nft+1}}}$, where $\ell_\star \sim\Unif{[L]}$ denotes the shared attribute type. We show that there exists an attention-only model that perfectly predicts the last token for such sequences.
\begin{proposition}[Informal]\label{prop:icl}
    Consider the input $\X$ in \cref{eq:const_ft_x}. There exists single-layer attention-only model such that when weights $\W_{KQ}^h$ are scaled by a sufficiently large $\beta$, 
    it correctly predicts the last token, returning the attribute $\atr^{\ell_\star}_{j_{\Nft+1}}$ corresponding to the query subject $\obj_{j_{\Nft+1}}$ and the shared attribute type $\ell_\star$.
\end{proposition}
\vspace{-1mm}
\textit{Proof Sketch.} We adapt the construction from \cref{prop:pt} with minimal changes, 
but with one crucial modification that reflects the role of finetuning.
The \textbf{subject} head operates similarly, attending to the query subject and filtering out attributes not associated with it. The key difference lies in the \textbf{relation} head: since \ICL sequences contain no explicit relation tokens, this head must now \emph{infer the attribute type implicitly} by attending to the attribute tokens in the context and mapping them to all attributes of the same type $\ell_\star$.

Specifically, for the \textbf{subject} head, we set 
\begin{align}
\W_{KQ}^\text{subj}&=\beta\Big(\sum_j\emb{\obj_j}+\blu{\emb{p_{-1}}}\Big)\emb{\sep}^\top,\label{eq:kq-subj}\\ \W_{OV}^{\text{subj}}&=-\sum_j\Big(\sum_{j'\neq j,\ell}\emb{\atr_{j'}^\ell}\Big)\emb{\obj_j}^\top.\nonumber
\end{align}
Note that as compared to the construction for \Bio sequences (\cref{prop:pt}), $\W_{OV}^{\text{subj}}$, which contains subject-attribute information, is unchanged. The only difference here is that $\W_{KQ}^\text{subj}$ now contains $\emb{p_{-1}}$, so that when $\beta\to\infty$, $g_\text{subj}(\X)=\emb{\obj_{j_{\Nft+1}}}$, \textit{i.e.}, the query subject, and hence, $f_\text{subj}(\X)=-\Big(\sum_{\uatr}\emb{\uatr}-\sum_{\ell}\emb{\atr_{j_{\Nft+1}}^\ell}\Big)$.

We now discuss the \textbf{relation} head. Since in our experiments, we finetune with a subset of subjects, let $\objb'\subset \objb$ denote a subset of subjects. Next, for each attribute type $\ell$, let $\uatrb'_\ell:=\cup_{j\in\objb'}\{\atr_j^\ell\}$ denote the set of unique attributes seen during finetuning. With this notation established, we set 
\begin{equation} \label{eq:combined_rel}
\begin{aligned}
\W_{KQ}^\text{rel}&=\beta\sum_{\ell}\Big(\emb{\rel_\ell}+\blu{\ub_\ell'}\Big)\emb{\sep}^\top, \quad \ub_\ell':=\sum_{\uatr\in\uatrb'_\ell}\emb{\uatr},\\ \W_{OV}^\text{rel}&=\!\sum_{\ell}\!\Big(\Big(\!\sum_{u\in\mathcal{U}_\ell}\emb{\uatr}\Big)\emb{\rel_\ell}^\top\!+\!\blu{\ub_\ell'\ub_\ell'^\top}\!\Big). \vspace{-2mm}
\end{aligned}
\vspace{-2mm}
\end{equation}

In contrast to the construction for \Bio sequences, where both $\W_{KQ}^\text{rel}$ and $\W_{OV}^\text{rel}$ rely on the relation tokens $\emb{\rel_\ell}$, in this case, they rely on the attributes seen during finetuning, \textit{i.e.}, $\sum_{\uatr\in\uatrb'_\ell}\emb{\uatr}$.
In this case, when $\beta\to\infty$, $g_\text{rel}(\X)=\tfrac{1}{\Nft}\sum_i\emb{\atr_{j_i}^{\ell_\star}}$, \textit{i.e.}, the average of the attributes that appear in the context, and $f_\text{rel}(\X)=\sum_{j}\emb{\uatr_{j}^{\ell_\star}}$, retrieving the sum of all attributes of the shared type $\ell_\star$.

Finally, by combining the outputs from the two heads, the model isolates the specific attribute $\atr_{j_{\Nft+1}}^{\ell_\star}$ that is encouraged by both, \textit{i.e.}
$f(\X)=\sum_{j}\emb{\uatr_{j}^{\ell_\star}}+\sum_{\ell}\emb{\atr_{j_{\Nft+1}}^\ell}-\sum_{\uatr}\emb{\uatr}$. Then, the final prediction $v^*=\atr_{j_{\Nft+1}}^{\ell_\star}$.

Note that, for the \textbf{relation} head to output all attributes of the same type as the \emph{attributes} that appear in the context, for any type $\ell$, the set of attributes seen during finetuning $\uatrb'_\ell$ should match the full set $\uatrb_\ell$. However, since attributes are shared among subjects, the set of subjects seen during finetuning can be much smaller than the full set. This helps explain why finetuning on a subset of subjects can enable generalization on held-out subjects.

\vspace{-2.5mm}
\paragraph{Training Dynamics for Finetuning.} In this section, we consider a 1-layer 2-head attention-only model initialized with the construction for \Bio sequences in \cref{prop:pt}, and examine its training dynamics for finetuning on \ICL sequences with the last-token prediction objective (using a subset of subjects, but with all attributes seen during FT). We show that after updating the weights with one or two steps of GD, under certain assumptions on the data generation process parameters, the model can attain perfect accuracy on \ICL sequences with held-out query subjects.
\begin{theorem}[Informal]\label{th:ft_dynamics}
    Consider a 1-layer 2-head attention-only model defined in \cref{eq:attnonly}, initialized with the weights from the construction for \Bio sequences in \cref{prop:pt}. Assuming sufficiently large $M,M_{\text{ft}}$, small $N_{\text{ft}}$, and $\mathcal{U}_\ell'=\mathcal{U}_\ell,\,\forall \ell\in[L]$, after one step of GD on $\W_{KQ}^\text{subj}$, one step of projected GD on $\W_{OV}^\text{rel}$, and two steps of GD on $\W_{KQ}^\text{rel}$, for last-token prediction on the population of \ICL sequences with a subset of $M_{\text{ft}}$ subjects, the updated weights satisfy
    \begin{equation}\label{eq:kq-subj2}
\W_{KQ}^{\text{subj}}\approx\beta\Big(\sum_j\emb{s_j}+\lambda_1\emb{p_{-1}}\Big)\emb{\sep}^\top,
\end{equation}\vspace{-2mm}
\begin{equation}
\begin{aligned}
\W_{KQ}^{\text{rel}}&\approx \beta\Big(\sum_\ell\Big(\emb{r_\ell}+\tfrac{2}{3}\tfrac{\lambda_2}{LU_L}
\ub_\ell\Big)+\tfrac{\lambda_2}{\Nft}\sum_{i=1}^{\Nft}
\emb{p_{1-3i}}\\-\lambda_2\Big(\tfrac{1}{\Tau}\pb&+\tfrac{1}{3}\tfrac{1}{M}\sum_s\emb{s}\Big)\Big)\emb{\sep}^\top, \quad \ub_\ell:=\sum_{u\in\mathcal{U}_\ell}\emb{\uatr},\\
\W_{OV}^{\text{rel}}&\approx\sum_{\ell}\ub_\ell\emb{\rel_{\ell}}^\top+\lambda_3\Nft\sum_{\ell}
\ub_\ell\ub_\ell^\top\\&+\lambda_3\frac{U_L^2}{M_{\text{ft}}}\sum_{j=1}^{M_{\text{ft}}}
\Big( \sum_{\ell} \emb{\atr_j^\ell}\Big) \emb{\obj_j}^\top,\vspace{-3mm}
    \end{aligned}
    \label{eq:combined_rel2}
    \end{equation}
    where $\lambda_{1},\lambda_2,\lambda_3$ are some constants. These weights, with sufficiently large $\beta$, lead the model to attain perfect accuracy on \ICL sequences for all subjects.
\end{theorem}
We defer the full statement and its proof to \cref{app:const}. 
Interestingly (albeit under certain assumptions), the updated weights after finetuning recover the construction in \cref{prop:icl} (upto some additive factors that don't change the model functionality), as we discuss next.

\textit{Subject head.} Comparing $\W_{KQ}^{\text{subj}}$ in \cref{eq:kq-subj2} 
with \cref{eq:kq-subj}, the only difference is the second term is scaled by $\lambda_1$ ($\propto \!\eta^{\text{subj}}_{KQ}$, the step-size). 
When $\beta \!\to \!\infty$, regardless of this coefficient, the attention concentrates on the query subject 
$\obj_{j_{\Nft+1}}$. Crucially,  $\W_{OV}^{\text{subj}}$ does not need to be updated since the subject-attribute associations are already learned during pretraining.

\textit{Relation head.} Next, comparing $\W_{KQ}^{\text{rel}}$ in \cref{eq:combined_rel2} with \cref{eq:combined_rel}, there are three terms in \cref{eq:combined_rel2}: the first term matches the construction (other than the coefficient $\tfrac{2\lambda_2}{3LU_L}$), while the second term contains positional encodings for the positions at which attributes appear; these terms boost the attention on the attribute tokens in the context. The third term contains i) a positional encoding term (which contributes uniformly to every token in the context, leaving the softmax scores unchanged), and ii) a subject token term that suppresses the attention at the subject tokens. For $\W_{OV}^{\text{rel}}$, the first two terms in \cref{eq:combined_rel2} match the construction (other than the coefficient $\lambda_3\Nft$). The third term encodes subject-attribute associations for the finetuning subjects. However, since the KQ part for this head isolates attribute tokens from the context, this term does not contribute to the final output.

\textit{Two-step Cascade.} Here, we briefly discuss why $\W_{KQ}^{\text{rel}}$ requires two update steps (see App.~\ref{app:const} for details). 
In the first step, the gradient update for $\W_{KQ}^{\text{rel}}$ is \emph{zero} 
because \ICL sequences contain no relation tokens. However, $\W_{OV}^{\text{rel}}$ begins learning attribute-to-attribute 
associations (for shared attribute types) in this step. In the second step, these updated 
OV weights generate a nonzero gradient for $\W_{KQ}^{\text{rel}}$, which then 
learns to attend to attribute tokens in the context.

\vspace{-2mm}
\paragraph{Experimental Validation.} We present experimental evidence to corroborate our theoretical constructions, using a 1-layer 3-head attention-only model (see App.~\ref{app:expts} for details).

\textit{Validation of Pretraining Mechanism.} We first train the model using \Bio sequences (\cref{eq:const_pt_x}) with the next-token prediction objective. In \cref{fig:attn_pt}, we visualize the attention scores for each head across a \Bio sequence, and find that the heads specialize into distinct roles consistent with \cref{prop:pt}. We find that head $2$ attention score concentrates on the most recent subject token, and it outputs $\emb{\obj_{j_\star}}$, matching the role of  $g_{\text{subj}}(\X)$ in our construction, while head~$0$ attends to the most recent relation token, \textit{i.e.}, it outputs $\emb{\rel_{\ell_{\Ntr}}}$, consistent with $g_{\text{rel}}(\X)$. Further, in \cref{fig:cos_pt}, we compute the cosine similarity between the head outputs and theoretical outputs of the subject and relation heads specified by our construction. Specifically, we report the cosine similarity with $f_{\text{subj}}(\X)$, \textit{i.e.}, negative sum of all attributes not associated with the subject $\obj_{j_\star}$ (left subplot) and $f_{\text{rel}}(\X)$, \textit{i.e.}, sum of all attributes of type $\ell_{\Ntr}$ (right), averaged across several sequences. We find that head $2$ output closely matches $f_{\text{subj}}(\X)$, while head $0$ matches $f_{\text{rel}}(\X)$. We designate these heads as subject and relation heads, respectively. Together, these results present experimental validation of our construction for \Bio sequences. 

\begin{figure}[h!]
    \centering
    \includegraphics[width=0.95\linewidth]{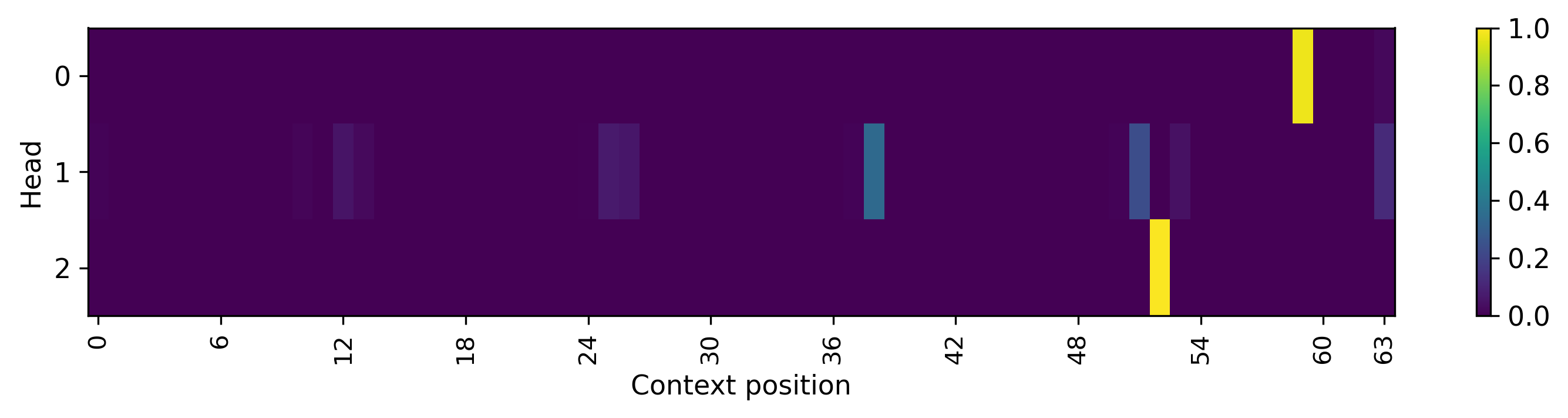}
    \caption{Attention scores for each head across a \Bio sequence at the end of pretraining. Head $2$ attends to the most recent subject, while head $0$ attends to the most recent relation token, as predicted by \cref{prop:pt}.}\vspace{-0.1in}
    \label{fig:attn_pt}
\end{figure}
\begin{figure}[h!]
    \centering
\includegraphics[width=0.73\linewidth]{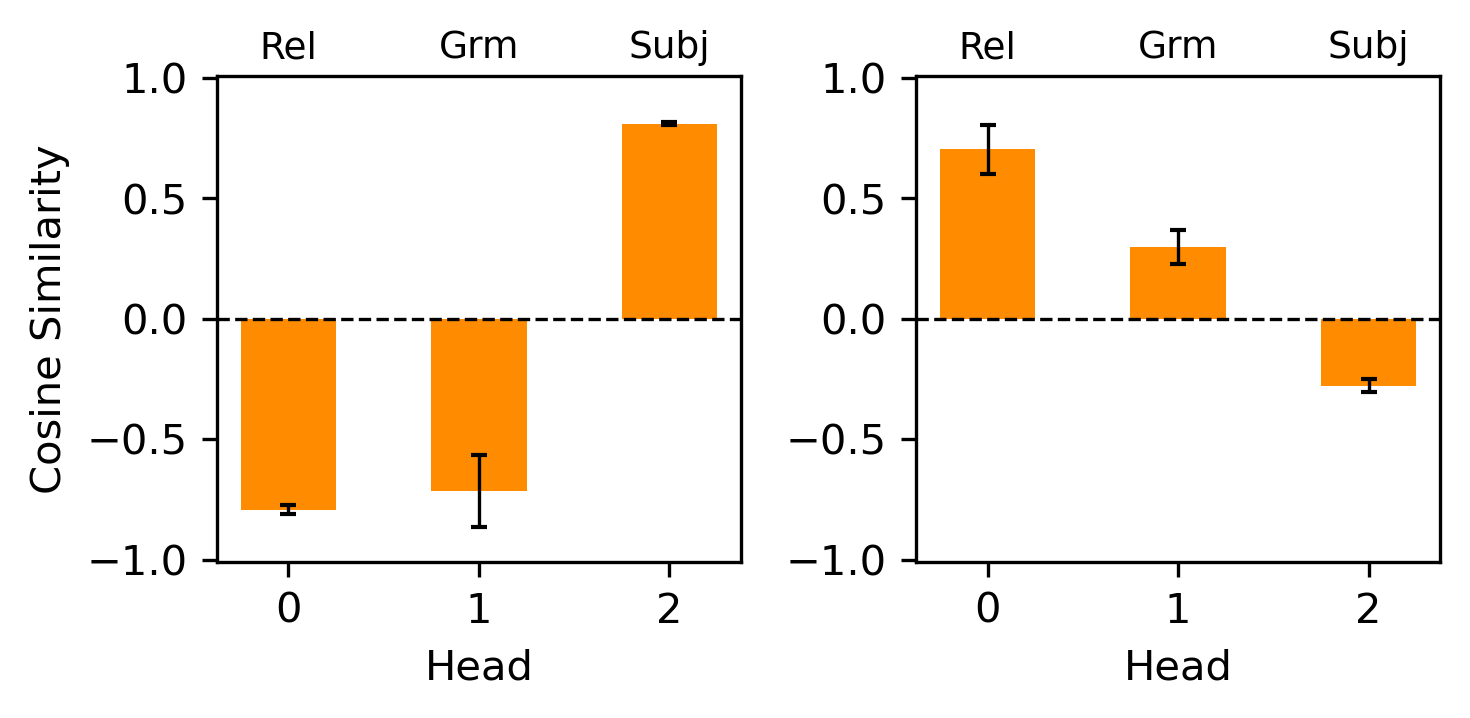}
    \caption{Cosine similarity between the actual outputs of each head and the outputs $f_{\text{subj}}(\X)$ (left) and $f_{\text{rel}}(\X)$ (right) from our construction in \cref{prop:pt}. Head $2$ output closely matches $f_{\text{subj}}(\X)$, while head~$0$ matches $f_{\text{rel}}(\X)$.} \vspace{-0.1in}
    \label{fig:cos_pt}
\end{figure}

\begin{figure}[h!]
    \centering
    \includegraphics[width=0.9\linewidth]{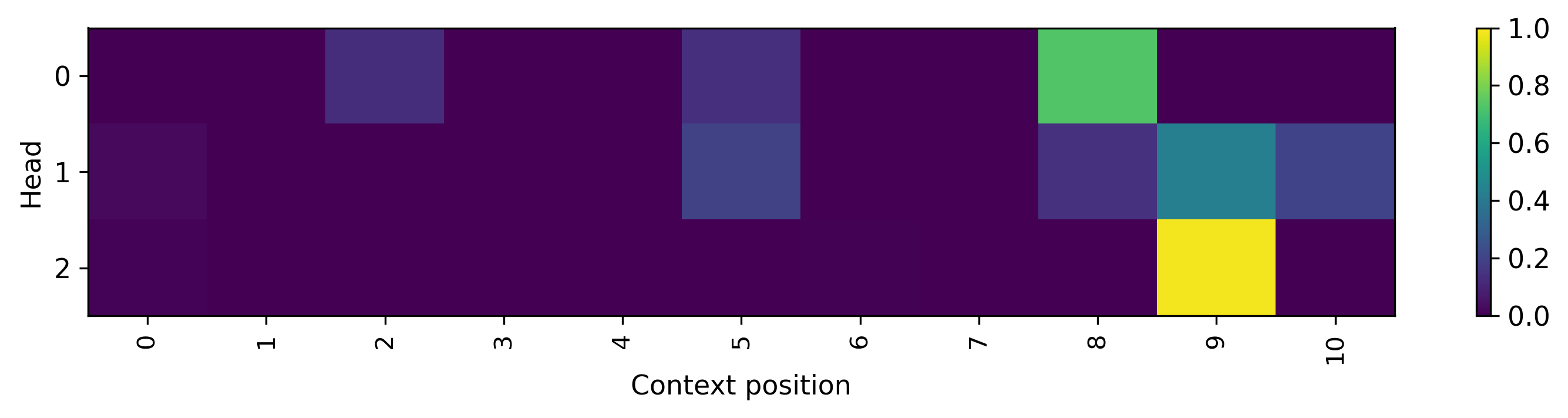}
    \caption{Attention scores for each head across an \ICL sequence at the end of finetuning. Head $2$ attends to the most recent subject, while head $0$ attends to the attribute tokens in the sequence, as predicted by \cref{prop:icl}.}\vspace{-0.1in}
    \label{fig:attn_icl}
\end{figure}

\begin{figure}[h!]
    \centering
    \includegraphics[width=0.73\linewidth]{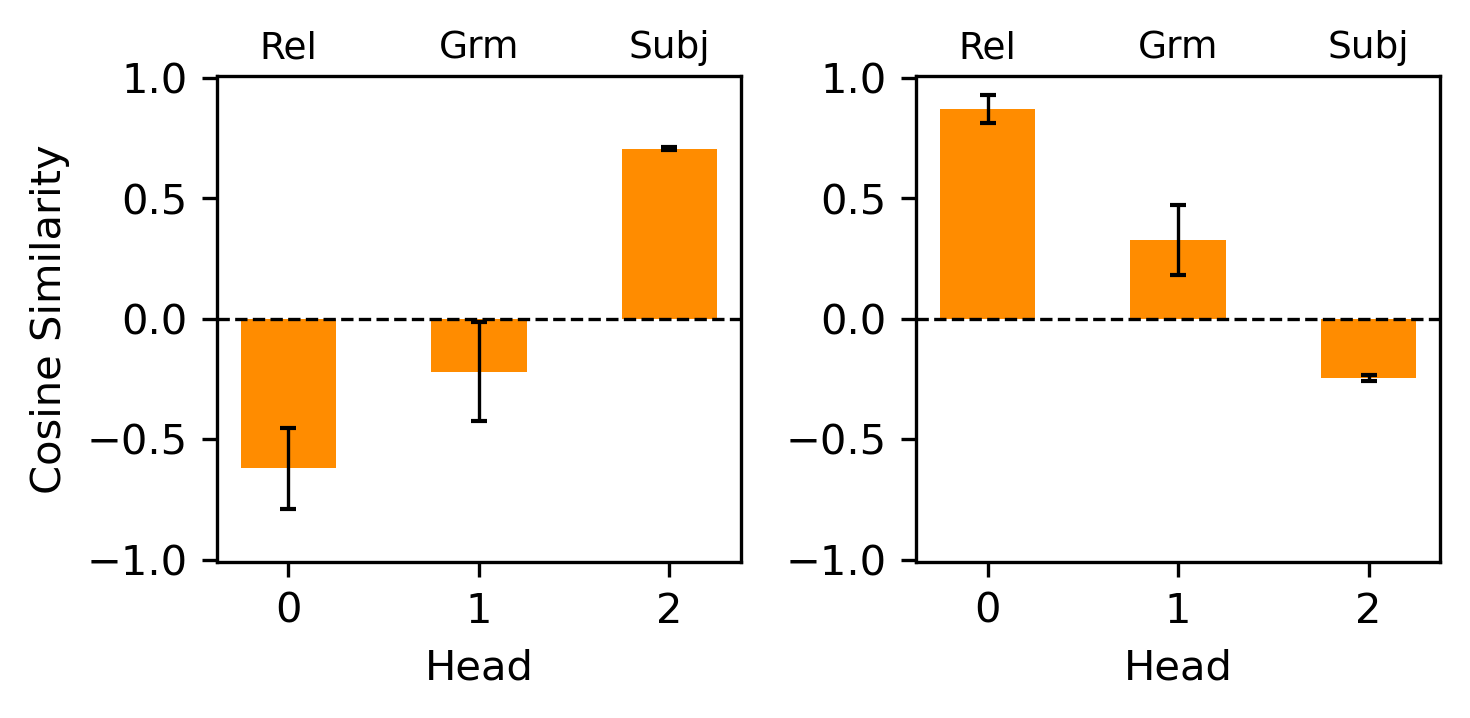}
    \caption{Cosine similarity between the actual outputs of each head and the outputs $f_{\text{subj}}(\X)$ (left) and $f_{\text{rel}}(\X)$ (right) from our construction in \cref{prop:icl}. Head $2$ output closely matches $f_{\text{subj}}(\X)$, while head~$0$ matches $f_{\text{rel}}(\X)$.}\vspace{-0.1in}
    \label{fig:cos_icl}
\end{figure}
\textit{Validation of ICL Mechanism.} We finetune the model on \ICL sequences (\cref{eq:const_ft_x}) with the last-token prediction objective, using $50\%$ of the total subjects (\cref{fig:attnonly_icl_acc} in the App. shows that the model generalizes on held-out subjects). Visualizing the attention scores on an \ICL sequence in \cref{fig:attn_icl} reveals that the model repurposes its heads for the new task, consistent with Prop. \ref{prop:icl}. We find that the head $2$ (subject head) attends to the most recent/query subject token, while head $0$ (relation head) attends to the attribute tokens in the context, \textit{i.e.}, it outputs a combination of attributes of type $\ell_\star$. Further, \cref{fig:cos_icl} confirms that the outputs of these heads closely match the theoretical outputs specified by our construction for \ICL sequences: head $0$ (subject head) outputs negative sum of all attributes not associated with the subject $\obj_{j_{\Nft+1}}$ (left subplot), while head $1$ (relation head) outputs the sum of all attributes of type $\ell_\star$ (right). 
\vspace{-2mm}

\begin{figure}[h!]
    \centering
    \includegraphics[width=0.55\linewidth]{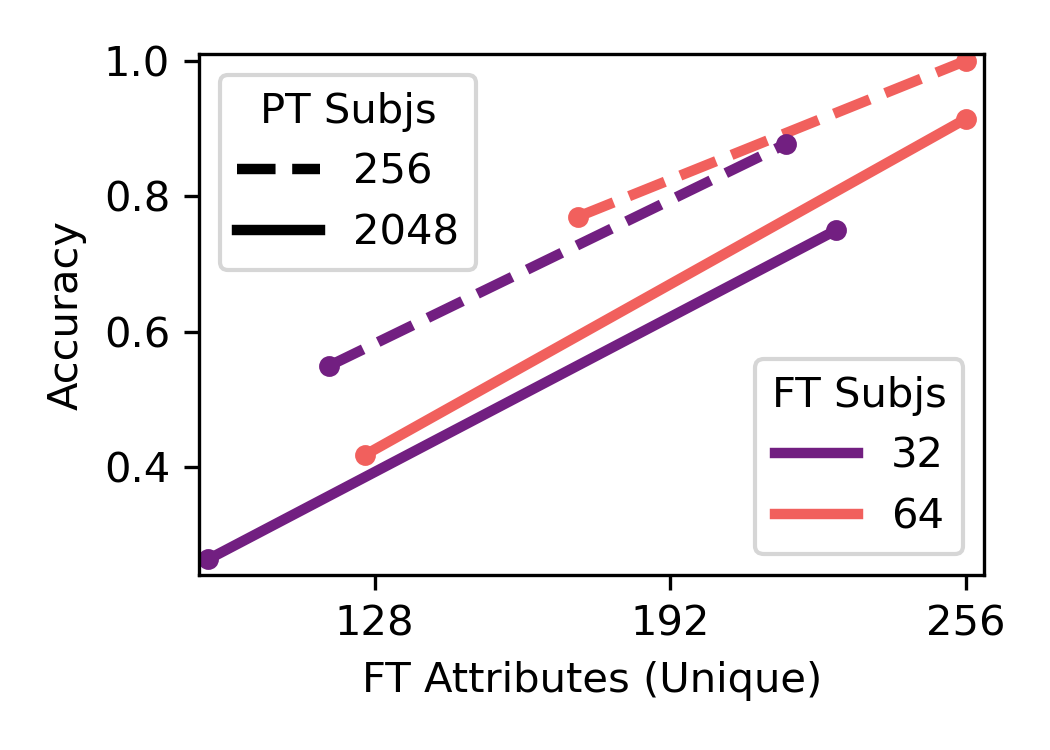}\vspace{-2mm}
    \caption{Effect of number of unique attributes seen during finetuning (with \ICL sequences) on the accuracy on \ICL sequences with held-out query subjects, across different number of PT and FT subjects. Performance improves as the number of unique attributes seen while finetuning increases.}\vspace{-0.15in}
    \label{fig:cov_icl}
\end{figure}

\begin{figure*}[b]
    \centering
    \includegraphics[width=0.3\linewidth]{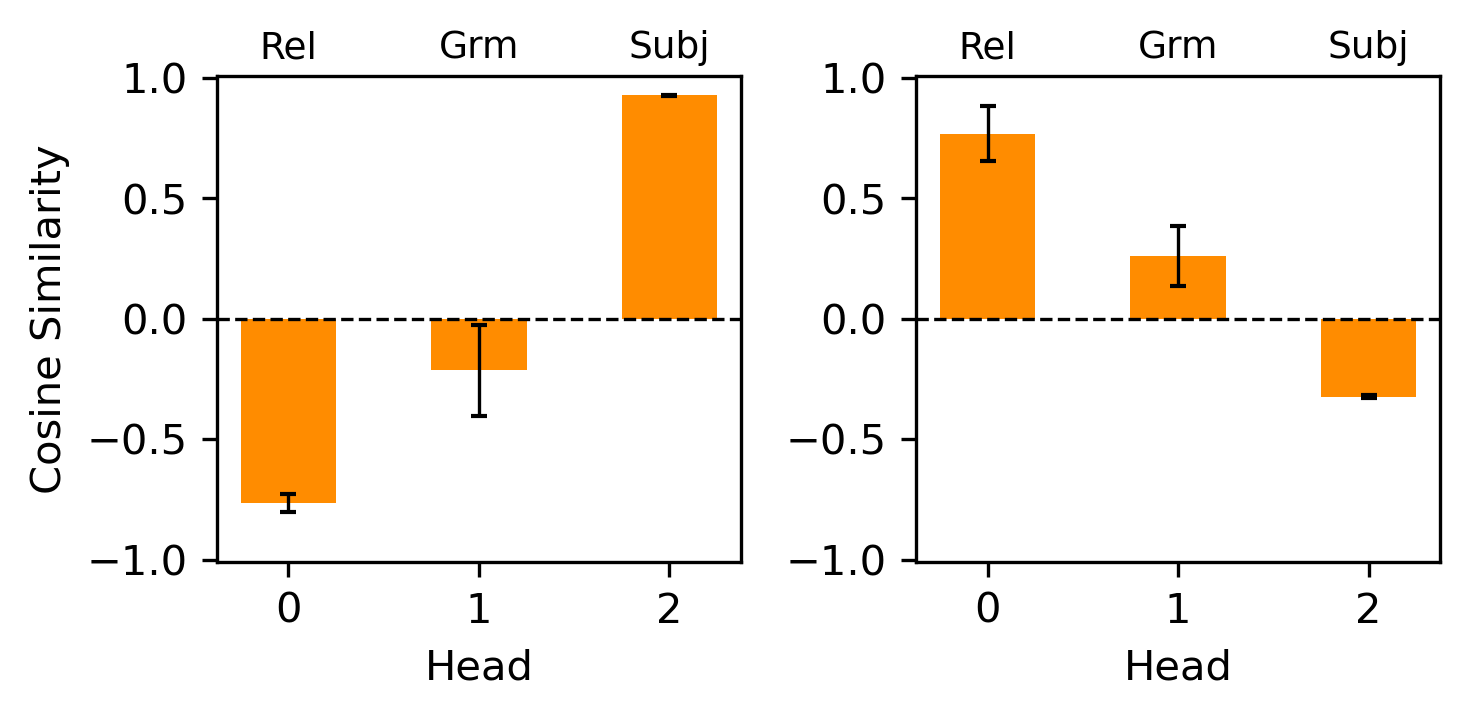}\includegraphics[width=0.4\linewidth]{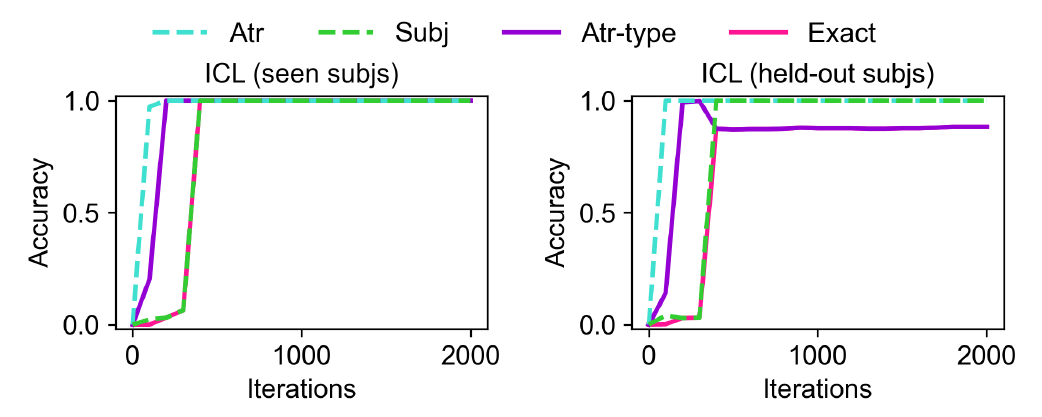}\\
    \includegraphics[width=0.3\linewidth]{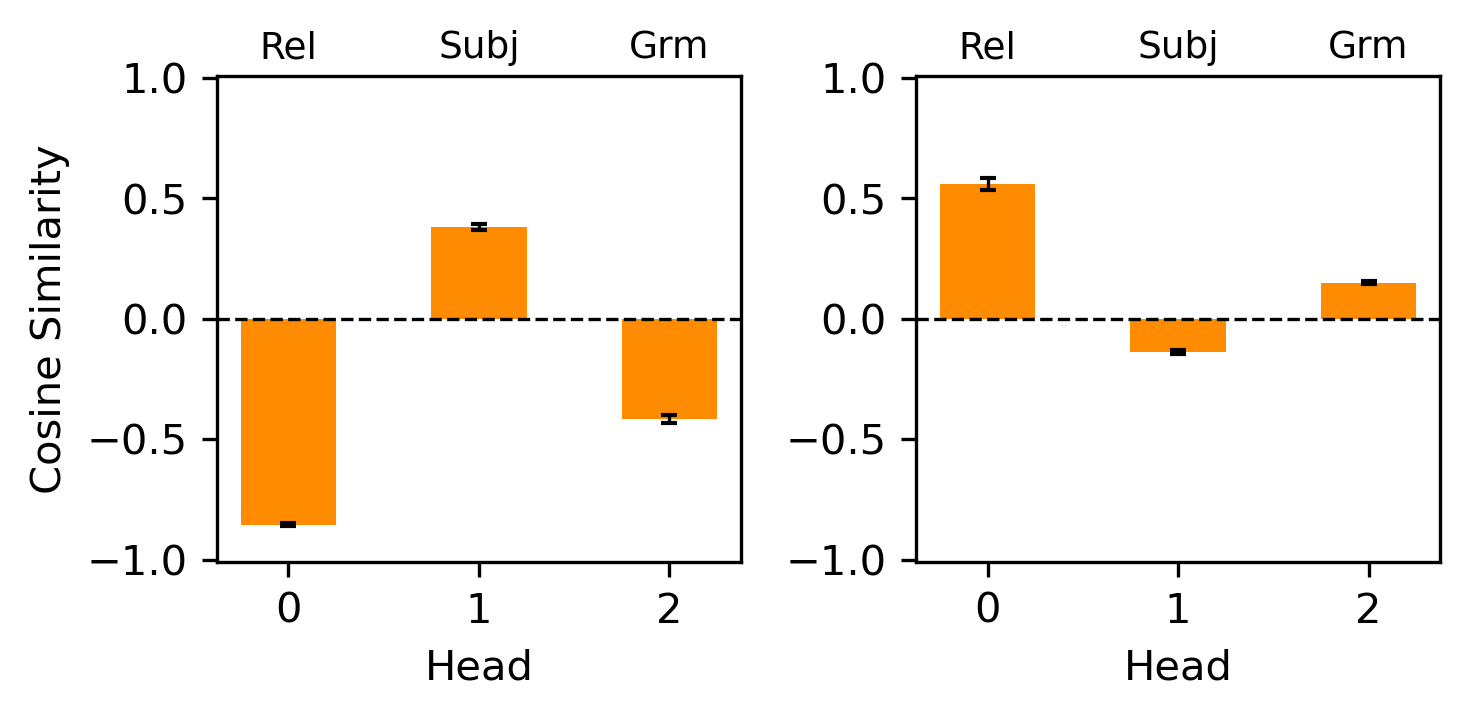}\includegraphics[width=0.4\linewidth]{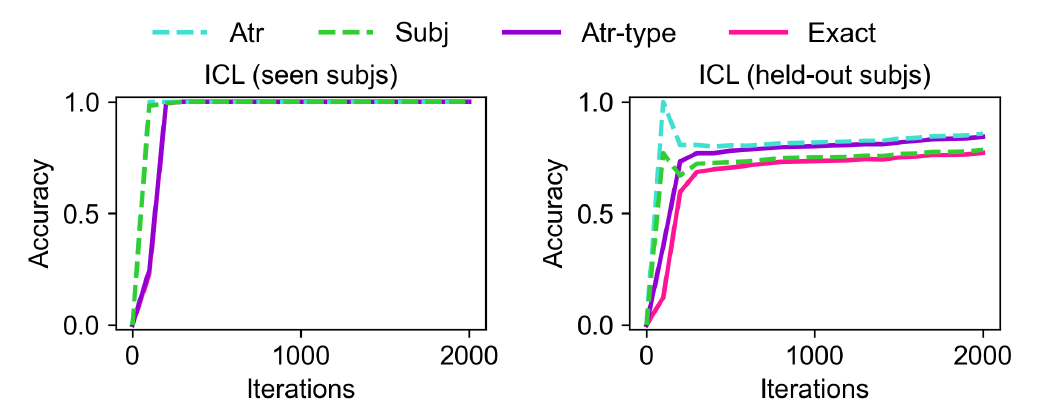}    
    \caption{Using 256 (top row) or 2048 (bottom row) PT subjects, (left) after pretraining, the cosine similarities of the output of each head with the subject and relation head outputs predicted by the construction in \cref{prop:pt}, and (right) after FT with 32 subjects, the performance on \ICL sequences with held-out query subjects. In the first case, the cosine similarities for subject and relation head are high, and after FT, the subject-match accuracy is high, while attribute-type match accuracy is lower because all attributes were not seen during FT. In the second case, the cosine similarities for subject and relation head are low, and consequently, after FT, the subject-match accuracy is low, as well as the attribute-type match accuracy is low because all attributes were not seen during FT.}
    \label{fig:256_2048}
\end{figure*}

\textit{Effect of Number of Attributes seen during Finetuning.} Based on our construction for \ICL sequences in \cref{prop:icl}, we note that generalization on held-out query subjects depends on how many unique attributes are seen during FT. To validate this effect, in \cref{fig:cov_icl}, we find that for a fixed number of subjects used during finetuning (32 or 64), increasing the number of unique attributes seen during finetuning improves the model's performance on \ICL sequences with held-out query subjects. This effect is consistent for $256$ or $2048$ subjects used while pretraining.

However, we also observe that for larger number of PT subjects, the performance is overall lower. To examine why this is the case, focusing on the setting with 32 FT subjects with larger number of unique attributes, in \cref{fig:256_2048}, we show i) after pretraining, the cosine similarities of the output of each head with the subject and relation head outputs predicted by the construction in \cref{prop:pt}, and ii) after FT, the performance on \ICL sequences with held-out query subjects. We observe the following. First, for 256 PT subjects, the cosine similarities for subject and relation head are high, and after FT, the subject-match accuracy is high, while attribute-type match accuracy is lower because all attributes were not seen during FT. On the other hand, for 2048 PT subjects, the cosine similarities for subject and relation head are low, and after FT, the subject-match accuracy is low, as well as the attribute-type match accuracy is low because all attributes were not seen during FT.

Interestingly, this suggests that representations learned during pretraining can affect downstream generalization performance after finetuning, even though in both cases, the model achieves perfect accuracy for factual recall as well as on \ICL sequences with seen subjects. We believe that a deeper investigation of this phenomenon including how to improve the representations learned after pretraining is an interesting direction for future work.

%% file: sections/conclusion.tex
\vspace{-3mm}
\section{Discussion and Conclusion}
\vspace{-1.5mm}
We studied contextual recall, a form of ICL that requires models pretrained to acquire knowledge about various subjects and associated attributes, to recall a specific attribute for a query subject by implicitly inferring the relevant attribute type based on in-context examples. Using a controlled synthetic framework, we showed that while pretraining only enables factual recall, targeted finetuning on tasks that require implicit inference, but are distinct from the evaluation ICL format, using a subset of subjects, leads to emergence of contextual recall capability on all subjects. Our results give insights into the complementary roles of pretraining and finetuning in enabling in-context reasoning involving learned knowledge. 

We believe that our synthetic data setup can be a useful testbed to study several questions. This includes examining the effects of using a mixture of pretraining and finetuning data instead of using them in a stage-wise manner, curriculum learning, non-uniform distributions over the subjects instead of uniform as we use, increasing the number of attribute types, and so on. A key question is to better understand when and how can the representations learned during pretraining be improved to achieve better generalization after finetuning. Another important direction is investigating how incorporating new knowledge via further finetuning might impact the learned contextual recall capability.

%% file: sections/app.tex
\section{Related Work}
\label{app:rel_work}
In this section, we discuss related work on studying ICL and factual recall in transformers using controlled/synthetic settings, as well as on task vectors for ICL.
\vspace{-1mm}
\paragraph{Factual Recall.} Recent studies have utilized controlled synthetic setups to systematically explore how transformers acquire and recall factual associations. \citet{allen2023physics} and \citet{zucchet2025language} employ synthetic biography datasets, where each entry contains information about an individual or subject in the form of multi-sentence paragraphs. Each sentence associates some fact or attribute with an individual or subject, via a sentence template that contains information about the relation or attribute type, such as a person's birthplace, in the linguistic structure. They consider a fixed set of templates for each attribute type. \citet{zucchet2025language} evaluate factual knowledge (\textit{i.e.}, correct attribute token prediction) using similar biography entries that are generated with templates from a held-out set, and identify a stage-wise learning dynamic where the model first learns to predict some attribute token, and then the correct attribute. We also observe similar stage-wise learning in our setup during pretraining in \cref{sec:pt_expts}. In \citet{allen2023physics}, the model is evaluated on question-answering (QA) formats, which constitutes a distribution shift from the biography-style format. Similar to our results in \cref{sec:ft_expts}, they find that while pretraining is insufficient, finetuning (on QA) with a subset of subjects enables generalization on held-out subjects. Crucially, however, their QA prompts still contain explicit relation or attribute type information (e.g., "What is [subject]'s city of birth?"), whereas our contextual recall task constitutes a more substantial distribution shift as the context does not contain explicit relation cues. Investigating whether we see similar results as in \cref{sec:ft_expts} for our synthetic setup with synthetic biography data used in these studies would be an interesting direction for future work.

In addition to these works, other studies have adopted abstracted setups with subject-grammar-attribute triplets to study factual recall \citep{nichani2024understanding,behnia2025factsstatsimpactspretraining}. \citet{nichani2024understanding} uses a data setup where the relation (or attribute type) is a single, dedicated token, while the rest of the tokens in the grammar sequence are sampled randomly. Using this setup, they show that a one-layer transformer model (with a self-attention layer followed by an MLP) succeeds on factual recall whenever either the total number of self-attention parameters or MLP parameters scales (up to log factors) linearly with the number of facts. In contrast, \citet{behnia2025factsstatsimpactspretraining} propose a 
setup to more faithfully capture the statistical structure in the grammar sequences, \textit{i.e.}, templates in the controlled language datasets in \citet{zucchet2025language,allen2023physics}. They model each template as a randomly sampled Markov chain, and omit the relation token (the facts are independent of the grammar), to study how sequence statistics affect generalization. Interestingly, they show that several phenomena observed with the controlled language datasets in \citet{zucchet2025language,allen2023physics} can be reproduced in this abstracted setup.

A key differentiator across these frameworks is how "relation" (or attribute type) information is represented. While in \citet{nichani2024understanding} the relation is a specific token, \citet{behnia2025factsstatsimpactspretraining} omits the relation token entirely. Our synthetic framework bridges these approaches: similar to \citet{behnia2025factsstatsimpactspretraining}, we model templates using Markov chains, but in contrast to their work, and similar to the other prior studies, we retain the relation information by associating the Markov chains with specific attribute types. 

For mechanistic analysis, we adopt an abstracted setup similar to \citet{nichani2024understanding} and compare our attention-only construction for factual recall (\cref{prop:pt}) with theirs. In both constructions, the relation heads perform a similar role: they attend to the (most recent) relation token in the context to output the sum of all attributes of the relevant type. However, the subject heads behave differently. While the construction in \citet{nichani2024understanding} uses subject heads to boost the attributes associated with the subject present in the context, our construction utilizes subject heads to suppress all attributes not associated with the subject. While we present experimental validation for our construction, we note that changing the initialization seed can lead to the model using a different mechanism, similar to their construction. 

As discussed in \cref{sec:intro}, while prior works focus on factual recall using explicit cues about the attribute type from the context, our work investigates contextual recall, a more challenging task that requires the model to infer the attribute type implicitly from in-context examples.
\vspace{-1mm}
\paragraph{In-Context Learning.} Several recent studies have leveraged controlled synthetic environments to analyze how Transformers learn in-context when trained from scratch. A common approach involves training models on well-defined function classes, most notably linear regression \citep{garg_icl, raventos2023pretraining, akyurek_icl,vonOswald_icl, ahn2023transformers, wu2023many} and Markov chains \citep{statistical_induction_heads, algorithmic_phases,rajaraman2024transformers, deora2025incontext}. These works often explore whether Transformers implement specific algorithms or functionalities—such as gradient descent \citep{vonOswald_icl, ahn2023transformers,bai2023transformers} or higher-order-algorithms \citep{fu2024transformers} for linear regression, and induction heads for copying tasks \citep{olsson2022context}, Markov chain-based setups \citep{statistical_induction_heads, rajaraman2024transformers, chen2024unveiling,dangelo2025selective}, settings that combine these \citep{bietti2023birth,chendistributional,kawata2025from}, as well as other controlled synthetic settings \citep{singh2024needs}. Further research \citep{raventos2023pretraining, algorithmic_phases, dual-operating-modes, singh2025strategycoopetitionexplainsemergence} has investigated task diversity, comparing \emph{task retrieval} and \emph{task learning} modes of ICL \citep{pan-etal-2023-context}. Additionally, research has explored transient dynamics to understand how these two modes evolve over the course of training \citep{singh2023transientnatureemergentincontext, carroll2025dynamicstransientstructureincontext, singh2025strategycoopetitionexplainsemergence}. Finally, a growing body of work also explores the training dynamics of in-context learning by examining the optimization dynamics of linear regression \citep{zhang2025trainingdynamicsincontextlearning, zhang2023trained, zhang2024context} in both one-layer linear attention and softmax attention models \citep{chen2024trainingdynamicsmultiheadsoftmax}, as well as for learning causal structures \citet{nichani2024how}.

A defining characteristic of the aforementioned studies is that the training and inference formats are identical; the model is evaluated on the same sequence structures it encountered during training. In contrast, our work on contextual recall introduces a significant prompt-distribution shift. While pretraining focuses on instilling factual knowledge in the model via explicit grammar-based cues, the ICL evaluation requires the model to transition a novel format necessitating inference of the relevant attribute type implicitly from the in-context demonstrations. Therefore, unlike standard ICL setups that focus on function induction, our framework requires the model to bridge the gap between structured knowledge acquisition and implicit in-context reasoning, as mentioned in \cref{sec:intro}.

In contrast to our work, \citet{xie2022an} investigate a distribution shift that is primarily compositional rather than structural. In their framework, the shift occurs when a model pretrained on long, continuous documents is prompted with a sequence of independent, i.i.d. examples. While they frame ICL as a statistical process of implicit Bayesian inference, our work provides a mechanistic perspective discussed in \cref{sec:rep_analysis} and \cref{sec:attn_only}. 

\paragraph{Task and Function Vectors in ICL.} Recent work studying ICL mechanisms in pretrained LLMs shows that certain directions in activation space encode the input–output relationship of the ICL task \citep{hendel2023incontextlearningcreatestask,todd2024function,yin2025attention}. These activation directions, referred to as task or function vectors, seem to encode abstract task information that is largely independent of the specific examples in the prompt, and enable causal interventions that modify model behavior \citep{liu_incontextvectors}. Other works use controlled synthetic setups to study the emergence of task vectors for ICL of some function classes \citep{yang2025taskvectorsincontextlearning,dong2026understanding}, how task vector encoding is correlated with task-level decoding \citep{han2025emergence}, as well as task vector disentanglement and composition for two-hop reasoning tasks, and geometry for tasks with numerical-valued continuous latent parameters \citep{hong2026latent}. Consistent with prior work, our analysis in \cref{sec:rep_analysis} suggests that finetuning drives the emergence of low-dimensional task vectors for the shared attribute type, which enables the model to do well on the contextual recall task. Beyond observing their existence, we make two additional contributions. First, we show that the separability and functional utility of these task vectors (measured through representation clustering and output steering) is directly governed by properties of the pretraining data, specifically the separation between attribute-specific Markov chains. Second, in \cref{sec:attn_only}, we provide a mechanistic explanation for how task vectors emerge through finetuning: the relation head's OV matrix 
learns attribute-to-attribute associations for shared attribute type, which induces the clustering of representations by attribute type.

\section{Additional Experimental Results and Details of Experimental Settings}
\label{app:expts}

\subsection{Details of Experimental Settings}

We use a 2-layer 1-head GPT-2 type decoder-only transformer model \citep{mingpt_karpathy} with embedding dimension $256$. We train the model with AdamW optimizer \citep{adamW} with learning rate $10^{-4}$, weight decay $0.001$, and batch size $64$, for both pretraining and finetuning.

For evaluation, we report four metrics: \emph{Exact Match} (predicted attribute matches the ground truth), \emph{Attribute-Type Match} (predicted attribute has the correct type), \emph{Subject Match} (predicted attribute belongs to the correct subject), and \emph{Attribute Match} (any attribute token is predicted). See \cref{fig:sequences} for an illustration.

For the experiments in \cref{sec:pt_expts}, we set $M=256,L=8,\Ntr=5,M_1=256,M_2=\dots=M_8=32,S=80$. We set the grammar-only subsequence probability $p_G=0.2$, and separation between Markov chains $\MCD\approx 0.5$. To control for $\MCD$, we first randomly generate a large pool of transition matrices, and then use greedy selection to curate a subset of transition matrices that are assigned to each attribute type. We pretrain the model for $20k$ iterations.

For the experiments in \cref{sec:ft_expts}, we use the same pretraining setting as in \cref{sec:pt_expts}, and set $N=\Nft=16$. We use $128$ subjects for finetuning, unless stated otherwise. For experiments with \ICLG sequences, $\Sft\sim\Unif{\{1,\dots,5\}}$. In all cases, we finetune for $2k$ iterations. In \cref{fig:subj_ft_plot}, we compare the best performance across finetuning iterations.

For \cref{fig:subj_ft_plot,fig:div_plot,fig:subj_plot}, the results are reported after averaging across two random initialization seeds.

For the experiments in \cref{sec:ablations}, the details are as follows. All results are reported at end of pretraining/finetuning. We consider $\MCD\approx0.1,0.3,0.5$. In \cref{fig:div_plot}, we use fixed $M=256$ and for $S=20$, we use $\Sft\sim\Unif{\{1,...,4\}}$, while for $S=40$ or $S=80$, we use $\Sft\sim\Unif{\{1,...,5\}}$. In \cref{fig:subj_plot}, we use fixed $S=80$, $\Sft\sim\Unif{\{1,...,5\}}$.

The results in \cref{fig:clustering} (bottom) are reported with the same setting as \cref{fig:ft_iclwgrm}. For the top plot, we only change $\MCD\approx 0.2$. We use $50$ sequences for each attribute type.

The settings for the experiments in \cref{sec:attn_only} and \cref{app:const} are as follows. We set $M=256,L=8,\Ntr=5,M_2=\dots=M_8=32,S=10,N=\Nft=5$ for both sections, while $M_1=32$ and $M_1=256$, respectively. We train an attention-only model with $d_h=256$ using AdamW optimizer with learning rate $0.001$, weight decay $0.001$, and batch size $64$, for both pretraining and finetuning. We use a cosine learning rate scheduler for pretraining. The results with $M=2048$ subjects use the same experimental setting, except we set $d_h=2048$. Unless stated otherwise, we pretrain for $5k$ iterations and finetune for $2k$ iterations. For the results for experimental validation in \cref{sec:attn_only}, we pretrain for $2k$ iterations and finetune for $1k$ iterations. 

\subsection{Additional Results}

\begin{figure}[h!]
    \centering
    \includegraphics[width=0.8\linewidth]{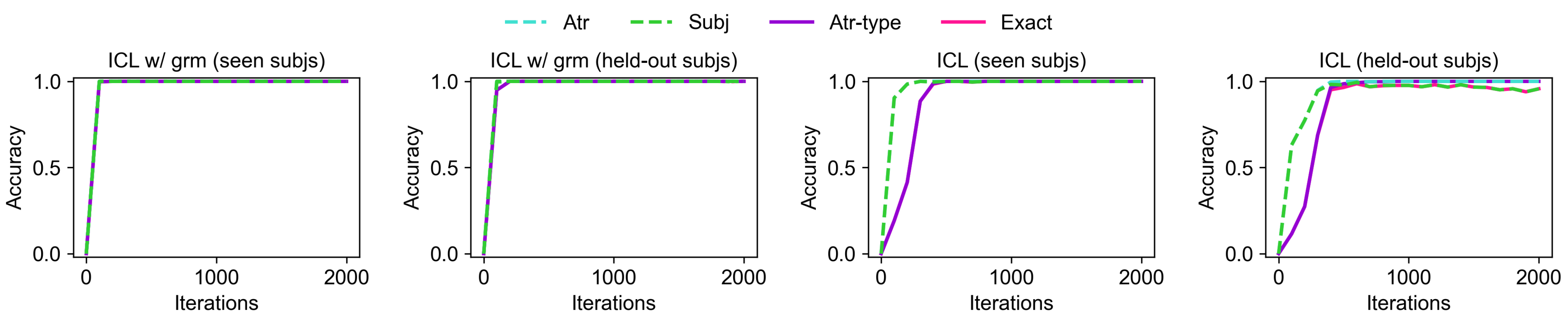}
    \caption{Performance of the model when finetuning with \ICLG sequences with a subset of subjects (same setting as \cref{fig:ft_iclwgrm}), on \ICLG sequences with $\Sft=1$ (left) and \ICL sequences (right) with seen or held-out subjects. Finetuning with \ICLG sequences enables out-of-distribution generalization on \ICL sequences with held-out subjects. Performance improves first on \ICLG and later on \ICL sequences.}\vspace{-0.1in}
    \label{fig:ft_iclwgrm_full}
\end{figure}

\begin{figure}[h!]
    \centering
    \includegraphics[width=0.3\linewidth]{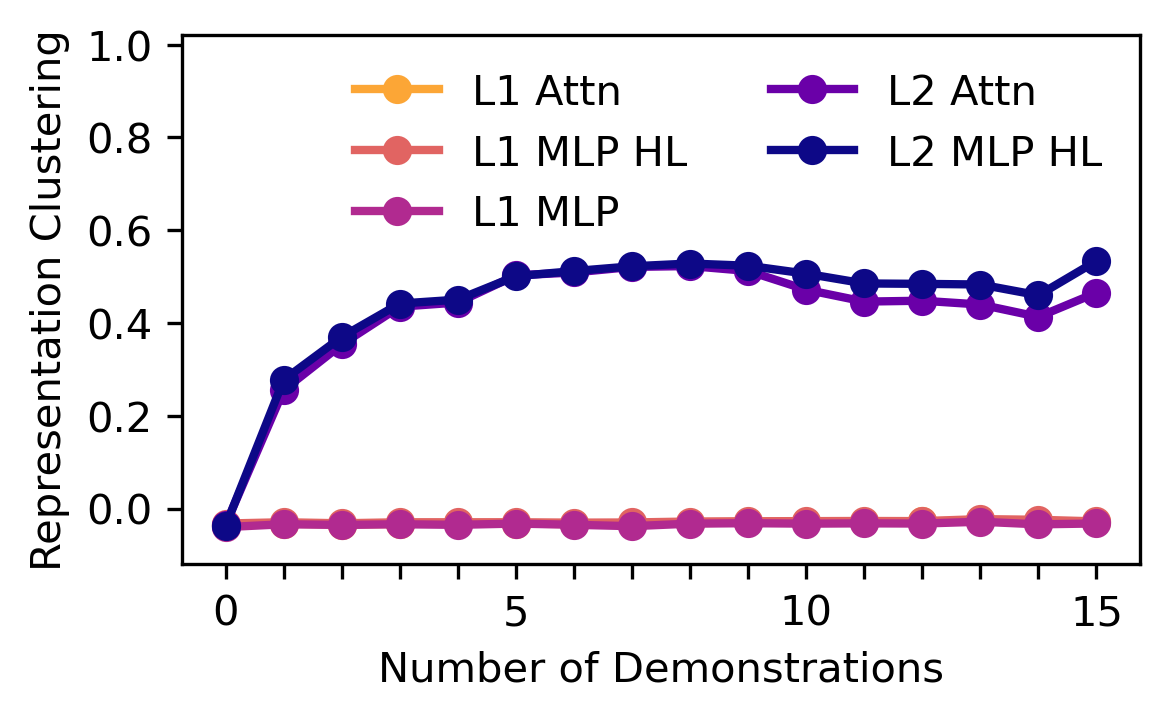}\hspace{1mm}
    \includegraphics[width=0.3\linewidth]{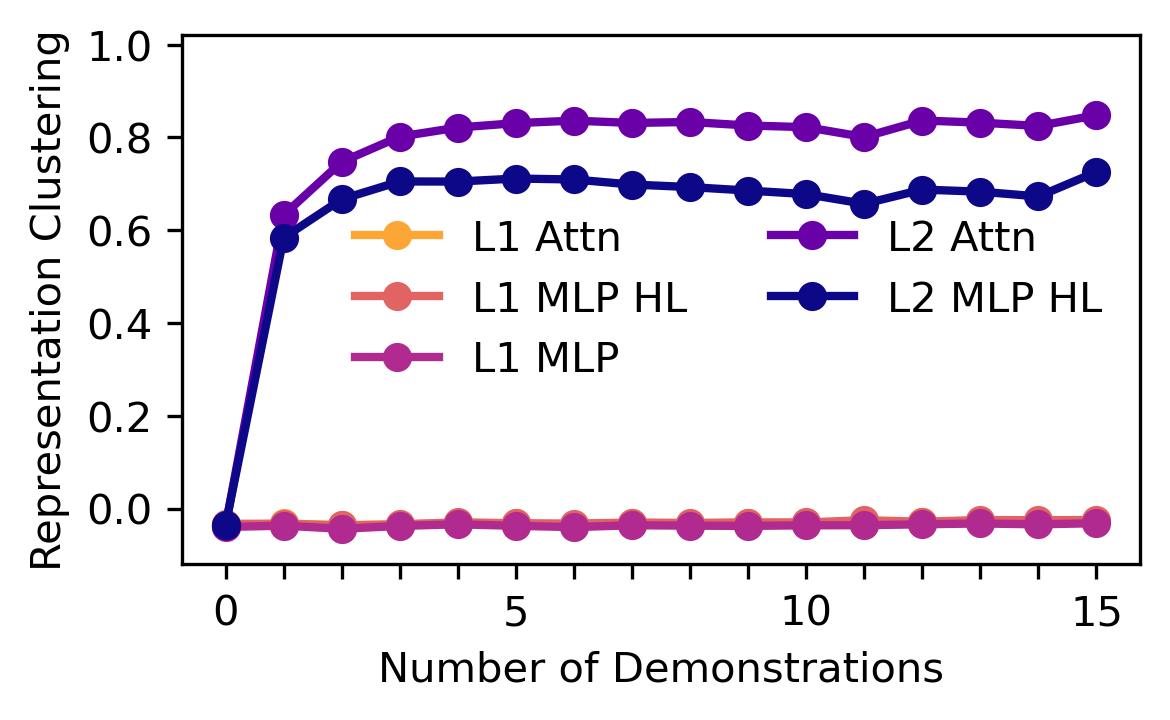}
    \caption{Comparison of clustering strength (see \cref{sec:rep_analysis} for details) using representations from different layers of the finetuned 2-layer 1-head transformer with \ICL sequences as inputs, as the number of demonstrations is increased (same setting as \cref{fig:clustering}) with $\MCD\approx0.2$ (left) and $\MCD\approx0.5$ (right). We find that layer-2 attention representations cluster most strongly, while layer-1 representations exhibit no clustering based on the attribute type information in the context. Further, using higher $\MCD$ while pretraining leads to stronger representation clustering after finetuning.}\vspace{-0.1in}
    \label{fig:clustering_layers}
\end{figure}

\begin{figure}[h!]
    \centering
    \includegraphics[width=0.8\linewidth]{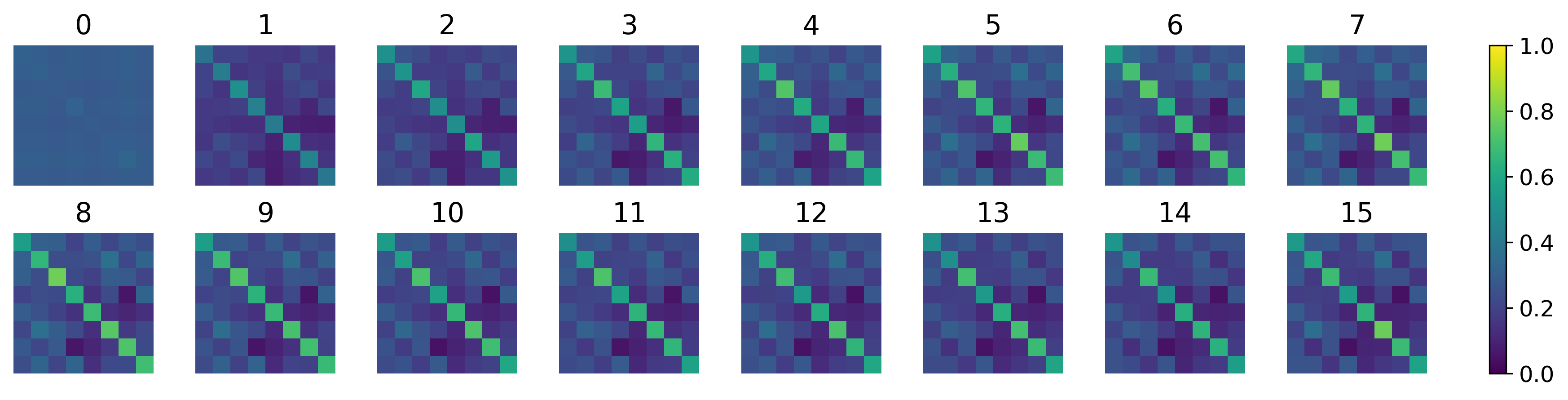}
    \includegraphics[width=0.8\linewidth]{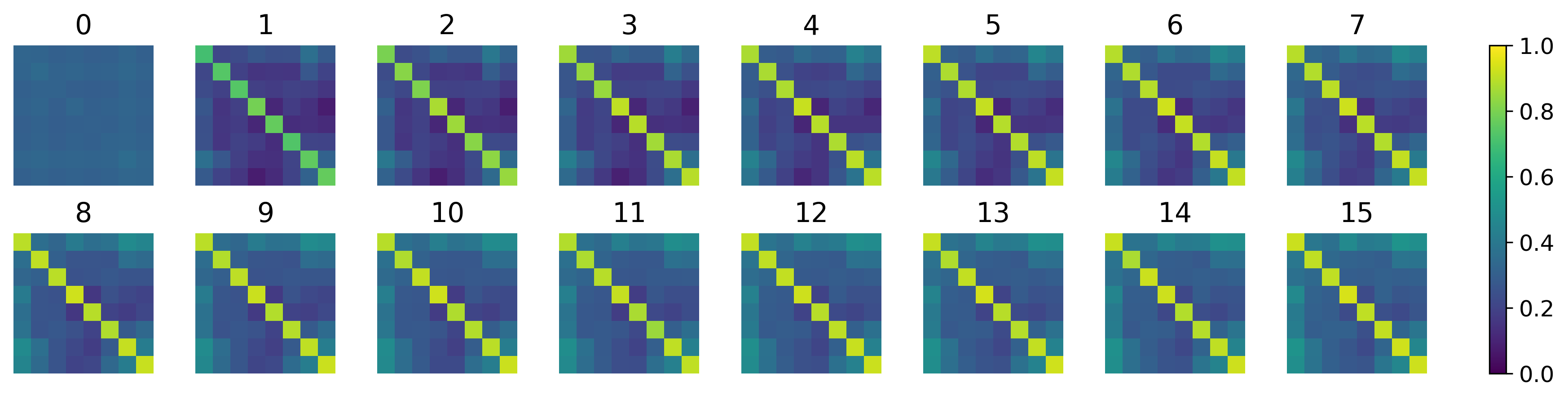}
    \caption{Each subfigure visualizes the cosine similarity for inter- and intra-task representations (from layer-2 attention layer of the finetuned model), $\bar{C}^t(\ell,\ell')$ (see \cref{sec:rep_analysis} for details) across attribute types $\ell,\ell'\in[L]$, averaged over $50$ sequences for each attribute type (same setting as \cref{fig:clustering}; top: $\MCD\approx 0.2$, bottom: $\MCD\approx 0.5$). We find that as the number of demonstrations is increased (from $0$ to $15$), the representations of \ICL sequences with the same attribute type get clustered together, and the clustering is stronger for higher $\MCD$ (\textit{i.e.}, more separated attribute-specific Markov shains).}\vspace{-0.1in}
    \label{fig:clustering_full}
\end{figure}

\begin{figure}[h!]
    \centering
    \includegraphics[width=0.4\linewidth]{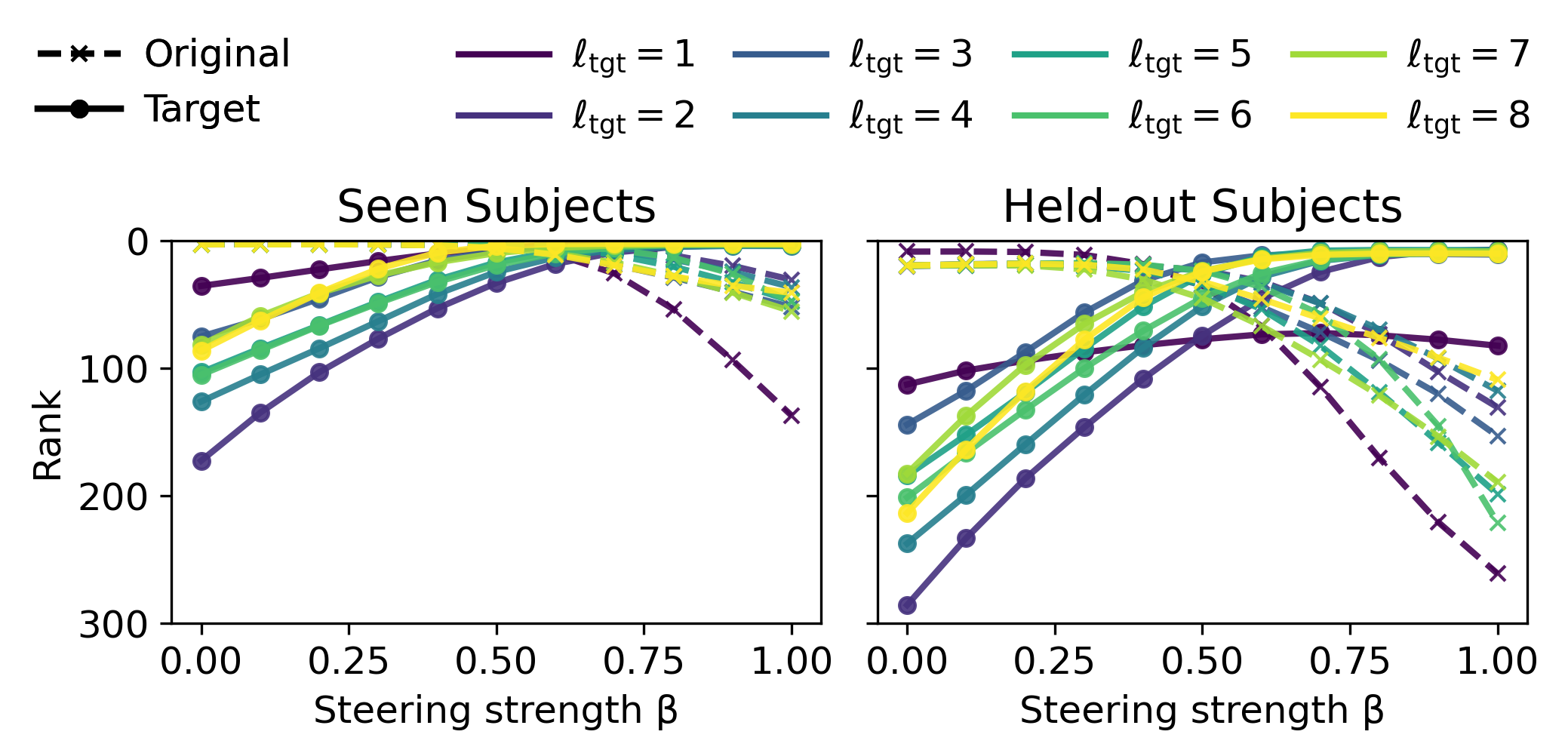}\hspace{1mm}
    \includegraphics[width=0.4\linewidth]{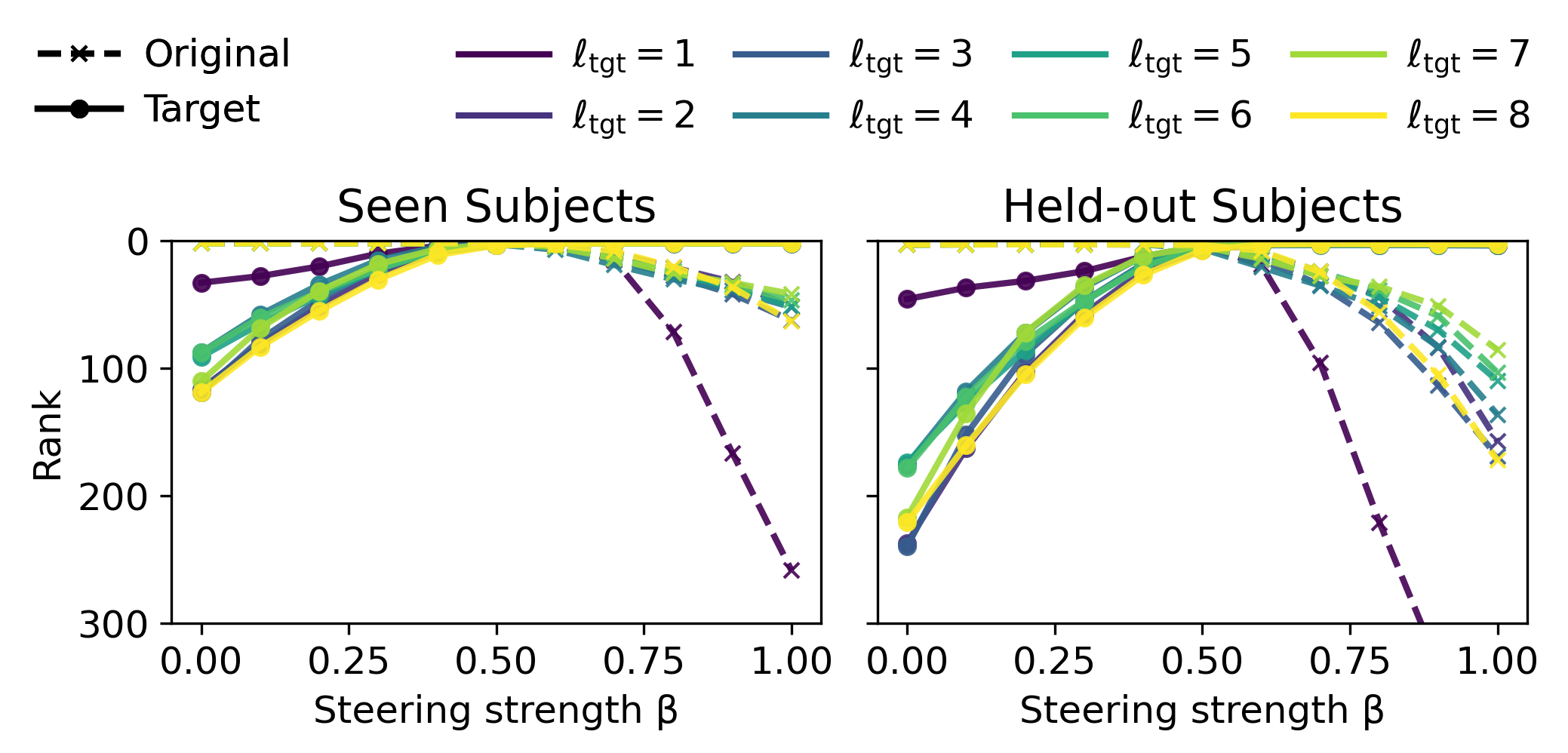}
    \caption{Comparison of rank of the original and target attribute type when steering the model output using the averaged layer-2 attention representations of the finetuned 2-layer 1-head transformer with \ICL sequences as inputs at the last token position (see \cref{sec:rep_analysis} for details) with $\MCD\approx0.2$ (left) and $\MCD\approx0.5$ (right), for seen and held-out FT subjects, for each target attribute type. We find that higher $\MCD$ used while pretraining leads to stronger task vectors and better steering with lower $\beta$ values. }\vspace{-0.1in}
    \label{fig:steering_all}
\end{figure}

\subsection{Details of Representational Analysis}
Consider \ICL sequences, which are of the form
\begin{align*}
    X_\ell= [\obj_{j_0}, \atr^\ell_{j_0}, \sep, \obj_{j_1}, \atr^\ell_{j_1}, \sep,\dots, \obj_{j_{N+1}}],
\end{align*}
for a fixed attribute type $\ell$. We sample $K$ such sequences for each $\ell\in[L]$, denoted by $X^k_\ell$. Also, define $X_{\ell,t}:= [\obj_{j_0}, \atr^\ell_{j_0}, \sep, \obj_{j_1}, \atr^\ell_{j_1}, \sep,\dots, \obj_{j_{t+1}}]$, where $t\in[N]$ denotes the number of demonstrations.  

Let $f_i(\cdot)$ denote the model's representation at layer $i$. Consider fixed $i,t$, and for $\ell, \ell' \in [L]$, define
\begin{align*}
\bar{C}^t_i(\ell,\ell')
= \frac{1}{K^2} \sum_{k,k'} 
\cos\!\big( f_i(X^{k}_{\ell,t}),\, f_i(X^{k'}_{\ell',t}) \big).
\end{align*}

Next, the clustering strength is quantified as follows. For fixed $i,t$, we subsume these and define $v^k_\ell := f^t_i(X^k_{\ell})$. Each representation is assigned cluster label based on the attribute type $\ell$. Define the cosine distance $d(v,w) := 1 - \cos(v,w)$.
For each point $v^k_\ell$, define the intra-cluster dissimilarity
\[
a(k,\ell)
= \frac{1}{K - 1}
\sum_{k'\ne k}
d(v^k_\ell, v^{k'}_\ell),
\]
and the nearest other-cluster dissimilarity
\[
b(k,\ell)
= \min_{\ell' \ne \ell}
\left(
\frac{1}{K} \sum_{k'} d(v^k_\ell, v^{k'}_{\ell'})
\right).
\]
The silhouette value for a sample is
\[
s(k,\ell)
= \frac{b(k,\ell) - a(k,\ell)}{\max\{\,a(k,\ell),\, b(k,\ell)\,\}}.
\]
Finally, the silhouette score for layer $i$ after $t$ demonstrations are seen is
\begin{align*}
\bar{S}^t_i
= \frac{1}{KL} \sum_{\ell} \sum_{k} s(k,\ell).
\end{align*}

\section{Analysis for the Attention-Only Model}
\label{app:const}
\subsection{Constructions for \Bio and \ICL Sequences} 
Consider \Bio sequences of the form
\begin{align}
\X \!=\! [\emb{\obj_{j_\star}},\emb{\gr_1},..., \emb{\rel_{\ell_1}},...,
\emb{\sep},\emb{\atr_{j_\star}^{\ell_1}},\emb{\obj_{j_\star}},...,\emb{\sep}],\label{eq:const_pt_x_}
\end{align}
where 
the correct last token prediction is $\emb{\atr_{j_\star}^{\ell_{\Ntr}}}$. The following result shows that there exists an attention-only model that always gives correct predictions at any position.
\begin{proposition}\label{prop:pt_}
    Consider the input $\X$ in \cref{eq:const_pt_x_}, and a 1-layer 3-head attention-only model with the weights
    \begin{align*}
\W_{KQ}^{\text{rel}}&=\beta\left(\sum_{\ell}\emb{\rel_\ell}+\pe\right)\emb{\sep}^\top, \quad \W_{OV}^{\text{rel}}=\sum_{\ell}\sum_{j}\emb{\uatr_{j}^\ell}\emb{\rel_{\ell}}^\top,\\
\W_{KQ}^{\text{subj}}&=\beta(\sum_j\emb{\obj_j})\emb{\sep}^\top, \quad \W_{OV}^{\text{subj}}=-\sum_j\left(\sum_{j'\neq j,\ell}\emb{\atr_{j'}^\ell}\right)\emb{\obj_j}^\top,\\
\W_{KQ}^\text{grm}&=\beta\Big(\sum_{\ell,j}\emb{\uatr_j^\ell}\emb{\uatr_j^\ell}^\top\!+\!\sum_j\emb{\obj_j}\emb{\obj_j}^\top\!+\!\sum_\ell(\emb{\rel_\ell}\!+\!\emb{\sep}\!+\!\pe)\emb{\rel_\ell}^\top\!+\!\Big(\sum_\ell\emb{\rel_\ell}\!+\!\emb{\sep}\!+\!\pe\Big)\sum_{\grm}\emb{\grm}^\top\Big)\\ 
\W_{OV}^\text{grm}&=\sum_j\sum_\ell\emb{\obj_j}\emb{\atr_j^\ell}^\top\!+\!\Big(\sum_{\grm}\emb{\grm}\!+\!\sum_\ell\emb{\rel_\ell}\Big)\Big(\sum_j\emb{\obj_j}\!+\!\emb{\sep}\Big)^\top\!+\!\Big(0.5\sum_{\grm}\emb{\grm}\!+\!\emb{\sep}\Big)\sum_\ell\emb{\rel_\ell}^\top,
\end{align*}
where $\pe:=\sum_{i=1}^{S+2}\emb{p_{-i+1}}$.
    Then, when the scalar $\beta\to\infty$, it always gives the correct prediction for any token position. 
\end{proposition}
\begin{proof}
We present a construction for a $3$-head model. For simplicity, we present the proof for the last subsequence, \textit{i.e.}, $t\in\{\Tau-S-2,\dots,\Tau\}$, but it can be easily extended to other cases where $t<\Tau-S-2$ as well. Hereafter, we assume that $t\geq \Tau-S-2$.

At a high level, we use two heads, which we call the \textbf{relation} and \textbf{subject} heads, to get the output following $\emb{\sep}$, \textit{i.e.}, $t=\Tau$, and the third \textbf{grammar} head for other cases $t<\Tau$. Let $\mathcal{T}:=\{\Tau-S,\dots,\Tau-2\}$.

Consider the \textbf{relation} head. Then, when $\beta\to\infty$, the outputs of this head are as follows. When $t=\Tau$, $g_{\text{rel}}(\X)=\emb{\rel_{\ell_{\Ntr}}}$, \textit{i.e.}, the most recent relation token in the sequence, and $f_\text{rel}(\X)=\sum_{j,\ellN}\emb{\uatr_j^{\ellN}}$, \textit{i.e.}, all attributes of type $\ell_{\Ntr}$.

On the other hand, when $t<\Tau$, $f^t_\text{rel}(\X)=\tfrac{1}{t}\sum_{j,\ellN}\emb{\uatr_j^{\ellN}}$.

Consider the \textbf{subject} head. Then, when $\beta\to\infty$, the outputs of this head are as follows. When $t=\Tau$, $g_\text{subj}(\X)=\emb{\obj_{j_\star}}$, the subject token in the sequence, and $f_\text{subj}(\X)=-\Big(\sum_{\uatr}\emb{\uatr}-\sum_{\ell}\emb{\atr_{j_\star}^\ell}\Big)$, \textit{i.e.}, negative of all attributes that are not associated with subject $\obj_{j_\star}$.

On the other hand, when $t<\Tau$, $f^t_\text{subj}(\X)=-\tfrac{1}{t}\Big(\sum_{\uatr}\emb{\uatr}-\sum_{\ell}\emb{\atr_{j_\star}^\ell}\Big)$.

Consider the \textbf{grammar} head. Let $g_h^t=g^t_h(\X)$ and similarly $f_h^t:=f^t_h(\X)$. The outputs in this case, when $\beta\to\infty$, are as follows:

\begin{itemize}
\item $t=\Tau-S-2$: $g_{\text{grm}}=\emb{\atr_{j_\star}^{\ell_{N-1}}}$, $f_{\text{grm}}=\emb{\obj_{j_\star}}$
\item $t=\Tau-S-1$: $g_{\text{grm}}=\emb{\obj_{j_\star}}$, $f_{\text{grm}}=\sum_{v\in\grmb\cup\relb}\emb{v}$
\item $t\in\mathcal{T}$: $g_{\text{grm}}\!=\!\begin{cases*}
    0.5(\emb{\sep}\!+\!\emb{\rel_{\ell}}), & if $\xb_{t-S-2:t}\!\in\!\relb$\\
    \emb{\sep}, & if $\xb_{t-S-2:t}\!\notin\!\relb$
\end{cases*}$, \\$f_{\text{grm}}\!=\!
\begin{cases*}
    \sum_{v\in\grmb\cup\relb}\emb{v}, & if $g_{\text{grm}}\!=\!\emb{\sep}$\\ 0.75\sum_{v\in\grmb}\emb{v}+0.5\sum_{v\in\relb}\emb{v}+0.5\emb{\sep}, & if $g_{\text{grm}}\!=\!0.5(\emb{\sep}+\emb{\rel_{\ell_{\Ntr}}})$
\end{cases*}$
\item $t=\Tau-1$: $g_{\text{grm}}=\emb{\rel_{\ell_{\Ntr}}}$, $f_{\text{grm}}=\emb{\sep}+0.5\sum_{v\in\grmb}\emb{\grm}$
\item $t=\Tau$: $g_{\text{grm}}=\frac{1}{\Tau}\sum_t\xb_t$, $f_{\text{grm}}=\tfrac{\Ntr-1}{\Tau}\emb{\obj_{j_\star}}+\tfrac{\Ntr}{\Tau}\Big(1.5\sum_{v\in\grmb}\emb{v}+\sum_\ell\emb{\rel_\ell}+\emb{\sep}\Big)$
\end{itemize}

Combining the outputs of the individual heads, the final output in different cases is as follows. 
\begin{align*}
    v^*=\begin{cases*}
        \atr_{j_\star}^{\ell_{\Ntr}},& if $t=\Tau$,\\
        \sep, & if $t=\Tau-1$,\\
        v\sim\Unif{\grmb\cup\relb}, & if $t\in\mathcal{T}$, $\nexists j\in\{t-S-2,...,t\}, \xb_j\in\relb$,\\
        v\sim\Unif{\grmb}, & if $t\in\mathcal{T}$, $\exists j\in\{t-S-2,...,t\}, \xb_j\in\relb$,\\
        v\sim\Unif{\grmb\cup\relb}, & if $t=\Tau-S-1$,\\
        \obj_{j_\star}, & if $t=\Tau-S-2$.
    \end{cases*}
\end{align*}

\end{proof}

Next, consider \ICL sequences of the form
\begin{align}
\X=[\emb{\obj_{j_1}},\emb{\sep},\emb{\atr^{\ell_\star}_{j_1}},...,\emb{\obj_{j_{\Nft+1}}},\emb{\sep}],\label{eq:const_ft_x_}
\end{align}
where the correct last token prediction is $\emb{\atr^{\ell_\star}_{j_{\Nft+1}}}$. The following result shows that there exists an attention-only model that always gives correct predictions for these sequences.
\begin{proposition}\label{prop:icl_}
    Consider the input $\X$ in \cref{eq:const_ft_x_}, and a 1-layer 3-head attention-only model with the weights
    \begin{align*}
\W_{KQ}^\text{subj}&=\beta\Big(\sum_j\emb{\obj_j}+\emb{p_{-1}}\Big)\emb{\sep}^\top,\quad \W_{OV}^{\text{subj}}=-\sum_j\left(\sum_{j'\neq j,\ell}\emb{\atr_{j'}^\ell}\right)\emb{\obj_j}^\top,\\
\W_{KQ}^\text{rel}&=\beta\sum_{\ell}\sum_{\uatr\in\uatrb'_\ell}\emb{\uatr}\emb{\sep}^\top,\quad \W_{OV}^\text{rel}=\sum_{\ell}\Big(\sum_{\uatr\in\uatrb'_\ell}\emb{\uatr}\Big)\Big(\sum_{\uatr\in\uatrb'_\ell}\emb{\uatr}+\emb{\rel_\ell}\Big)^\top,
    \end{align*}
    for the subject and relation heads, and the same weights for the grammar head as in \cref{prop:pt_}. Here, for each attribute type $\ell$, $\uatrb'_\ell:=\cup_{j\in\objb'}\{\atr_j^\ell\}$ denotes the set of unique attributes for the subjects in $\objb'\subset \objb$ seen during finetuning. Then, when the scalar $\beta\to\infty$ and $\uatrb'_\ell=\uatrb$, this model always gives the correct prediction for the last token position. 
\end{proposition}
\begin{proof}
We present a construction for a $3$-head model, with minimal changes to the construction in the proof of \cref{prop:pt}. Specifically, the \textbf{grammar} head is unchanged. At a high level, the \textbf{subject} head attends to the last/query subject and maps it to the negative of the sum of attributes not associated with it, while the \textbf{relation} head attends to the attributes in the context, and maps them to all attributes of the same type $\ell_\star$.

Consider the \textbf{subject} head. Using the aforementioned weights, when $\beta\to\infty$, $g_\text{subj}(\X)=\emb{\obj_{j_{\Nft+1}}}$, \textit{i.e.}, the query subject, and $f_\text{subj}(\X)=-\Big(\sum_{\uatr}\emb{\uatr}-\sum_{\ell}\emb{\atr_{j_{\Nft+1}}^\ell}\Big)$.

Consider the \textbf{relation} head. Using the aforementioned weights, when $\beta\to\infty$, $g_\text{rel}(\X)=\tfrac{1}{\Nft}\sum_i\emb{\atr_{j_i}^{\ell_\star}}$, \textit{i.e.}, the average of the attributes that appear in the context, and $f_\text{rel}(\X)=\sum_{\uatr\in\uatrb_{\ell_\star}}\emb{\uatr}$.

Combining the outputs of the individual heads, the model output is
\begin{align*}
    f(\X)=\sum_{\uatr\in\uatrb_{\ell_\star}}\emb{\uatr}-\Big(\sum_{\uatr}\emb{\uatr}-\sum_{\ell}\emb{\atr_{j_{\Nft+1}}^\ell}\Big)+\tfrac{1}{\Tau}\sum_{i\in\calM''}\emb{\obj_{i}}+\tfrac{\Nft}{\Tau}\sum_{v\in\grmb\cup\relb}\emb{v},
\end{align*}
where $\calM''\subseteq [M]$. Then, the final prediction $v^*=\atr_{j_{\Nft+1}}^{\ell_\star}$.
\end{proof}

\subsection{Optimization Dynamics for Finetuning}
We state the formal version of \cref{th:ft_dynamics} below followed by its proof.
\begin{theorem}
    Consider the two-head attention-only model defined in \cref{eq:attnonly}, initialized with the weights from the construction for \Bio sequences in \cref{prop:pt}:
    \begin{align*}
\W_{KQ}^{\text{subj}}=\beta\Big(\sum_j\emb{\obj_j}\Big)\emb{\sep}^\top,&
\qquad
\W_{OV}^{\text{subj}}=-\sum_j\Big(\sum_{j'\neq j,\ell}\emb{\atr_{j'}^\ell}\Big)\emb{\obj_j}^\top,\\
\W_{KQ}^{\text{rel}}=\beta\Big(\sum_{\ell}\emb{\rel_\ell}+\pe\Big)\emb{\sep}^\top,&\qquad\W_{OV}^{\text{rel}}=\sum_{\ell}\Big(\sum_{j}\emb{\uatr_{j}^\ell}\Big)\emb{\rel_{\ell}}^\top.
    \end{align*}
    Assume that $\mathcal U_\ell'=\mathcal U_\ell\,\forall \ell\in[L]$, $M\gg mL,M>M_{\text{ft}}\gg U_L$, $\Tau\leq S+2$. Consider 
   last-token prediction on ICL sequences of the form
\begin{align}\label{eq:icl_th}
\X = [\xb_1,\dots,\xb_\Tau]=[\emb{\obj_{j_1}\!},\emb{\sep},\emb{\atr^{\ell_\star}_{j_1}\!},\!...,\emb{\obj_{j_{\Nft+1}}\!},\emb{\sep}], 
\qquad \xb_\Tau=\emb{\sep},
\qquad y=\atr^{\ell_\star}_{j_{\Nft+1}},
\end{align}
on the population, where the attribute index and each subject in the context are sampled i.i.d. with replacement from a randomly selected set of $M_{\text{ft}}$ subjects. 
    Then, keeping the weight $\W_{OV}^\text{subj}$ fixed, and updating the other weights with 
    one step of GD on $\W_{KQ}^\text{subj}$ with step size $\eta_{KQ}^{\text{subj}}$, one step of projected GD (right projection with $\Ib-\emb{\sep}\emb{\sep}^\top$) on $\W_{OV}^\text{rel}$ with step size $\eta_{OV}^{\text{rel}}$, and two steps of GD on $\W_{KQ}^\text{rel}$ with step size $\eta_{KQ}^{\text{rel}}$, the updated weights satisfy
    \begin{align*}
\W_{KQ}^{\text{subj}}&\approx\beta\left(\sum_j\emb{s_j}+\alpha_1\emb{p_{-1}}\right)\emb{\sep}^\top,\\
\W_{KQ}^{\text{rel}}&\approx\beta\left(\sum_\ell\emb{r_\ell}+\alpha_3\left(\frac{2}{3}\frac{1}{LU_L}
\sum_{u}
\emb{u}+\frac{1}{m-1}\sum_{i=1}^{m-1}
\pb_{3i+2}-\frac{1}{\Tau}\pb-\frac{1}{3}\frac{1}{M}\sum_s\emb{s}\right)\right)\emb{\sep}^\top,\\
\W_{OV}^{\text{rel}}&\approx\sum_{\ell}\Big(\sum_{j}\emb{\uatr_{j}^\ell}\Big)\emb{\rel_{\ell}}^\top+\alpha_2\left((m-1)\sum_{\ell=1}^L
\Big(\sum_{j=1}^{U_L} \emb{\uatr_j^\ell}\Big)\Big(\sum_{j=1}^{U_L} \emb{\uatr_j^\ell}\Big)^\top+\frac{U_L^2}{ M_{\text{ft}}}\sum_{j=1}^{M_{\text{ft}}}
\Big( \sum_{\ell=1}^L \emb{\atr_j^\ell}\Big) \emb{\obj_j}^\top\right),
    \end{align*}
    where $\pb_i=\emb{p_{i-\Tau+1}}$, $\alpha_1:=\eta_{KQ}^{\text{subj}}m^{-1}\left(1-U_L^{-1}\right)$, $\alpha_2=\eta_{OV}^{\text{rel}}(L\Tau U_L^2)^{-1}$, and $\alpha_3:=\eta_{KQ}^{\text{rel}}(m-1)^2\alpha_2$. Further,
    when $\beta\to\infty$, this model gets perfect accuracy on \ICL sequences with seen or held-out query subjects.
\end{theorem}
\begin{proof}
Note that for a given \ICL sequence $\X$, the query subject $\obj_{j_{\Nft+1}}$ appears at the fixed relative position. The last-token negative log-likelihood is
\begin{equation*}
\mathcal L(\X,y)
=
-\emb{y}^\top f(\X)
+
\log\sum_{v\in\mathcal V}\exp\big(\emb{v}^\top f(\X)\big),
\end{equation*}
with gradient
\begin{equation*}
\ub:=\nabla_f \mathcal L
=
\sum_{v\in\mathcal V}\pi_v\,\emb v
-
\emb y,
\qquad
\pi_v=\frac{e^{\emb v^\top f}}{\sum_u e^{\emb u^\top f}}.
\end{equation*}

Recall that $f(\X)=\sum_{k=1}^H\W_{OV}^kg_k(\X)$, where $g_k(\X)=\sum_{i=1}^\Tau \alpha_i^k\xb_i$,  $\alpha_i^k:=\sft{\lambda_i^k}$, $\lambda_i^k:= (\xb_i+\pb_i)^\top \W_{KQ}^k\emb{\sep}$. 

Next, we compute the gradients with respect to $\W_{OV}^k$ and $\W_{KQ}^k$.

We have that
\begin{align*}
    \nabla_{\W_{OV}^{k}}\mathcal L
=
\ub g_{k}(\X)^\top=\ub \sum_{i=1}^\Tau \alpha_i^k\xb_i^\top.
\end{align*}
Further, for each $i,k$,
\begin{equation*}
\frac{\partial \mathcal L}{\partial \alpha^k_i}
=
\inpb{\nabla_{g_k}\mathcal L}{ \frac{\partial g_k}{\partial \alpha^k_i}}
=
\ub^\top\W_{OV}^k \xb_i
=: r^k_i,
\end{equation*}
where we use $\nabla_{g_k}\mathcal L=(\W_{OV}^k)^\top\nabla_f\mathcal L=(\W_{OV}^k)^\top\ub$. Next, we have that
\[
\frac{\partial \alpha^k_i}{\partial \lambda^k_j}=\alpha^k_i(\ind{i=j}-\alpha^k_j).
\]
Therefore,
\begin{align}
\frac{\partial \mathcal L}{\partial \lambda_j^k}
&=
\sum_{i=1}^{\Tau}\frac{\partial \mathcal L}{\partial \alpha^k_i}\frac{\partial \alpha^k_i}{\partial \lambda^k_j}
=
\sum_{i=1}^{\Tau} r^k_i\alpha^k_i(\ind{i=j}-\alpha^k_j)=
r^k_j\alpha^k_j-\alpha^k_j\sum_{i=1}^{\Tau}\alpha^k_i r^k_i
=
\alpha_j^k\left(r^k_j-\bar r^k\right),\nonumber
\end{align}
where we define $\bar r^k:=\sum_{i=1}^{\Tau}\alpha^k_i r^k_i$. Finally, we have that
\begin{align}\label{eq:grad_KQ_general}
\nabla_{\W_{KQ}^{k}}\mathcal L
&=
\sum_{j=1}^{\Tau}\frac{\partial \mathcal L}{\partial \lambda^k_j}\,
\nabla_{\W_{KQ}^{k}} \lambda^k_j
=\beta
\sum_{j=1}^{\Tau}
\alpha_j^k(r^k_j-\bar r^k)\,(\xb_j+\pb_j)\emb{\sep}^\top.
\end{align}

We next analyze the gradient updates for two iterations of finetuning. We keep $\W_{OV}^{\text{subj}}$ fixed throughout. 

\paragraph{First Iteration.} For the first iteration, we analyze the updates for $\W_{KQ}^\text{subj}$, $\W_{KQ}^\text{rel}$, and $\W_{OV}^\text{rel}$. 

\paragraph{Subject Head ($\W_{KQ}^{\text{subj}}$).} We first focus on the \emph{subject head}. 
Recall from the construction for \Bio sequences that $\W_{KQ}^\text{subj}=
\beta
\Big(\sum_{j\in\objb}\emb{\obj_j}\Big)\emb{\sep}^\top$.
Let $\mathcal I_{\obj}:=\{i:\exists s\in \objb, \xb_i=\emb{\obj}\}$ denote the set of subject positions, with
$|\mathcal I_{\obj}|=\Nft+1=:m$. For $i\in\mathcal{I}_s$, let $J(i)$ denote the subject index $j_i$ in the context, such that $j_i\in[M]$, $s_{j_i}\in\objb$. For $\beta\to\infty$, the resulting attention weights are
\begin{equation*}
\alpha_i
=
\begin{cases}
\frac{1}{m}, & i\in\mathcal I_{\obj},\\
0, & i\notin\mathcal I_{\obj},
\end{cases}
\end{equation*}
where we omit the superscript $k$ for brevity. Substituting into \cref{eq:grad_KQ_general}, we obtain
\begin{equation*}
\nabla_{\W_{KQ}^{\text{subj}}}\mathcal L
=
\frac{\beta}{m}
\sum_{i\in\mathcal I_{\obj}}
(r_i-\bar r)\,(\xb_i+\pb_i)\,\emb{\sep}^\top,
\qquad
\bar r=\frac{1}{m}\sum_{i\in\mathcal I_{\obj}} r_i .
\end{equation*}
Let $\hat r_i:=r_i-\sum_{v\in\mathcal{V}}\pi_v\emb{v}^\top\W_{OV}\xb_i$. 
Let $\atrb_j$ denote the set of all attributes associated with subject $\obj_j$. We have that \begin{align*}
\sum_{v\in\mathcal{V}}\pi_v\emb{v}^\top\W_{OV}^{\text{subj}}\xb_i=\ind{i\in\mathcal{I}_s}\sum_{u\in\atrb_{J(i)}}\pi_u.
\end{align*}
From \cref{lem:logit-bound}, 
$\pi_u\leq \frac{1}{V+LU_L(C\exp^{-1}-1)}$.
Then, we have
\begin{align*}
    |\hat r_i-r_i|\leq \frac{L}{V+LU_L(C\exp^{-1}-1)}\ll \frac{1}{m},
\end{align*}
since $V\geq M+U_LL\gg mL$. We also define \begin{align*}
\nabla_{\W_{KQ}^{\text{subj}}}\hat{\mathcal L}
:=
\frac{\beta}{m}
\sum_{i\in\mathcal I_{\obj}}
(\hat r_i-\bar{\hat r})(\xb_i+\pb_i)\,\emb{\sep}^\top,
\qquad
\bar {\hat r}=\frac{1}{m}\sum_{i\in\mathcal I_{\obj}} \hat r_i.
\end{align*}
Then, we have that
\begin{align*}
\frac{1}{\beta}\norm{\nabla_{\W_{KQ}^{\text{subj}}}\hat{\mathcal L}-\nabla_{\W_{KQ}^{\text{subj}}}{\mathcal L}}&\leq \frac{1}{m}\left(\norm{\sum_i(r_i-\hat r_i)(\xb_i+\pb_i)\emb{\sep}^\top}+\norm{(\bar r-\bar{\hat r})\sum_i(\xb_i+\pb_i)\emb{\sep}^\top}\right)\\
&\leq \frac{\sqrt{2}}{m}\left(m\max_i|r_i-\hat r_i|+m\max_i|r_i-\hat r_i|\right)\ll \frac{2\sqrt{2}}{m}.
\end{align*}
Therefore, in the next part, we analyze $\nabla_{\W_{KQ}^{\text{subj}}}\hat{\mathcal L}$ (as we will see, $\beta^{-1}\norm{\nabla_{\W_{KQ}^{\text{subj}}}\hat{\mathcal L}}$ is $\Theta(1/m)$, \textit{i.e.} it dominates the norm of the difference).
We have that
\begin{align*}
    \hat r_i=-\begin{cases}
    \ind{\atr_{J(i)}^{\ell_\star}=y}, &i\in\mathcal I_{\obj},\\
        0, & \text{o.w.}
    \end{cases}
\end{align*}
Let $q=3\Nft$ denote the index of the query subject. Then, we have that
\begin{align*}
    \hat r_i-\bar{\hat r}=-\begin{cases}
    1-\frac{1}{m}\sum_{k\in\mathcal{I}_s}\ind{\atr_{J(k)}^{\ell_\star}=y}, &i=q,\\
    \ind{\atr_{J(i)}^{\ell_\star}=y}-\frac{1}{m}\sum_{k\in\mathcal{I}_s}\ind{\atr_{J(k)}^{\ell_\star}=y}, &i\in\mathcal{I}_s \setminus \{q\},\\
        0, & \text{o.w.}
    \end{cases}
\end{align*}

Let $C_y:=\sum_{i=1}^{m} \ind{a_{j_i}^{\ell_\star}=y}$. Then,
\begin{align}\label{eq:wkq_subj}
\nabla_{\W_{KQ}^{\text{subj}}}\hat{\mathcal L}
=-\frac{\beta}{m}\left(\sum_{k=1}^m(\ind{\atr_{j_k}^{\ell_\star}=y}-m^{-1}C_y)\emb{\obj_{j_k}}+\sum_{k=1}^{m-1}(\ind{\atr_{j_k}^{\ell_\star}=y}-m^{-1}C_y)\pb_{3k}+(1-m^{-1}C_y)\emb{p_{-1}}\right)\emb{\sep}^\top.
\end{align}

We first calculate some expectations. Under i.i.d.\ context sampling with replacement, conditioned on $y$, we have $
\Pr(a_i^{\ell_\star}=y\mid y)=\frac{1}{U_L}$, and hence $\E[C_y\mid y]=\frac{m}{U_L}$. Further, averaging over $y$ gives $\E[C_y]=\frac{m}{U_L}$.

Let $\mu_s(y):=\E[\emb{s}\mid a^{\ell_\star}=y]$. Under i.i.d.\ context sampling with replacement,
\[
\E\!\left[\sum_i \ind{a_i^{\ell_\star}=y}\emb{s_i} \middle| y\right]
=
\sum_i \E[\ind{a_i^{\ell_\star}=y}s_i\mid y]
=
m\frac{1}{U_L}\mu_s(y),
\]
since $\Pr(a_i^{\ell_\star}=y\mid y)=1/U_L$.
Hence,
\[
\E\!\left[\sum_i \ind{a_i^{\ell_\star}=y}\emb{s_i} \right]
=
\frac{m}{U_L}\E[\mu_s(y)].
\]
Since each value appears on $M/U_L$ subjects, $\E[\mu_s(y)]
=
\frac{1}{M}\sum_{j=1}^M \emb{s_j}$,
and therefore
\[
\E\!\left[\sum_i \ind{a_i^{\ell_\star}=y}\emb{s_i} \right]
=
\frac{m}{U_L M}\sum_{j=1}^M \emb{s_j}.
\]

Taking expectation in \cref{eq:wkq_subj} over the context, and substituting these to simplify, we have that 
\begin{align*}
\E\big[\nabla_{\W_{KQ}^{\text{subj}}}\hat{\mathcal L}
\big]&=-\frac{\beta}{m}\left(\left(\frac{m}{U_LM}-\frac{1}{U_LM}m\right)\sum_j\emb{\obj_j}+\sum_{k=1}^{m-1}\left(\frac{1}{U_L}-\frac{1}{U_L}\right)\pb_{3k}+(1-U_L^{-1})\emb{p_{-1}}\right)\emb{\sep}^\top\\
&=-\frac{\beta}{m}\left(1-\frac{1}{U_L}\right)
\emb{p_{-1}}\emb{\sep}^\top.
\end{align*}
Since $\alpha_1=\eta_{KQ}^{\text{subj}}m^{-1}\left(1-U_L^{-1}\right)$, the updated weight matrix is
\begin{align*}
\W_{KQ}^{\text{subj}}\approx\beta\left(\sum_j\emb{s_j}+\alpha_1\emb{p_{-1}}\right)\emb{\sep}^\top.
\end{align*}

\paragraph{Relation Head.} Next, we consider the \textit{relation} head. Recall from the construction for \Bio sequences that $\W_{KQ}^{\text{rel}}=\beta\left(\sum_{\ell}\emb{\rel_\ell}+\pe\right)\emb{\sep}^\top$. Since $S+2\geq \Tau$, first consider the update for the KQ part.

We have that 
\begin{align*}
r_i=\frac{1}{\Tau}\sum_i\ub^\top\Big(\sum_{\ell}\emb{\uatr_\ell}\Big)\ind{\xb_i=\emb{\rel_\ell}}=0,
\end{align*}
since the context does not contain a relation token. This means that the gradient is zero and therefore, $\W_{KQ}^{\text{rel}}$ is not updated in this iteration. 

\paragraph{$\W_{OV}^{\text{rel}}$.} Next, consider the update for the OV part. We have that
\begin{align*}
\nabla_{\W_{OV}^{\mathrm{rel}}}\mathcal L
=
\frac{1}{\Tau}\sum_{i=1}^{\Tau}\Big(\sum_{v\in\mathcal V}\pi_v\emb v-\emb y\Big)\xb_i^\top.
\end{align*}
Define
\begin{align*}
\nabla_{\W_{OV}^{\mathrm{rel}}}\hat{\mathcal L}=-\frac{1}{\Tau}\emb{y}\left(\sum_i\emb{a_i^{\ell_\star}}+\sum_i\emb{s_i}+m\emb{\sep}\right)^\top.
\end{align*}
Then, using \cref{lem:logit-bound}, we have that
\begin{align*}
\norm{\nabla_{\W_{OV}^{\mathrm{rel}}}\hat{\mathcal L}-\nabla_{\W_{OV}^{\mathrm{rel}}}\mathcal L}\leq \sum_v\pi_v\leq \frac{V}{V+(Ce^{-1}-1)LU_L}\ll 1,
\end{align*}
since $C\gg e$. Therefore, we analyze $\nabla_{\W_{OV}^{\mathrm{rel}}}\hat{\mathcal L}$ hereon (as we will see, $\norm{\nabla_{\W_{OV}^{\mathrm{rel}}}\hat{\mathcal L}}=\Theta(\sqrt{U_LL})>1$, 
\textit{i.e.} it dominates the norm of the difference). We have that
\begin{align*}
\E[\nabla_{\W_{OV}^{\mathrm{rel}}}\hat{\mathcal L}]
&= -\frac{m-1}{L\Tau U_L^2}\sum_{\ell=1}^L
\Big(\sum_{j=1}^{U_L} \emb{\uatr_j^\ell}\Big)\Big(\sum_{j=1}^{U_L} \emb{\uatr_j^\ell}\Big)^\top-
\frac{m-1}{L\Tau U_LM_{\text{ft}}}\sum_{\ell=1}^L
\Big(\sum_{j=1}^{U_L} \emb{\uatr_j^\ell}\Big)\Big(\sum_{j=1}^{M_{\text{ft}}} \emb{\obj_j}\Big)^\top\\&-\frac{1}{L\Tau M_{\text{ft}}}\sum_{j=1}^{M_{\text{ft}}}
\Big( \sum_{\ell=1}^L \emb{\atr_j^\ell}\Big) \emb{\obj_j}^\top-
\frac{m}{L\Tau U_L}\sum_{\ell=1}^L\sum_{j=1}^{U_L} \emb{\uatr_j^\ell}\emb{\sep}^\top.
\end{align*}
We right project this gradient onto the subspace orthogonal to $\emb{\sep}$, and since $M_\text{ft}\gg U_L$, $\lambda:=(L\Tau U_L^2)^{-1}$, we have
\begin{align*}
    -\E[\nabla_{\W_{OV}^{\mathrm{rel}}}\hat{\mathcal L}](\Ib-\emb{\sep}\emb{\sep}^\top) \approx \lambda\left((m-1)\sum_{\ell=1}^L
\Big(\sum_{j=1}^{U_L} \emb{\uatr_j^\ell}\Big)\Big(\sum_{j=1}^{U_L} \emb{\uatr_j^\ell}\Big)^\top+\frac{U_L^2}{ M_{\text{ft}}}\sum_{j=1}^{M_{\text{ft}}}
\Big( \sum_{\ell=1}^L \emb{\atr_j^\ell}\Big) \emb{\obj_j}^\top\right).
\end{align*}

Since $\alpha_2=\eta_{OV}^{\text{rel}}(L\Tau U_L^2)^{-1}$, the updated weight matrix is
\begin{align*}
    \W_{OV}^{\text{rel}}\approx\sum_{\ell}\Big(\sum_{j}\emb{\uatr_{j}^\ell}\Big)\emb{\rel_{\ell}}^\top+\alpha_2\left((m-1)\sum_{\ell=1}^L
\Big(\sum_{j=1}^{U_L} \emb{\uatr_j^\ell}\Big)\Big(\sum_{j=1}^{U_L} \emb{\uatr_j^\ell}\Big)^\top+\frac{U_L^2}{ M_{\text{ft}}}\sum_{j=1}^{M_{\text{ft}}}
\Big( \sum_{\ell=1}^L \emb{\atr_j^\ell}\Big) \emb{\obj_j}^\top\right).
\end{align*}

\paragraph{Second Iteration.} In this case, we only update $\W_{KQ}^{\text{rel}}$, keeping the other weights fixed. 
\paragraph{$\W_{KQ}^{\text{rel}}$.} We use the updated $\W_{OV}^{\text{rel}}$ to compute the gradient wrt $\W_{KQ}^{\text{rel}}$. Since $\W_{KQ}^{\text{rel}}$ remains unchanged in the first iteration, $\alpha_i^\text{rel}=\tfrac{1}{\Tau}$, and we have
\begin{equation*}
\nabla_{\W_{KQ}^{\mathrm{rel}}}\mathcal L
=
\frac{\beta}{\Tau}\sum_{i=1}^{\Tau}
(r_i-\bar r)\,(\xb_i+\pb_i)\,\emb{\sep}^\top,
\qquad
r_i := \left(\sum_{v\in\mathcal V}\pi_v\,\emb v
-
\emb y\right)^\top\W_{OV}^{\mathrm{rel}} \xb_i,\quad
\quad
\bar r:=\frac{1}{\Tau}\sum_i  r_i.
\end{equation*}
Also define
\begin{equation*}
\nabla_{\W_{KQ}^{\mathrm{rel}}}\hat{\mathcal L}
=
\frac{\beta}{\Tau}\sum_{i=1}^{\Tau}
(\hat r_i-\bar {\hat r})(\xb_i+\pb_i)\emb{\sep}^\top,
\qquad
\hat r_i :=  
-
\emb{y}^\top\W_{OV}^{\mathrm{rel}} \xb_i,\quad
\quad
\bar{\hat{r}}:=\frac{1}{\Tau}\sum_i  \hat r_i.
\end{equation*}

Following similar steps as in the update for the subject head, using \cref{lem:logit-bound-new}, we have that
\begin{align*}
    (\beta\alpha_2)^{-1}\norm{\nabla_{\W_{KQ}^{\mathrm{rel}}}\hat{\mathcal L}-\nabla_{\W_{KQ}^{\mathrm{rel}}}{\mathcal L}}\leq 2\sqrt{2}\alpha_2^{-1}\max_i|r_i-\hat r_i|\leq (\max_{u\in\uatrb} \pi_u)\left((m-1)^2U_L+mL\frac{U_L^2}{M_{\text{ft}}}\right)\leq \frac{em^2U_L}{M}\ll m^2,
\end{align*}
since $M>M_{\text{ft}}\gg U_L$. Therefore, we analyze $\nabla_{\W_{KQ}^{\mathrm{rel}}}\hat{\mathcal L}$ hereafter (as we will see, $(\beta\alpha_2)^{-1}\norm{\nabla_{\W_{KQ}^{\mathrm{rel}}}\hat{\mathcal L}}=\Theta(m^2)$, 
\textit{i.e.} it dominates the norm of the difference). We have that
\begin{align*}
    \hat r_i=-\alpha_2\begin{cases}
        m-1,&\xb_i\in\uatrb_{\ell_\star}\\
        \frac{U_L^2}{M_{\text{ft}}}\ind{\atr_i^{\ell_\star}=y},&\xb_i\in\objb
    \end{cases},\qquad 
    \bar{\hat r}=-\frac{\alpha_2}{\Tau}\left((m-1)^2
    +\frac{U_L^2}{M_{\text{ft}}}\sum_i\ind{\atr_i^{\ell_\star}=y}\right).
\end{align*}
Therefore,
\begin{align*}
    \hat r_i-\bar{\hat r}=-\alpha_2\begin{cases}
        (m-1)\left(1-\frac{m-1}{\Tau}\right)-\frac{U_L^2}{\Tau M_{\text{ft}}}\sum_i\ind{\atr_i^{\ell_\star}=y},&\xb_i\in\uatrb_{\ell_\star},\\
        \frac{U_L^2}{M_{\text{ft}}}\ind{\atr_i^{\ell_\star}=y}-\frac{U_L^2}{\Tau M_{\text{ft}}}\sum_i\ind{\atr_i^{\ell_\star}=y}-\frac{(m-1)^2}{\Tau},&\xb_i\in\objb.
    \end{cases}
\end{align*}

Let $\lambda_m:=(m-1)\left(1-\tfrac{m-1}{\Tau}\right)$. Then, $\nabla_{\W_{KQ}^{\mathrm{rel}}}\hat{\mathcal L}$ can be written as
\begin{align*}
-\frac{\beta\alpha_2}{\Tau}\left(\sum_{i=1}^{m-1}
\left(\lambda_m-\tfrac{U_L^2}{\Tau M_{\text{ft}}}C_y\right)(\emb{\atr_{j_i}^{\ell_\star}}+\pb_{3i+2})+\sum_{i=1}^{m}\left(\tfrac{U_L^2}{M_{\text{ft}}}\ind{\atr_i^{\ell_\star}=y}-\tfrac{U_L^2}{\Tau M_{\text{ft}}}C_y-\tfrac{(m-1)^2}{\Tau}\right)(\emb{\obj_{j_i}}+\pb_{3i})\right)\emb{\sep}^\top.
\end{align*}

We have that
\[
\E[C_y\emb{y}]=\E\!\left[\E[C_y\emb{y}\mid y]\right]=
\E\!\left[\emb{y}\E[C_y\mid y]\right]=\frac{m}{U_L}\E[\emb{y}]=\frac{m}{LU_L^2}\sum_u\emb{u}.
\] 
The expectations of the terms with attribute, subject and positional tokens can be simplified as follows.

First, for the attribute term, we have
\begin{align*}
    \E \left[\sum_{i=1}^{m-1}
\left(\lambda_m-\tfrac{U_L^2}{\Tau M_{\text{ft}}}C_y\right)\emb{\atr_{j_i}^{\ell_\star}}\right]&=\lambda_m(m-1)\tfrac{1}{LU_L}\sum_u\emb{u}-\tfrac{U_L^2}{\Tau M_{\text{ft}}}(m-1)\E[C_y\emb{y}]\\
&=\left((m-1)\left(1-\tfrac{m-1}{\Tau}\right)-\tfrac{mU_L}{\Tau M_{\text{ft}}}\right)\frac{m-1}{LU_L}\sum_u\emb{u},\\
&\approx \frac{2(m-1)^2}{3}\frac{1}{LU_L}\sum_u\emb{u},
\end{align*}
where the last step follows as $\Tau=3m-1$ and $M_{\text{ft}}\gg U_L$.

Next, for the subject term, we have
\begin{align*}
    \E\left[\sum_{i=1}^{m}\left(\tfrac{U_L^2}{M_{\text{ft}}}\ind{\atr_i^{\ell_\star}=y}-\tfrac{U_L^2}{\Tau M_{\text{ft}}}C_y-\tfrac{(m-1)^2}{\Tau}\right)\emb{\obj_{j_i}}\right]&=\tfrac{U_L^2}{M_{\text{ft}}}\E\sum_{i=1}^{m}\ind{\atr_i^{\ell_\star}=y}\emb{\obj_{j_i}}-\E\sum_{i=1}^{m}\left(\tfrac{U_L^2}{\Tau M_{\text{ft}}}C_y+\tfrac{(m-1)^2}{\Tau}\right)\emb{\obj_{j_i}}\\
    &=\left(\tfrac{U_L}{M_{\text{ft}}}-\tfrac{mU_L}{\Tau M_{\text{ft}}}-\tfrac{(m-1)^2}{\Tau}\right)\tfrac{m}{M}\sum_s\emb{s}\\
    &\approx -\tfrac{(m-1)^2}{3}\tfrac{1}{M}\sum_s\emb{s}.
\end{align*}
Next, for the positional encoding term, we have 
\begin{align*}
    &\E\left[\sum_{i=1}^{m-1}
(m-1)\pb_{3i+2}+\tfrac{U_L^2}{M_{\text{ft}}}\sum_{i=1}^{m}\ind{\atr_i^{\ell_\star}=y}\pb_{3i}
-\sum_{i=1}^\Tau\left(\tfrac{U_L^2}{\Tau M_{\text{ft}}}C_y+\tfrac{(m-1)^2}{\Tau}\right)\pb_{i}\right]\\&=(m-1)\sum_{i=1}^{m-1}
\pb_{3i+2}+\tfrac{mU_L}{M_{\text{ft}}}\sum_{i=1}^{m}
\pb_{3i}-\tfrac{1}{\Tau}\left(\tfrac{mU_L}{M_{\text{ft}}}+(m-1)^2\right)\sum_i\pb_i\\
&=\left(m-1-\tfrac{mU_L}{M_{\text{ft}}}\right)\sum_{i=1}^{m-1}
\pb_{3i+2}+\left(\tfrac{mU_L}{M_{\text{ft}}}\left(1-\Tau^{-1}\right)-\tfrac{(m-1)^2}{\Tau}\right)\pb\\
&\approx (m-1)\sum_{i=1}^{m-1}
\pb_{3i+2}-\frac{(m-1)^2}{\Tau}\pb.
\end{align*}
Combining these terms, we have that
\begin{equation*}
\E[\nabla_{\W_{KQ}^{\mathrm{rel}}}\hat{\mathcal L}]
\approx
-(m-1)^2\alpha_2\beta\left(\frac{2}{3}\frac{1}{LU_L}
\sum_{u}
\emb{u}+\frac{1}{m-1}\sum_{i=1}^{m-1}
\pb_{3i+2}-\frac{1}{\Tau}\pb-\frac{1}{3}\frac{1}{M}\sum_s\emb{s}\right)\emb{\sep}^\top.
\end{equation*}
Since $\alpha_3=\eta_{KQ}^{\text{rel}}(m-1)^2\alpha_2$, the updated weight matrix is 
\begin{align*}
\W_{KQ}^{\text{rel}}\approx\beta\left(\sum_\ell\emb{r_\ell}+\alpha_3\left(\frac{2}{3}\frac{1}{LU_L}
\sum_{u}
\emb{u}+\frac{1}{m-1}\sum_{i=1}^{m-1}
\pb_{3i+2}-\frac{1}{\Tau}\pb-\frac{1}{3}\frac{1}{M}\sum_s\emb{s}\right)\right)\emb{\sep}^\top.
\end{align*}
Using these updated weights, following similar steps as in the proof of \cref{prop:icl_}, we can show that when $\beta\to\infty$, the model attains perfect accuracy on \ICL sequences with seen or held-out subjects.
\end{proof}

\subsubsection{Helper Lemmas}
\begin{lemma}\label{lem:logit-bound}
    Under the same conditions as \cref{th:ft_dynamics}, assuming $M\geq CU_L$, where $C>e$ is a constant, and using the weights
    \begin{align*}
\W_{KQ}^{\text{subj}}=\beta\Big(\sum_j\emb{\obj_j}\Big)\emb{\sep}^\top,&
\qquad
\W_{OV}^{\text{subj}}=-\sum_j\Big(\sum_{j'\neq j,\ell}\emb{\atr_{j'}^\ell}\Big)\emb{\obj_j}^\top,\\
\W_{KQ}^{\text{rel}}=\beta\Big(\sum_{\ell}\emb{\rel_\ell}+\pe\Big)\emb{\sep}^\top,&\qquad\W_{OV}^{\text{rel}}=\sum_{\ell}\Big(\sum_{j}\emb{\uatr_{j}^\ell}\Big)\emb{\rel_{\ell}}^\top,
\end{align*}
    the prediction probabilities $\pi_v$ are bounded as:
    \begin{align*} 
        \pi_v\leq \frac{1}{V+(Ce^{-1}-1)LU_L}.
    \end{align*}
\end{lemma}
\begin{proof}

Let $m:=\Nft+1$ denote the number of subject positions in the \ICL sequence, and let $k_j$ be the multiplicity of subject $j$ among these. 

When $\beta\to\infty$, the subject head attends uniformly over the $m$ subject positions, hence
\[
g_{\text{subj}}(\X)=\frac1m\sum_{i=1}^{m}\emb{\obj_{j_i}}.
\]
Therefore
\[
f_{\text{subj}}(\X)
=
\W_{OV}^{\text{subj}}g_{\text{subj}}(\X)
=
-\frac1m\sum_{i=1}^{m}\sum_{j'\neq j_i,\ell}\emb{\atr_{j'}^\ell}.
\]
Moreover, since $\emb{v}^\top \W_{OV}^{\text{rel}}=0$ for all attributes other than the relation tokens, and since relation tokens don't appear in the \ICL sequence, the relation head does not contribute to the logits. Therefore,
\begin{align}
\emb{\atr_j^\ell}^\top f(\X)=\emb{\atr_j^\ell}^\top f_{\text{subj}}(\X)=
-\frac1m\sum_{i=1}^{m}\ind{j\neq j_i}
=
-\Big(1-\frac{k_j}{m}\Big)
=: -c_j.
\label{eq:atomic_attr_logit_proof}
\end{align}

Also, note that for any $v\notin\uatrb$, the logit is $0$. Next, let $Z:=\sum_{u\in\mathcal V}\exp(\emb u^\top f(\X))$. Using \cref{eq:atomic_attr_logit_proof}, we have
\begin{align*}
Z
&=
\sum_{j}\sum_{\ell=1}^{L}\exp(\emb{\atr_j^\ell}^\top f(\X))
+(V-LU_L)=
\sum_j\sum_{\ell=1}^{L}e^{-c_j}
+(V-LU_L)=
L\sum_j e^{-c_j}
+(V-LU_L)\\&=
Le^{-1}\sum_j e^{k_j/m}+(V-LU_L).
\end{align*}
Hence, for any attribute token $\atr_j^\ell$,
\begin{align*}
\pi_{\atr_j^\ell}
=
\frac{\exp(\emb{\atr_j^\ell}^\top f(\X))}{Z}
&\le
\frac{e^{-c_j}}{Le^{-1}\sum_{j'} e^{k_{j'}/m}+(V-LU_L)}
=
\frac{e^{k_j/m}}{L\sum_{j'} e^{k_{j'}/m}+e(V-LU_L)}.
\end{align*}
Since $\sum_{j'} e^{k_{j'}/m}\ge M$, and $k_j\le m$, we have
\[
\pi_{\atr_j^\ell}\le \frac{1}{V-LU_L+e^{-1}LM}.
\]
Similarly, for $v\notin\uatrb$, $\pi_v=Z^{-1}\leq (V-LU_L+e^{-1}LM)^{-1}$.
Using the assumption on $M$ finishes the proof.
\end{proof}

\begin{lemma}
\label{lem:logit-bound-new}
Under the same conditions as \cref{th:ft_dynamics}, assuming $\alpha_2\leq (\lambda_1+m\lambda_2)^{-1}$, when using the weights,
\begin{align*}
\W_{KQ}^{\text{subj}}&=\beta\Big(\sum_j\emb{\obj_j}+\alpha_1\emb{p_{-1}}\Big)\emb{\sep}^\top,
\qquad
\W_{OV}^{\text{subj}}=-\sum_j\Big(\sum_{j'\neq j,\ell}\emb{\atr_{j'}^\ell}\Big)\emb{\obj_j}^\top,\\\W_{KQ}^{\text{rel}}&=\beta\Big(\sum_{\ell=1}^L\emb{\rel_\ell}+\pe\Big)\emb{\sep}^\top,\\ 
\W_{OV}^{\text{rel}}
&=
\sum_{\ell}\Big(\sum_{j}\emb{\uatr_{j}^\ell}\Big)\emb{\rel_{\ell}}^\top+\alpha_2\left((m-1)\sum_{\ell=1}^L
\Big(\sum_{a=1}^{U_L}\emb{\uatr_a^\ell}\Big)\Big(\sum_{a=1}^{U_L}\emb{\uatr_a^\ell}\Big)^\top
+\frac{U_L^2}{M_{\text{ft}}}\sum_{j=1}^{M_{\text{ft}}}
\Big(\sum_{\ell=1}^L\emb{\atr_j^\ell}\Big)\emb{\obj_j}^\top\right),\\
\end{align*}
the prediction probabilities $\pi_v$ are bounded as:
\begin{align*}
    \pi_v\leq\frac{e}{V+LU_L(e^{-1}-1)}\leq \frac{e}{M}.
\end{align*}

\end{lemma}

\begin{proof}
Let $m:=\Nft+1$ denote the number of subject positions in the \ICL sequence, and let $k_j$ be the multiplicity of subject $j$ among these positions (so $\sum_j k_j=m$), and define $\objb':=\{j:\,k_j>0\}$.

When $\beta\to\infty$, the relation head attends uniformly over all $\Tau$ positions, while the subject head attends to the query subject $\emb{\obj_{j_{\Nft+1}}}$. For any attribute token \(v=\atr_j^\ell\), define
\[
c_j(\X):=|\{\,i\in[\Nft+1]:\, j_i=j\,\}|,
\]
\textit{i.e.}, the number of times subject \(j\) appears among the subjects in the given sequence \(\X\). Then
\[
\emb{\atr_j^\ell}^{\top} f(\X)=\alpha_2\left((m-1)\frac{\Nft}{\Tau}\ind{\ell=\ell_\star}
+\frac{U_L^2}{M_{\mathrm{ft}}\Tau}c_j(\X)\right)
-\ind{j\neq j_q}.
\]

Let $Z:=\sum_{u\in\mathcal V}\exp(\emb u^\top f(\X))$ and let
\[\pi_v=\frac{\exp(\emb v^\top f(\X))}{Z}.\]

Let $\lambda_1:=(m-1)\Nft/\Tau$, $\lambda_2:=U_L^2/(M_{\text{ft}}\Tau)$. Then, we have 
\begin{align*}
    \pi_v\leq \frac{\exp(\alpha_2(\lambda_1+m\lambda_2))}{U_L(\exp(\alpha_2(\lambda_1+\lambda_2)-1)+(L-1)\exp(\alpha_2\lambda_2-1))+V-LU_L}\leq \frac{e}{V+LU_L(e^{-1}-1)},
\end{align*}
since $\alpha_2\leq (\lambda_1+m\lambda_2)^{-1}$.

\end{proof}

\subsection{Experimental Validation} 
\subsubsection{Relative Positional Encoding}
\begin{wrapfigure}[9]{r}{0.52\linewidth}
    \centering
    \vspace{-12mm}
    \includegraphics[width=0.75\linewidth]{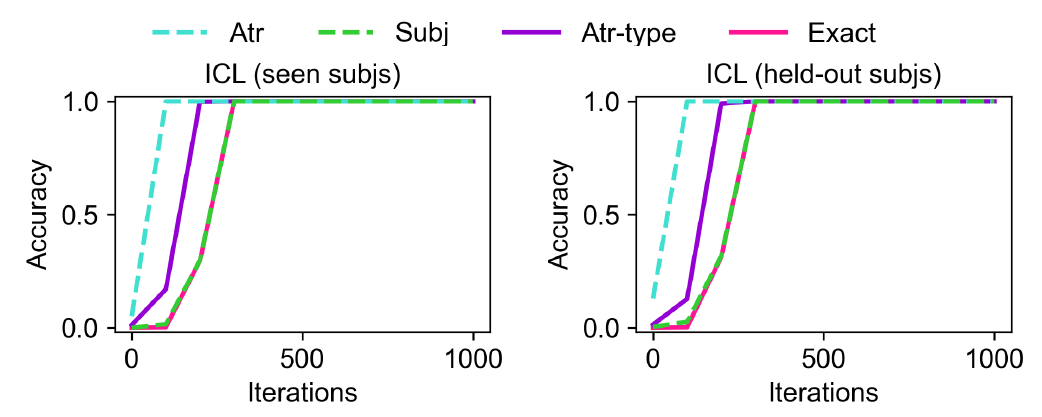}
    \caption{Finetuning the single-layer attention-only model with relative positional encoding pretrained on \Bio sequences using \ICL sequences with a subset of subjects enables generalization on \ICL sequences with held-out query subjects.}
    \label{fig:attnonly_icl_acc}
\end{wrapfigure}
In this section we present some additional experimental results for experimental validation in \cref{sec:attn_only}. Specifically, in \cref{fig:wov_pt,fig:wov_icl}, we visualize a sliced version of attention weights $\W_{OV}^\text{subj}$ (left) and $\W_{OV}^\text{rel}$ (right) after pretraining and finetuning, respectively. Specifically, rowwise or columnwise, the first $8$ entries correspond to subjects, the next $8$ to attribute type $\ell=1$ tokens, the next $8$ to $\ell=2$, the next $5$ to grammar and separator tokens, and the final $8$ to the relation tokens. From the two figures, we observe that the subject heads encode information about subject-attribute associations by suppressing attributes not associated with the subject, and don't change much after finetuning. On the other hand, the relation head initially encodes information about relation type-attribute associations, but after finetuning, it changes significantly, encoding attribute-attribute associations for the same attribute type. This is consistent with our constructions in \cref{sec:attn_only}.

\begin{figure}[h!]
\centering

\begin{minipage}{0.48\linewidth}
    \centering
    \includegraphics[width=0.48\linewidth]{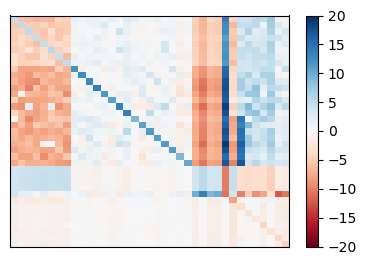}\hspace{1mm}%
    \includegraphics[width=0.48\linewidth]{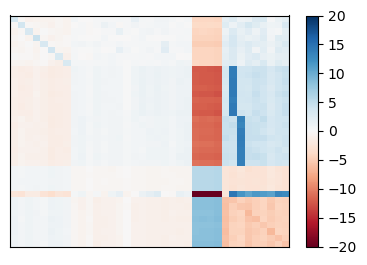}
    \caption{Visualization of the weight matrices $\W_{OV}^{\text{subj}}$ (left) and
    $\W_{OV}^{\text{rel}}$ (right) after pretraining the single-layer attention-only model on \Bio sequences.}
    \label{fig:wov_pt}
\end{minipage}
\hfill
\begin{minipage}{0.48\linewidth}
    \centering
    \includegraphics[width=0.48\linewidth]{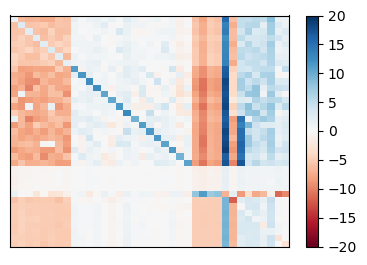}\hspace{1mm}%
    \includegraphics[width=0.48\linewidth]{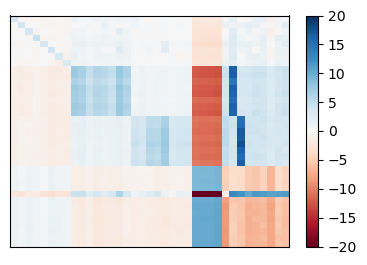}
    \caption{Visualization of the weight matrices $\W_{OV}^{\text{subj}}$ (left) and
    $\W_{OV}^{\text{rel}}$ (right) after finetuning the single-layer attention-only model on \ICL sequences.}
    \label{fig:wov_icl}
\end{minipage}

\vspace{-0.1in}
\end{figure}

\subsubsection{Absolute Positional Encoding}

In this section, we present additional experimental evidence to corroborate our theoretical constructions, using a 1-layer 3-head attention-only model with absolute positional encodings (see App.~\ref{app:expts} for details).

\begin{figure}[h!]
\centering
\begin{minipage}{0.51\linewidth}
    \includegraphics[width=\linewidth]{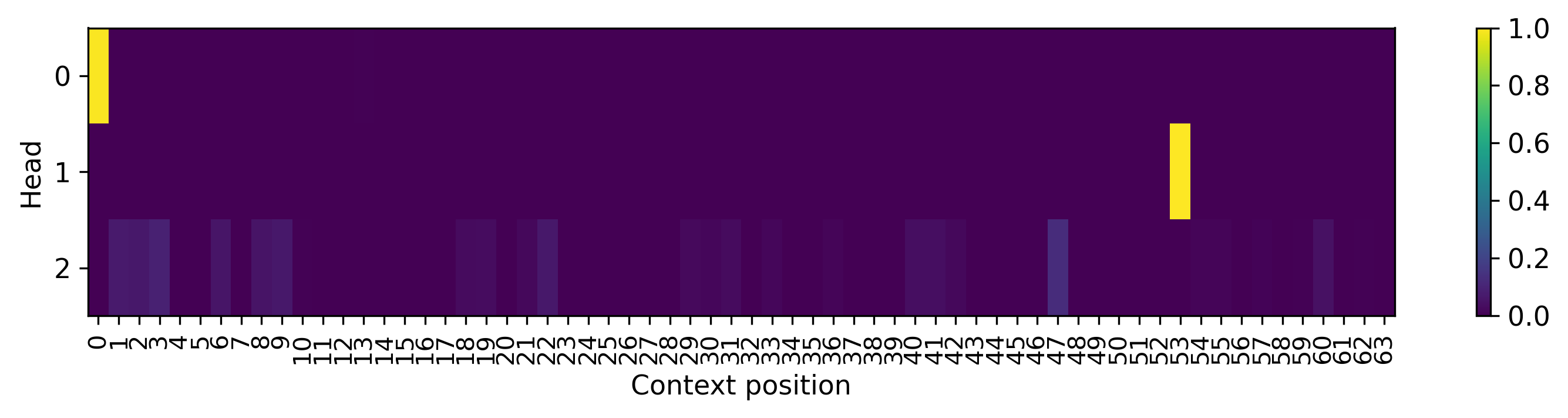}
    \caption{Attention scores for each head across a \Bio sequence at the end of pretraining. Head $0$ attends to the first subject, while head $1$ attends to the most recent relation token, as predicted by \cref{prop:pt}.}\vspace{-0.1in}
    \label{fig:attn_pt_a}
\end{minipage}
\hfill
\begin{minipage}{0.47\linewidth}
    \centering
\includegraphics[width=0.8\linewidth]{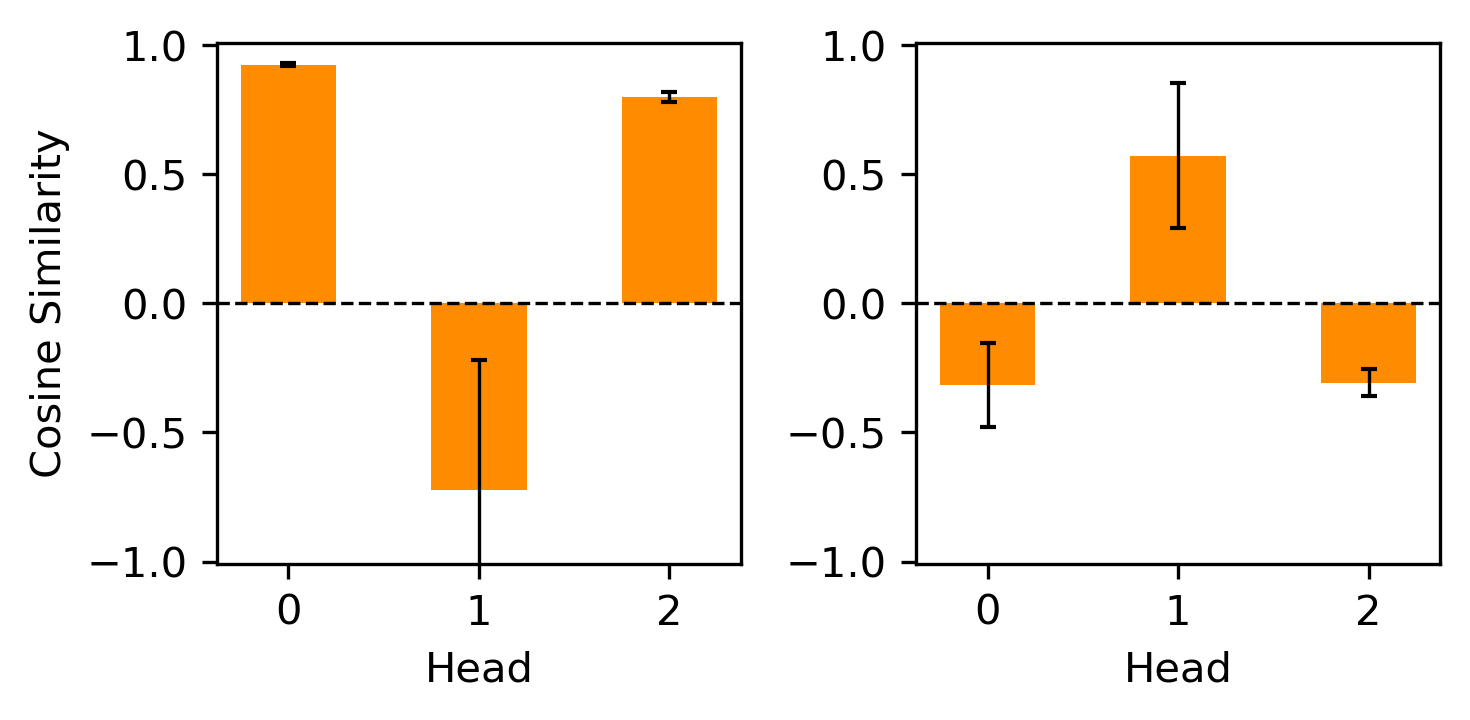}
    \caption{Cosine similarity between the actual outputs of each head and the outputs $f_{\text{subj}}(\X)$ (left) and $f_{\text{rel}}(\X)$ (right) from our construction in \cref{prop:pt}. Head $0$ output closely matches $f_{\text{subj}}(\X)$, while head~$1$ matches $f_{\text{rel}}(\X)$.} \vspace{-0.1in}
    \label{fig:cos_pt_a}
\end{minipage}
\end{figure}

\vspace{-1mm}
\textit{Validation of Pretraining Mechanism.} We first train the model using \Bio sequences (\cref{eq:const_pt_x}) with the next-token prediction objective. In \cref{fig:attn_pt_a}, we visualize the attention scores for each head across a \Bio sequence, and find that the heads specialize into distinct roles consistent with \cref{prop:pt}. We find that head $0$ attention score concentrates on the first subject token, and it outputs $\emb{\obj_{\bar{j}}}$, matching the role of  $g_{\text{subj}}(\X)$ in our construction, while head~$1$ attends to the most recent relation token, \textit{i.e.}, it outputs $\emb{\rel_{\ell_{\Ntr}}}$, consistent with $g_{\text{rel}}(\X)$. Further, in \cref{fig:cos_pt_a}, we compute the cosine similarity between the head outputs and theoretical outputs of the subject and relation heads specified by our construction. Specifically, we report the cosine similarity with $f_{\text{subj}}(\X)$, \textit{i.e.}, negative sum of all attributes not associated with the subject $\obj_{\bar{j}}$ (left subplot) and $f_{\text{rel}}(\X)$, \textit{i.e.}, sum of all attributes of type $\ell_{\Ntr}$ (right), averaged across several sequences. We find that head $0$ output closely matches $f_{\text{subj}}(\X)$, while head $1$ matches $f_{\text{rel}}(\X)$. We designate these heads as subject and relation heads, respectively. Together, these results present experimental validation of our construction for \Bio sequences.

\begin{wrapfigure}[12]{r}{0.51\linewidth}
    \centering
    \vspace{-1mm}
    \includegraphics[width=0.9\linewidth]{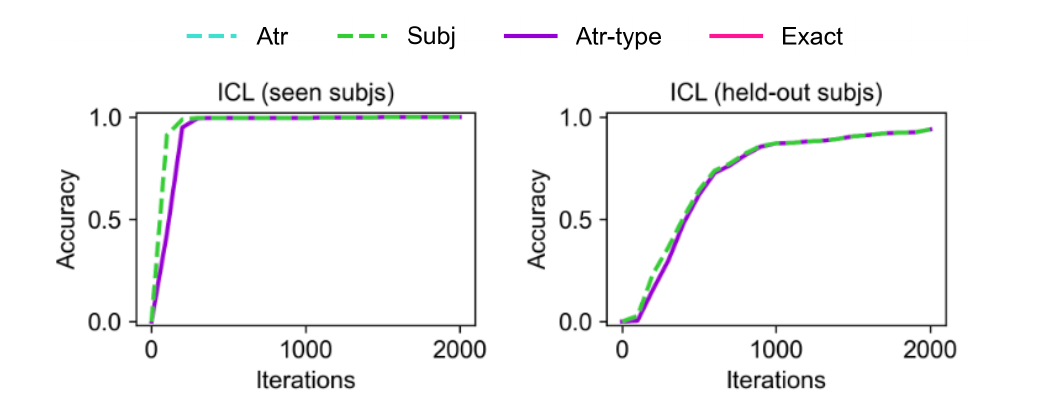}
    \caption{Finetuning the single-layer attention-only model with absolute positional encoding pretrained on \Bio sequences using \ICL sequences with a subset of subjects enables generalization on \ICL sequences with held-out query subjects.}
    \label{fig:attnonly_icl_acc_a}
\end{wrapfigure}
\textit{Validation of ICL Mechanism.} We finetune the model on \ICL sequences (\cref{eq:const_ft_x}) with the last-token prediction objective, using $50\%$ of the total subjects (\cref{fig:attnonly_icl_acc_a} shows that the model generalizes on held-out subjects). Visualizing the attention scores on an \ICL sequence in \cref{fig:attn_icl_a} reveals that the model repurposes its heads for the new task, consistent with Prop. \ref{prop:icl}. We find that the head $0$ (subject head) attends to the most recent/query subject token, while head $1$ (relation head) attends to the attribute tokens in the context, \textit{i.e.}, it outputs a combination of attributes of type $\ell_\star$. Further, \cref{fig:cos_icl_a} confirms that the outputs of these heads closely match the theoretical outputs specified by our construction for \ICL sequences: head $0$ (subject head) outputs negative sum of all attributes not associated with the subject $\obj_{j_{\Nft+1}}$ (left), while head $1$ (relation head) outputs the sum of all attributes of type $\ell_\star$ (right). 

\begin{figure}[h!]
\centering
\begin{minipage}{0.51\linewidth}
    \centering
    \includegraphics[width=0.9\linewidth]{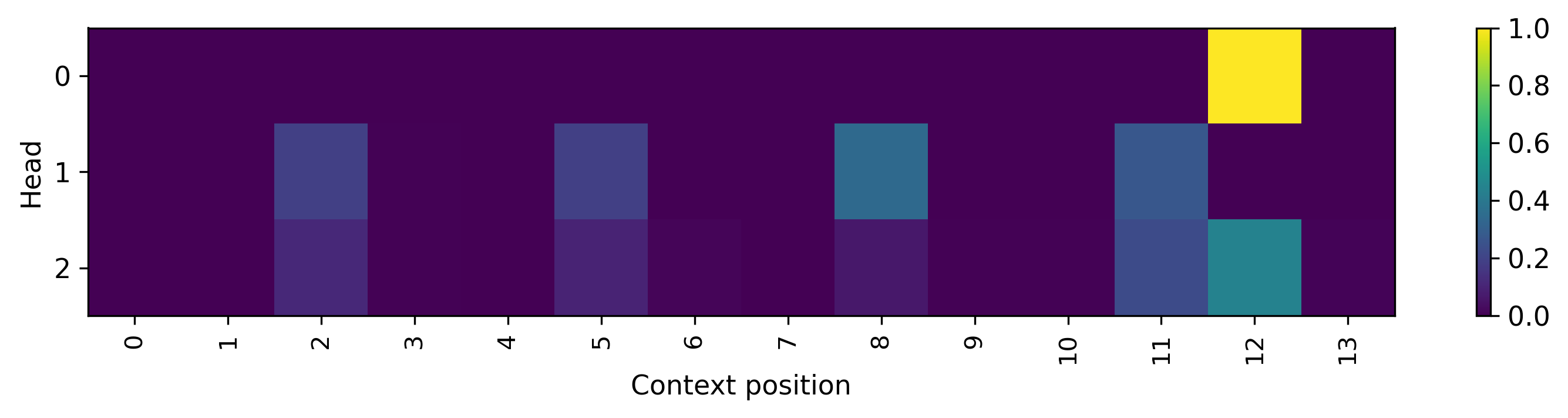}
    \caption{Attention scores for each head across an \ICL sequence at the end of finetuning. Head $0$ attends to the most recent subject, while head $1$ attends to the attribute tokens in the sequence, as predicted by \cref{prop:icl}.}\vspace{-0.1in}
    \label{fig:attn_icl_a}
\end{minipage}
\hfill
\begin{minipage}{0.47\linewidth}
    \centering
    \includegraphics[width=0.8\linewidth]{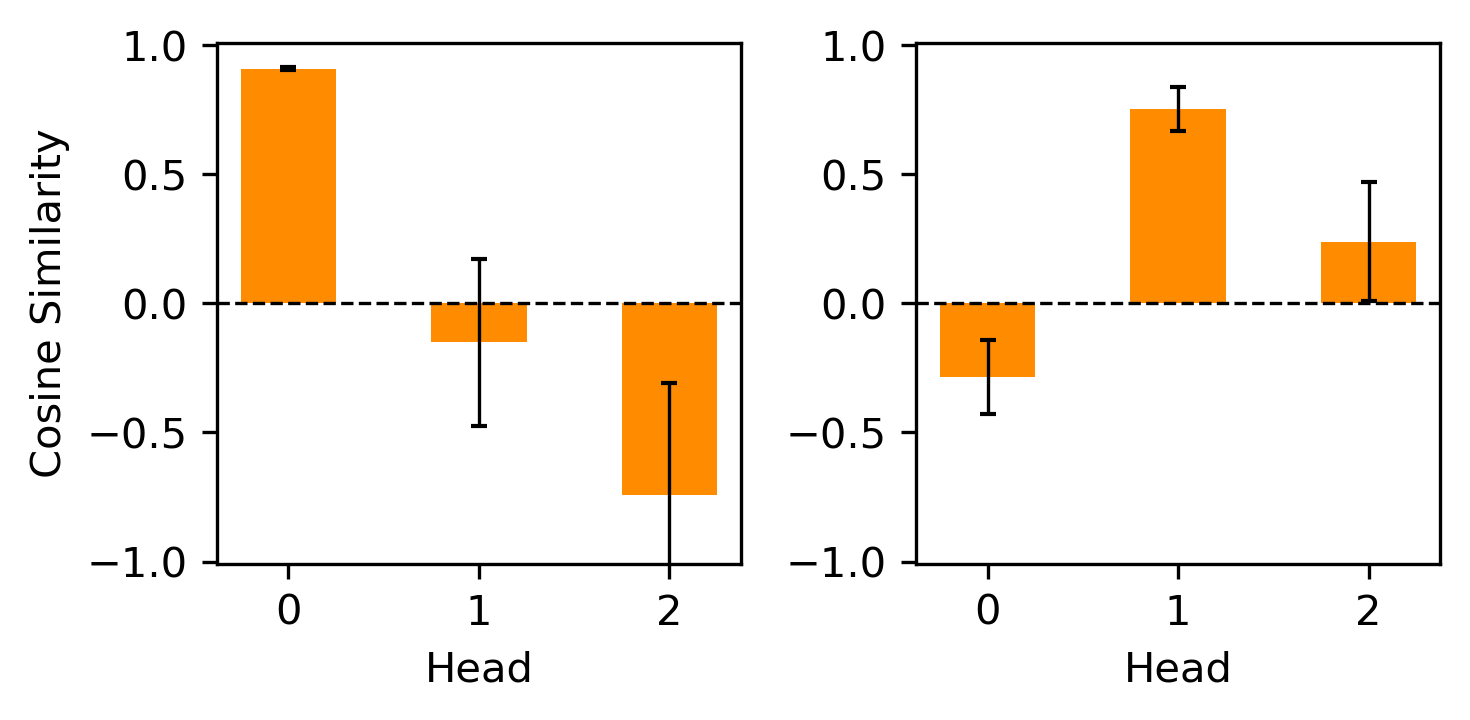}
    \caption{Cosine similarity between the actual outputs of each head and the outputs $f_{\text{subj}}(\X)$ (left) and $f_{\text{rel}}(\X)$ (right) from our construction in \cref{prop:icl}. Head $0$ output closely matches $f_{\text{subj}}(\X)$, while head~$1$ matches $f_{\text{rel}}(\X)$.}\vspace{-0.15in}
    \label{fig:cos_icl_a}
\end{minipage}
\end{figure}

%% file: main.bbl
\begin{thebibliography}{46}
\providecommand{\natexlab}[1]{#1}
\providecommand{\url}[1]{\texttt{#1}}
\expandafter\ifx\csname urlstyle\endcsname\relax
  \providecommand{\doi}[1]{doi: #1}\else
  \providecommand{\doi}{doi: \begingroup \urlstyle{rm}\Url}\fi

\bibitem[Ahn et~al.(2023)Ahn, Cheng, Daneshmand, and Sra]{ahn2023transformers}
Ahn, K., Cheng, X., Daneshmand, H., and Sra, S.
\newblock Transformers learn to implement preconditioned gradient descent for in-context learning.
\newblock In \emph{Thirty-seventh Conference on Neural Information Processing Systems}, 2023.

\bibitem[Aky{\"u}rek et~al.(2023)Aky{\"u}rek, Schuurmans, Andreas, Ma, and Zhou]{akyurek_icl}
Aky{\"u}rek, E., Schuurmans, D., Andreas, J., Ma, T., and Zhou, D.
\newblock What learning algorithm is in-context learning? {I}nvestigations with linear models.
\newblock In \emph{Int. Conference on Learning Representations}, 2023.

\bibitem[Allen-Zhu \& Li(2023)Allen-Zhu and Li]{allen2023physics}
Allen-Zhu, Z. and Li, Y.
\newblock Physics of language models: Part 3.1, knowledge storage and extraction.
\newblock \emph{arXiv preprint arXiv:2309.14316}, 2023.

\bibitem[Bai et~al.(2023)Bai, Chen, Wang, Xiong, and Mei]{bai2023transformers}
Bai, Y., Chen, F., Wang, H., Xiong, C., and Mei, S.
\newblock Transformers as statisticians: Provable in-context learning with in-context algorithm selection.
\newblock In \emph{Thirty-seventh Conference on Neural Information Processing Systems}, 2023.

\bibitem[Behnia et~al.(2025)Behnia, Deora, and Thrampoulidis]{behnia2025factsstatsimpactspretraining}
Behnia, T., Deora, P., and Thrampoulidis, C.
\newblock Facts in stats: Impacts of pretraining diversity on language model generalization.
\newblock \emph{arXiv preprint arXiv:2510.16096}, 2025.

\bibitem[Bhattamishra et~al.(2024)Bhattamishra, Patel, Blunsom, and Kanade]{bhattamishra2024understanding}
Bhattamishra, S., Patel, A., Blunsom, P., and Kanade, V.
\newblock Understanding in-context learning in transformers and {LLM}s by learning to learn discrete functions.
\newblock In \emph{The Twelfth International Conference on Learning Representations}, 2024.

\bibitem[Bietti et~al.(2023)Bietti, Cabannes, Bouchacourt, Jegou, and Bottou]{bietti2023birth}
Bietti, A., Cabannes, V., Bouchacourt, D., Jegou, H., and Bottou, L.
\newblock Birth of a transformer: A memory viewpoint.
\newblock \emph{Advances in Neural Information Processing Systems}, 36:\penalty0 1560--1588, 2023.

\bibitem[Brown et~al.(2020)Brown, Mann, Ryder, Subbiah, Kaplan, Dhariwal, Neelakantan, Shyam, Sastry, Askell, Agarwal, Herbert-Voss, Krueger, Henighan, Child, Ramesh, Ziegler, Wu, Winter, Hesse, Chen, Sigler, Litwin, Gray, Chess, Clark, Berner, McCandlish, Radford, Sutskever, and Amodei]{brown2020language}
Brown, T., Mann, B., Ryder, N., Subbiah, M., Kaplan, J.~D., Dhariwal, P., Neelakantan, A., Shyam, P., Sastry, G., Askell, A., Agarwal, S., Herbert-Voss, A., Krueger, G., Henighan, T., Child, R., Ramesh, A., Ziegler, D., Wu, J., Winter, C., Hesse, C., Chen, M., Sigler, E., Litwin, M., Gray, S., Chess, B., Clark, J., Berner, C., McCandlish, S., Radford, A., Sutskever, I., and Amodei, D.
\newblock Language models are few-shot learners.
\newblock In Larochelle, H., Ranzato, M., Hadsell, R., Balcan, M., and Lin, H. (eds.), \emph{Advances in Neural Information Processing Systems}, volume~33, pp.\  1877--1901. Curran Associates, Inc., 2020.

\bibitem[Carroll et~al.(2025)Carroll, Hoogland, Farrugia-Roberts, and Murfet]{carroll2025dynamicstransientstructureincontext}
Carroll, L., Hoogland, J., Farrugia-Roberts, M., and Murfet, D.
\newblock Dynamics of transient structure in in-context linear regression transformers.
\newblock \emph{arXiv preprint arXiv:2501.17745}, 2025.

\bibitem[Chen et~al.(2025)Chen, Bruna, and Bietti]{chendistributional}
Chen, L., Bruna, J., and Bietti, A.
\newblock Distributional associations vs in-context reasoning: A study of feed-forward and attention layers.
\newblock In \emph{The Thirteenth International Conference on Learning Representations}, 2025.

\bibitem[Chen et~al.(2024{\natexlab{a}})Chen, Sheen, Wang, and Yang]{chen2024trainingdynamicsmultiheadsoftmax}
Chen, S., Sheen, H., Wang, T., and Yang, Z.
\newblock Training dynamics of multi-head softmax attention for in-context learning: Emergence, convergence, and optimality.
\newblock \emph{arXiv preprint arXiv:2402.19442}, 2024{\natexlab{a}}.

\bibitem[Chen et~al.(2024{\natexlab{b}})Chen, Sheen, Wang, and Yang]{chen2024unveiling}
Chen, S., Sheen, H., Wang, T., and Yang, Z.
\newblock Unveiling induction heads: Provable training dynamics and feature learning in transformers.
\newblock \emph{arXiv preprint arXiv:2409.10559}, 2024{\natexlab{b}}.

\bibitem[D'Angelo et~al.(2025)D'Angelo, Croce, and Flammarion]{dangelo2025selective}
D'Angelo, F., Croce, F., and Flammarion, N.
\newblock Selective induction heads: How transformers select causal structures in context.
\newblock In \emph{The Thirteenth International Conference on Learning Representations}, 2025.

\bibitem[Deora et~al.(2025)Deora, Vasudeva, Behnia, and Thrampoulidis]{deora2025incontext}
Deora, P., Vasudeva, B., Behnia, T., and Thrampoulidis, C.
\newblock In-context occam{\textquoteright}s razor: How transformers prefer simpler hypotheses on the fly.
\newblock In \emph{Second Conference on Language Modeling}, 2025.

\bibitem[Dong et~al.(2026)Dong, Jiang, Zhu, and Ning]{dong2026understanding}
Dong, Y., Jiang, J., Zhu, Z., and Ning, X.
\newblock Understanding task vectors in in-context learning: Emergence, functionality, and limitations.
\newblock In \emph{The Fourteenth International Conference on Learning Representations}, 2026.
\newblock URL \url{https://openreview.net/forum?id=CLBVilFk7N}.

\bibitem[Edelman et~al.(2024)Edelman, Tsilivis, Edelman, Malach, and Goel]{statistical_induction_heads}
Edelman, E., Tsilivis, N., Edelman, B.~L., Malach, E., and Goel, S.
\newblock The evolution of statistical induction heads: In-context learning markov chains.
\newblock In \emph{The Thirty-eighth Annual Conference on Neural Information Processing Systems}, 2024.

\bibitem[Fu et~al.(2024)Fu, Chen, Jia, and Sharan]{fu2024transformers}
Fu, D., Chen, T.-Q., Jia, R., and Sharan, V.
\newblock Transformers learn to achieve second-order convergence rates for in-context linear regression.
\newblock \emph{Advances in Neural Information Processing Systems}, 37:\penalty0 98675--98716, 2024.

\bibitem[Garg et~al.(2022)Garg, Tsipras, Liang, and Valiant]{garg_icl}
Garg, S., Tsipras, D., Liang, P., and Valiant, G.
\newblock What can transformers learn in-context? a case study of simple function classes.
\newblock In Oh, A.~H., Agarwal, A., Belgrave, D., and Cho, K. (eds.), \emph{Advances in Neural Information Processing Systems}, 2022.

\bibitem[Han et~al.(2025)Han, Song, Gore, and Agrawal]{han2025emergence}
Han, S., Song, J., Gore, J., and Agrawal, P.
\newblock Emergence and effectiveness of task vectors in in-context learning: An encoder decoder perspective.
\newblock In \emph{Forty-second International Conference on Machine Learning}, 2025.

\bibitem[Hendel et~al.(2023)Hendel, Geva, and Globerson]{hendel2023incontextlearningcreatestask}
Hendel, R., Geva, M., and Globerson, A.
\newblock In-context learning creates task vectors.
\newblock \emph{arXiv preprint arXiv:2310.15916}, 2023.

\bibitem[Hong et~al.(2026)Hong, Vasudeva, Sharan, Rashtchian, Raghavan, and Panigrahy]{hong2026latent}
Hong, G.~Z., Vasudeva, B., Sharan, V., Rashtchian, C., Raghavan, P., and Panigrahy, R.
\newblock Latent concept disentanglement in transformer-based language models.
\newblock In \emph{The Fourteenth International Conference on Learning Representations}, 2026.

\bibitem[Karpathy()]{mingpt_karpathy}
Karpathy, A.
\newblock mingpt.
\newblock \url{https://github.com/karpathy/minGPT?tab=readme-ov-file} [Accessed: March 4, 2024].

\bibitem[Kawata et~al.(2025)Kawata, Song, Bietti, Nishikawa, Suzuki, Vaiter, and Wu]{kawata2025from}
Kawata, R., Song, Y., Bietti, A., Nishikawa, N., Suzuki, T., Vaiter, S., and Wu, D.
\newblock From shortcut to induction head: How data diversity shapes algorithm selection in transformers.
\newblock In \emph{The Thirty-ninth Annual Conference on Neural Information Processing Systems}, 2025.

\bibitem[Lin \& Lee(2024)Lin and Lee]{dual-operating-modes}
Lin, Z. and Lee, K.
\newblock Dual operating modes of in-context learning.
\newblock In \emph{Proceedings of the 41st International Conference on Machine Learning}, ICML'24. JMLR.org, 2024.

\bibitem[Liu et~al.(2024)Liu, Ye, Xing, and Zou]{liu_incontextvectors}
Liu, S., Ye, H., Xing, L., and Zou, J.
\newblock In-context vectors: making in context learning more effective and controllable through latent space steering.
\newblock In \emph{Proceedings of the 41st International Conference on Machine Learning}, ICML'24. JMLR.org, 2024.

\bibitem[Loshchilov \& Hutter(2019)Loshchilov and Hutter]{adamW}
Loshchilov, I. and Hutter, F.
\newblock Decoupled weight decay regularization.
\newblock \emph{arXiv preprint arXiv:1711.05101}, 2019.

\bibitem[Nichani et~al.(2024{\natexlab{a}})Nichani, Damian, and Lee]{nichani2024how}
Nichani, E., Damian, A., and Lee, J.~D.
\newblock How transformers learn causal structure with gradient descent.
\newblock In \emph{Forty-first International Conference on Machine Learning}, 2024{\natexlab{a}}.

\bibitem[Nichani et~al.(2024{\natexlab{b}})Nichani, Lee, and Bietti]{nichani2024understanding}
Nichani, E., Lee, J.~D., and Bietti, A.
\newblock Understanding factual recall in transformers via associative memories.
\newblock \emph{arXiv preprint arXiv:2412.06538}, 2024{\natexlab{b}}.

\bibitem[Olsson et~al.(2022)Olsson, Elhage, Nanda, Joseph, DasSarma, Henighan, Mann, Askell, Bai, Chen, Conerly, Drain, Ganguli, Hatfield-Dodds, Hernandez, Johnston, Jones, Kernion, Lovitt, Ndousse, Amodei, Brown, Clark, Kaplan, McCandlish, and Olah]{olsson2022context}
Olsson, C., Elhage, N., Nanda, N., Joseph, N., DasSarma, N., Henighan, T., Mann, B., Askell, A., Bai, Y., Chen, A., Conerly, T., Drain, D., Ganguli, D., Hatfield-Dodds, Z., Hernandez, D., Johnston, S., Jones, A., Kernion, J., Lovitt, L., Ndousse, K., Amodei, D., Brown, T., Clark, J., Kaplan, J., McCandlish, S., and Olah, C.
\newblock In-context learning and induction heads.
\newblock \emph{Transformer Circuits Thread}, 2022.
\newblock https://transformer-circuits.pub/2022/in-context-learning-and-induction-heads/index.html.

\bibitem[Pan et~al.(2023)Pan, Gao, Chen, and Chen]{pan-etal-2023-context}
Pan, J., Gao, T., Chen, H., and Chen, D.
\newblock What in-context learning {\textquotedblleft}learns{\textquotedblright} in-context: Disentangling task recognition and task learning.
\newblock In Rogers, A., Boyd-Graber, J., and Okazaki, N. (eds.), \emph{Findings of the Association for Computational Linguistics: ACL 2023}, pp.\  8298--8319, Toronto, Canada, July 2023. Association for Computational Linguistics.
\newblock \doi{10.18653/v1/2023.findings-acl.527}.

\bibitem[Park et~al.(2025)Park, Lubana, and Tanaka]{algorithmic_phases}
Park, C.~F., Lubana, E.~S., and Tanaka, H.
\newblock Algorithmic phases of in-context learning.
\newblock In \emph{The Thirteenth International Conference on Learning Representations}, 2025.

\bibitem[Rajaraman et~al.(2024)Rajaraman, Bondaschi, Makkuva, Ramchandran, and Gastpar]{rajaraman2024transformers}
Rajaraman, N., Bondaschi, M., Makkuva, A.~V., Ramchandran, K., and Gastpar, M.
\newblock Transformers on markov data: Constant depth suffices.
\newblock In \emph{The Thirty-eighth Annual Conference on Neural Information Processing Systems}, 2024.

\bibitem[Raventos et~al.(2023)Raventos, Paul, Chen, and Ganguli]{raventos2023pretraining}
Raventos, A., Paul, M., Chen, F., and Ganguli, S.
\newblock Pretraining task diversity and the emergence of non-bayesian in-context learning for regression.
\newblock In \emph{Thirty-seventh Conference on Neural Information Processing Systems}, 2023.

\bibitem[Singh et~al.(2023)Singh, Chan, Moskovitz, Grant, Saxe, and Hill]{singh2023transientnatureemergentincontext}
Singh, A.~K., Chan, S. C.~Y., Moskovitz, T., Grant, E., Saxe, A.~M., and Hill, F.
\newblock The transient nature of emergent in-context learning in transformers.
\newblock \emph{arXiv preprint arXiv:2311.08360}, 2023.

\bibitem[Singh et~al.(2024)Singh, Moskovitz, Hill, Chan, and Saxe]{singh2024needs}
Singh, A.~K., Moskovitz, T., Hill, F., Chan, S.~C., and Saxe, A.~M.
\newblock What needs to go right for an induction head? a mechanistic study of in-context learning circuits and their formation.
\newblock In \emph{International Conference on Machine Learning}, pp.\  45637--45662. PMLR, 2024.

\bibitem[Singh et~al.(2025)Singh, Moskovitz, Dragutinovic, Hill, Chan, and Saxe]{singh2025strategycoopetitionexplainsemergence}
Singh, A.~K., Moskovitz, T., Dragutinovic, S., Hill, F., Chan, S. C.~Y., and Saxe, A.~M.
\newblock Strategy coopetition explains the emergence and transience of in-context learning.
\newblock \emph{arXiv preprint arXiv:2503.05631}, 2025.

\bibitem[Todd et~al.(2024)Todd, Li, Sharma, Mueller, Wallace, and Bau]{todd2024function}
Todd, E., Li, M., Sharma, A.~S., Mueller, A., Wallace, B.~C., and Bau, D.
\newblock Function vectors in large language models.
\newblock In \emph{The Twelfth International Conference on Learning Representations}, 2024.

\bibitem[von Oswald et~al.(2022)von Oswald, Niklasson, Randazzo, Sacramento, Mordvintsev, Zhmoginov, and Vladymyrov]{vonOswald_icl}
von Oswald, J., Niklasson, E., Randazzo, E., Sacramento, J., Mordvintsev, A., Zhmoginov, A., and Vladymyrov, M.
\newblock Transformers learn in-context by gradient descent.
\newblock \emph{arXiv preprint arXiv:2212.07677}, 2022.

\bibitem[Wu et~al.(2023)Wu, Zou, Chen, Braverman, Gu, and Bartlett]{wu2023many}
Wu, J., Zou, D., Chen, Z., Braverman, V., Gu, Q., and Bartlett, P.~L.
\newblock How many pretraining tasks are needed for in-context learning of linear regression?
\newblock \emph{arXiv preprint arXiv:2310.08391}, 2023.

\bibitem[Xie et~al.(2022)Xie, Raghunathan, Liang, and Ma]{xie2022an}
Xie, S.~M., Raghunathan, A., Liang, P., and Ma, T.
\newblock An explanation of in-context learning as implicit bayesian inference.
\newblock In \emph{International Conference on Learning Representations}, 2022.

\bibitem[Yang et~al.(2025)Yang, Lin, Lee, Papailiopoulos, and Nowak]{yang2025taskvectorsincontextlearning}
Yang, L., Lin, Z., Lee, K., Papailiopoulos, D., and Nowak, R.
\newblock Task vectors in in-context learning: Emergence, formation, and benefit.
\newblock \emph{arXiv preprint arXiv:2501.09240}, 2025.

\bibitem[Yin \& Steinhardt(2025)Yin and Steinhardt]{yin2025attention}
Yin, K. and Steinhardt, J.
\newblock Which attention heads matter for in-context learning?
\newblock In \emph{Forty-second International Conference on Machine Learning}, 2025.

\bibitem[Zhang et~al.(2023)Zhang, Frei, and Bartlett]{zhang2023trained}
Zhang, R., Frei, S., and Bartlett, P.~L.
\newblock Trained transformers learn linear models in-context.
\newblock \emph{arXiv preprint arXiv:2306.09927}, 2023.

\bibitem[Zhang et~al.(2024)Zhang, Wu, and Bartlett]{zhang2024context}
Zhang, R., Wu, J., and Bartlett, P.
\newblock In-context learning of a linear transformer block: Benefits of the mlp component and one-step gd initialization.
\newblock \emph{Advances in Neural Information Processing Systems}, 37:\penalty0 18310--18361, 2024.

\bibitem[Zhang et~al.(2025)Zhang, Singh, Latham, and Saxe]{zhang2025trainingdynamicsincontextlearning}
Zhang, Y., Singh, A.~K., Latham, P.~E., and Saxe, A.
\newblock Training dynamics of in-context learning in linear attention.
\newblock \emph{arXiv preprint arXiv:2501.16265}, 2025.

\bibitem[Zucchet et~al.(2025)Zucchet, Bornschein, Chan, Lampinen, Pascanu, and De]{zucchet2025language}
Zucchet, N., Bornschein, J., Chan, S., Lampinen, A., Pascanu, R., and De, S.
\newblock How do language models learn facts? dynamics, curricula and hallucinations.
\newblock \emph{arXiv preprint arXiv:2503.21676}, 2025.

\end{thebibliography}
